\documentclass[10pt]{article}
\usepackage[utf8]{inputenc}

\usepackage{natbib}
\usepackage[T1]{fontenc}    
\usepackage{hyperref}       
\usepackage{url}            
\usepackage{booktabs}       
\usepackage{amsfonts}       
\usepackage{nicefrac}       
\usepackage[patch=none]{microtype}      
\usepackage{xcolor}         

\usepackage{bbm}
\usepackage{authblk}
\usepackage{enumitem}
\usepackage{graphicx}
\usepackage{appendix}
\usepackage{caption}
\usepackage{subcaption}
\usepackage{amsmath}
\usepackage{amssymb}
\usepackage{mathtools}
\usepackage{amsthm}
\usepackage{algorithmic}
\usepackage{algorithm}
\usepackage{algorithm,algorithmic,latexsym}
\usepackage{blkarray}
\usepackage{multirow}
\usepackage{makecell}
\usepackage{tikz}

\usepackage{fancyhdr}
\fancypagestyle{firstpage}{
  \fancyhf{} 

}
\usepackage{comment}
\usetikzlibrary{automata, positioning, arrows}

\theoremstyle{plain}
\newtheorem{theorem}{Theorem}[section]
\newtheorem{proposition}[theorem]{Proposition}
\newtheorem{lemma}[theorem]{Lemma}
\newtheorem{corollary}[theorem]{Corollary}
\theoremstyle{definition}
\newtheorem{definition}[theorem]{Definition}
\newtheorem{assumption}{Assumption}
\theoremstyle{remark}
\newtheorem{remark}[theorem]{Remark}

\usepackage{tcolorbox}
\tcbuselibrary{most}

\newcommand{\behpi}{\pi_{\mathrm{b}}}
\newcommand{\tmix}{t_{\mathrm{mix}}}

\newcommand{\linftynorm}[1]{\left\|#1\right\|_{\infty}}

\newcommand{\norm}[2][q]{\left\| #2\right\|_{#1}}
\newcommand{\seminorm}[2][span]{sp_{\mathrm{#1}}\left(#2\right)}
\newcommand{\spannorm}[1]{sp\left(#1\right)}
\newcommand{\spanstar}[1]{\widetilde{sp}\left(#1\right)}

\newcommand{\bracket}[1]{\left( #1 \right)}
\newcommand{\innerprod}[2]{\langle #1, #2 \rangle}

\newcommand{\hatT}{\widehat{\mathcal{T}}}

\newcommand{\bPi}{\boldsymbol{\Pi}}

\usepackage{threeparttable}
\usetikzlibrary{arrows.meta,positioning,calc}

\newcommand{\E}{\mathbb{E}}

\newcommand{\R}{\mathbb{R}}

\newcommand{\N}{\mathbb{N}}

\newcommand{\Aa}{\mathbf{A}}
\newcommand{\Ii}{\mathbf{I}}
\newcommand{\Kk}{\mathbf{K}}
\newcommand{\Pp}{\mathbf{P}}

\newcommand{\cA}{\mathcal{A}}

\newcommand{\cC}{\mathcal{C}}
\newcommand{\cD}{\mathcal{D}}

\newcommand{\cF}{\mathcal{F}}

\newcommand{\cM}{\mathcal{M}}

\newcommand{\cS}{\mathcal{S}}
\newcommand{\cT}{\mathcal{T}}

\newcommand{\cY}{\mathcal{Y}}
\newcommand{\cZ}{\mathcal{Z}}

\newcommand{\set}[1]{\left\{{#1}\right\}}

\newcommand{\sqbk}[1]{\left[ #1 \right]}
\newcommand{\sqbkcond}[2]{\left[ #1 \middle| #2 \right]}

\newcommand{\argmax}[1]{\underset{#1}{\operatorname{arg}\,\operatorname{max}}\;}
\newcommand{\argmin}[1]{\underset{#1}{\operatorname{arg}\,\operatorname{min}}\;}

\definecolor{azure}{rgb}{0.0, 0.4, 0.9}

\definecolor{darkred}{rgb}{0.6, 0, 0}

\usepackage{pgfplots}
\pgfplotsset{compat=1.15}
\usepackage{setspace}
\usepackage{geometry}
\newcounter{BMatrix}

\newcommand{\setmaxwd}[1]{%
  \eqmakebox[BM-\theBMatrix][\BMalign]{$#1$}%
}
\MHInternalSyntaxOn

\MHInternalSyntaxOff
\makeatother

\lstdefinestyle{mystyle}{
    backgroundcolor=\color{backcolour},   
    commentstyle=\color{codegreen},
    keywordstyle=\color{magenta},
    numberstyle=\tiny\color{codegray},
    stringstyle=\color{codepurple},
    basicstyle=\scriptsize\ttfamily,
    identifierstyle=\color{codeblue},
    breakatwhitespace=false,         
    breaklines=true,                 
    captionpos=b,                    
    keepspaces=true,                 
    numbers=left,                    
    numbersep=4pt,                  
    showspaces=false,                
    showstringspaces=false,
    showtabs=true,                  
    tabsize=3
}
\numberwithin{equation}{section}

\title{Achieving $\varepsilon^{-2}$ Dependence for Average-Reward Q-Learning with a New Contraction Principle}
\author[1]{Zijun Chen}
\author[2]{Zaiwei Chen}
\author[3]{Nian Si\thanks{Corresponding author: niansi@ust.hk}}
\author[4]{Shengbo Wang}
\affil[1]{Department of Computer Science and Engineering, HKUST}
\affil[2]{Edwardson School of Industrial Engineering, Purdue University}
\affil[3]{Department of Industrial Engineering and Decision Analytics, HKUST}
\affil[4]{Daniel J. Epstein Department of Industrial and Systems Engineering, USC}
\date{\today}
\begin{document}
	\maketitle
	\thispagestyle{firstpage}
	
	\begin{abstract}
   We present the convergence rates of synchronous and asynchronous Q-learning for average-reward Markov decision processes, where the absence of contraction poses a fundamental challenge. Existing non-asymptotic results overcome this challenge by either imposing strong assumptions to enforce seminorm contraction or relying on discounted or episodic Markov decision processes as successive approximations, which either require unknown parameters or result in suboptimal sample complexity. In this work, under a reachability assumption, we establish optimal $\widetilde{O}(\varepsilon^{-2})$ sample complexity guarantees (up to logarithmic factors) for a simple variant of synchronous and asynchronous Q-learning that samples from the lazified dynamics, where the system remains in the current state with some fixed probability. At the core of our analysis is the construction of an instance-dependent seminorm and showing that, after a lazy transformation of the Markov decision process, the Bellman operator becomes one-step contractive under this seminorm. 
\end{abstract}
	
	\section{Introduction}
\label{sec:introduction}
Reinforcement learning has emerged as a principled framework for sequential decision making and has demonstrated remarkable success across a wide range of domains, including superhuman performance in Atari games \citep{Mnih2013PlayingAW}, mastery of Go \citep{silver2017mastering}, and robotics \citep{kober2013reinforcement}.

One of the fundamental building blocks of reinforcement learning is \emph{Q-learning} \citep{watkins1992q}. Q-learning is a model-free method that learns the optimal action-value function based on the Bellman operator, from which an optimal policy can be obtained by acting greedily. 

While the theoretical analysis of Q-learning is well developed in the discounted setting,  results in the \emph{average-reward} setting remain comparatively limited, and existing results are often weaker, relying either on restrictive assumptions or on complicated, hard-to-implement modifications of the vanilla algorithm.  In the average-reward formulation, the objective is to maximize the long-run expected reward per time step. Specifically, for a Markov decision process with state $s \in \cS$, action $a\in \cA$ and reward function $r(\cdot,\cdot)$, we seek a policy $\pi$ that maximizes
\[
\max_{\pi}\; \lim_{T\to\infty} \frac{1}{T}\,
\mathbb{E}_{\pi}\!\left[\sum_{t=0}^{T-1} r(s_t,a_t) \right].
\]
This criterion naturally models long-term behaviors and continuing tasks \citep{doi:10.1287/stsy.2021.0081} in which discounting is inappropriate. However, the absence of a natural contraction property introduces substantial technical challenges, rendering the theoretical guarantees for Q-learning in the average-reward setting far less satisfactory than their discounted counterparts \citep{Devraj2020QLearningWU, pmlr-v238-jin24b}.

Several works attempt to address this difficulty by imposing explicit seminorm contraction assumptions \citep{NEURIPS2021_096ffc29, chen2025nonasymptotic}. These assumptions are hard to verify and stronger than necessary, as we demonstrate through an example in Section~\ref{sec:numerical_experiments}. Other approaches aim to control weaker surrogate criteria, such as Bellman residuals under nonexpansive dynamics \citep{doi:10.1137/22M1515550}, or circumvent the noncontractive nature of the average-reward Bellman operator via discount approximation \citep{pmlr-v238-jin24b}. However, these methods typically lead to suboptimal dependence on the accuracy parameter~$\varepsilon$.

In this work, we instead impose assumptions directly on the underlying Markov decision processes (MDPs), rather than relying on high-level and hard-to-verify analytical conditions such as contractivity.
Specifically, we assume reachability (Assumption~\ref{ass:reachability}): there exists a state that can be reached from any other state under any stationary deterministic policy.
Under this assumption, we show that the average-reward Bellman operator admits a one-step contraction property with respect to an instance-dependent seminorm.
Building on this structure, we design a simple variant of Q-learning—namely, \textit{lazy Q-learning}—and show that it achieves an $\varepsilon^{-2}$ dependence.
To the best of our knowledge, this is the first result in the literature that achieves an $\varepsilon^{-2}$ sample complexity without assuming any \emph{a priori} contraction of the Bellman operator.



Our main contributions are as follows:
\begin{enumerate}[label=(\roman*), leftmargin=2.2em]
\item \textbf{A New Contraction Principle.}
    Assuming only reachability (Assumption~\ref{ass:reachability}), we show that, after a simple \emph{lazy} transformation, there exists a new seminorm $\widetilde{sp}(\cdot)$ under which the Bellman operator becomes a strict one-step contraction. This observation reveals a fundamental property of the Bellman operator and is of independent interest. It may also have potential applications beyond Q-learning.
    To construct $\widetilde{sp}(\cdot)$ and establish the contraction property, we leverage the lazy transformation together with ideas from extremal norm theory.
    
    \item \textbf{Synchronous Lazy Q-learning.}
    We propose two easy-to-implement synchronous variants of lazy Q-learning: one with explicit lazy
sampling (Algorithm~\ref{alg:sync_q_explicit_lazy_sampling}) and one with implicit
lazy sampling (Algorithm~\ref{alg:sync_q_implicit_lazy_sampling}). For
both algorithms, we achieve a sample complexity of
$
\widetilde{O}\!\left(|\mathcal{S}||\mathcal{A}|\varepsilon^{-2}\right),
$
which guarantees that the last-iterate Q-function attains an error at most
$\varepsilon$ in the span seminorm, and simultaneously that the induced policy is $\varepsilon$-optimal
with high probability (Theorem~\ref{thm:sync_q_learning_optimal_rate}).

    \item \textbf{Asynchronous Lazy Q-learning.} We further study the asynchronous setting, where data are collected from a single-trajectory Markov
chain induced by a behavior policy. For the two variants of Q-learning
with explicit lazy sampling (Algorithm~\ref{alg:async_q_explicit_lazy_sampling}) and
with implicit lazy sampling (Algorithm~\ref{alg:async_q_implicit_lazy_sampling}), we achieve a sample complexity of
$
\widetilde{O}\!\left(\varepsilon^{-2}\right); 
$ 
see
Theorems~\ref{thm:async_q_learning_optimal_rate_explicit} and
\ref{thm:async_q_learning_optimal_rate_implicit} for the full dependence on other problem parameters. We obtain this result via a
novel Lyapunov-function construction enabled by our new contraction principle.

\end{enumerate}

\subsection{Literature Review}
\label{sec:literature_review}
Q-learning is widely used and has been extensively studied in the discounted setting, where its convergence and finite-sample behavior are by now well understood; see, e.g., \citet{gheshlaghi2013minimax, wainwright2019stochasticapproximationconecontractiveoperators, wainwright2019variancereducedqlearningminimaxoptimal, Li2020SampleCO, pmlr-v139-li21b, doi:10.1287/opre.2023.2450,chen2024lyapunov} and the references therein. In contrast, theoretical progress for Q-learning under the \emph{average-reward} criterion remains comparatively limited and, in many respects, unsatisfactory, as we discuss below.

\noindent\textbf{Model-based sample complexity for average-reward Markov Decision Processes (AMDPs).} 
Recent works \citep{pmlr-v139-jin21b, wang2024optimal, zurek2024span, pmlr-v291-zurek25a, zurek2025the} studied model-based algorithms for average-reward MDPs (AMDPs) and established minimax-optimal sample complexity bounds of order $\widetilde{\Theta}(|\cS||\cA|\tmix \varepsilon^{-2})$ under various assumptions and forms of prior knowledge. These results are information-theoretic in nature and rely on access to an explicit model, and thus do not apply to model-free methods such as Q-learning.

\noindent\textbf{Asymptotic convergence for average-reward Q-learning.}
Using stochastic approximation and ODE techniques, prior works \citep{10.1137/S0363012997331639,abounadi2001learning} established almost sure convergence of average-reward Q-learning under unichain-type conditions. Subsequent extensions \citep{wan2021averagereward,wan2024convergenceaveragerewardqlearningweakly,yu2025asynchronousstochasticapproximationapplications} covered broader classes such as weakly communicating MDPs and semi-MDPs. These results are asymptotic and do not yield finite-time last-iterate guarantees.

\noindent\textbf{Finite-sample analysis of average-reward Q-learning.} A number of recent works study finite-sample guarantees for average-reward Q-learning by imposing explicit contraction assumptions. For example, \citet{NEURIPS2021_096ffc29} analyze $J$-step synchronous Q-learning under a $J$-step $\gamma$-contraction 
$
\widetilde O\!\left(|\cS||\cA|J^3(1-\gamma)^{-5}\varepsilon^{-2}\right),
$
which depends on implicit contraction parameters such as $\gamma$ and $J$. In the asynchronous setting, \citet{chen2025nonasymptotic} derives non-asymptotic convergence guarantees for average-reward Q-learning under a $1$-step $\gamma$-contraction assumption, with sample complexity scaling as
$
\widetilde O\!\left(\varepsilon^{-2}\right).
$ (omitting dependence on other parameters).

  

\begin{table*}[t]
\caption{Summary of S.O.T.A. results on average-reward Q-learning variants for finding an $\varepsilon$-optimal Q-function. We report the dependence on $|\cS|,|\cA|$ and $\varepsilon$, treating all other instance-dependent parameters as constants absorbed into the $\widetilde O(\cdot)$ notation.}
\label{tab:summary_rate}%
\centering
\begin{threeparttable}
  \begin{tabular}{llll}
    \toprule
    Setting & Assumption & Sample Complexity & Origin \\
    \midrule 
    \multirow{4}{*}{Synchronous} & $J$-step $\gamma$-contraction &  $\widetilde O\bracket{|\cS||\cA|\varepsilon^{-2}}$  &\citet{NEURIPS2021_096ffc29} \\
     &Unichain & $\widetilde O\bracket{|\cS||\cA|\varepsilon^{-9}}$ & \citet{pmlr-v238-jin24b} \\
    & Unichain & $\widetilde{O}\bracket{|\cS||\cA|\varepsilon^{-6}}$ & \citet{doi:10.1137/22M1515550}\\
     & Reachability & $\widetilde{O}\bracket{|\cS||\cA|\varepsilon^{-2}}$ & \textbf{This work }(Theorem~\ref{thm:sync_q_learning_optimal_rate}) \\
    \midrule
    \multirow{2}{*}{Asynchronous\tnote{1}} &$1$-step $\gamma$-contraction & $\widetilde O\bracket{\varepsilon^{-2}}$ & \citet{chen2025nonasymptotic} \\
    & Reachability & $\widetilde{O}\bracket{\varepsilon^{-2}}$ & \textbf{This work }(Theorems~\ref{thm:async_q_learning_optimal_rate_explicit},~\ref{thm:async_q_learning_optimal_rate_implicit}) \\ 
    \bottomrule
  \end{tabular}
  \begin{tablenotes}[flushleft]
    \footnotesize
    \item[1] In the asynchronous setting, the dependence on $|\cS|, |\cA|$ is implicit through the minimal stationary probability; hence we only report the dependence on $\varepsilon$.
  \end{tablenotes}
\end{threeparttable}
\end{table*}

Without the contraction assumption, existing approaches usually results in suboptimal dependence of $\varepsilon$. For instance, \citet{pmlr-v238-jin24b} propose a Q-learning algorithm using discounted approximation with an adaptively changing discount factor and obtain a bound of order
$
\widetilde{O}\!\left(|\cS||\cA|\varepsilon^{-9}\right).
$
\citet{doi:10.1137/22M1515550} study stochastic approximation schemes for nonexpansive operators, which can be instantiated for average-reward Q-learning without assuming contraction, and establish finite-sample bounds on the expected Bellman residual in the span seminorm of order
$
\widetilde{O}\!\left(|\cS||\cA|\varepsilon^{-6}\right),
$
which has suboptimal dependence of $\varepsilon$.

Some related works include \citet{zhang2023sharper}, \citet{bravo2024stochastic}, and \citet{lee2025nearoptimal}, which are technically model-free and based on Q-learning–type updates. However, these methods require sampling at least $\Omega(\varepsilon^{-1})$ times for each state-action pair. This effectively forces the algorithm to learn local transition information everywhere, making its design closer in spirit to model-based methods. As a result, these approaches do not retain the simplicity and flexibility that make genuinely model-free algorithms like vanilla Q-learning appealing, and they are difficult to extend to the single-trajectory asynchronous setting.

We summarize the state-of-the-art sample complexity results for average-reward Q-learning in the literature, together with a comparison to our work, in Table~\ref{tab:summary_rate}.

\noindent\textbf{Other settings of average-reward reinforcement learning.} 
A related line of work studies average-reward reinforcement learning from an online learning perspective, where the primary objective is to minimize regret; see, for example, \citet{10.5555/3524938.3525880, zhang2023sharper, pmlr-v291-agrawal25a}. In addition, learning optimal average-reward policies from offline datasets has been investigated in \citet{ozdaglar2024offlinereinforcementlearninglinearprogramming, gabbianelli2024offline, lee2025finitetime}. 
While both the online and offline settings are highly useful and important, the corresponding methods and techniques depart substantially from the Q-learning paradigm studied in this paper when deriving sample complexity guarantees.
	
	\section{Preliminaries}
\label{sec:preliminaries}


An AMDP is characterized by a tuple $\cM = (\cS, \cA, P, r)$, where $\cS$ and $\cA$ denote the state
and action spaces, respectively. We assume that both $\cS$ and $\cA$ are finite. The transition
kernel $P = \set{p(s' \mid s,a)}_{(s,a,s')\in \cS\times \cA\times \cS}$ specifies a collection of
transition distributions $p:\cS\times \cA \to \Delta(\cS)$, mapping each state--action pair to a
probability distribution over next states, where $\Delta(\cS)$ denotes the probability simplex
over $\cS$. The reward function $r:\cS\times \cA\to[0,1]$ specifies the immediate reward collected
when action $a\in\cA$ is taken in state $s\in\cS$.


Let $\Pi^{\mathrm{HD}}$, $\Pi^{\mathrm{M}}$, $\Pi^{\mathrm{S}}$, and $\Pi^{\mathrm{SD}}$ denote the
classes of history-dependent, Markovian, stationary, and stationary deterministic policies,
respectively. A history-dependent policy $\pi=\bracket{\pi_t}_{t\ge 0}\in \Pi^{\mathrm{HD}}$ is
a sequence of conditional distributions $\pi_t : (\cS\times \cA)^t \times \cS \to \Delta(\cA)$
over actions given the past history and the current state. A policy $\pi\in\Pi^{\mathrm{M}}$ is \emph{Markov} if $\pi_t$ depends on the history only through the current state; that is, for any history and $a\in\cA$, $\pi_t(a \mid s_0,a_0,\ldots,s_t)=\pi_t(a\mid s_t).$ Note that both history-dependent and Markov policies may be time dependent.

A policy $\pi\in \Pi^{\mathrm{S}}$ is \emph{stationary} if it is Markov and time-invariant,
i.e.\ $\pi_t(\cdot\mid s)\equiv \pi'(\cdot\mid s)$ for some $\pi': \cS \to \Delta(\cA)$.
In addition, a stationary deterministic policy $\pi\in \Pi^{\mathrm{SD}}$ further satisfies that the range of
$\pi'$ is $\set{\delta_a: a\in\cA}$, where $\delta_a$ is the point-mass measure at $a$.
Therefore, stationary and stationary deterministic policies can be identified with mappings
$\pi:\cS\to \Delta(\cA)$ and $\pi:\cS\to \cA$, respectively.

Evidently, the relationships among these policy classes are as follows: $\Pi^{\mathrm{SD}} \subset \Pi^{\mathrm{S}} \subset \Pi^{\mathrm{M}} \subset \Pi^{\mathrm{HD}}.$ Moreover, for $\pi$ in any of the policy classes, let $\mathbb P^\pi[\cdot]$ and $\mathbb E^\pi[\cdot]$ denote the probability measure and expectation on the canonical space $(\cS\times\cA)^\N$ induced by $P$ and $\pi$.

\paragraph{Reachability of a reference state.} Central to our analysis is the following reachability assumption.  
\begin{assumption}[Reachability]
\label{ass:reachability}
There exists a state $s^\dagger\in \cS$ such that, for any $s\in \cS$ and any policies $\pi \in \Pi^{\mathrm{SD}}$, there exists $t\geq 0$ such that
$\mathbb P^\pi\sqbkcond{s_t = s^\dagger}{s_0 = s} > 0. $
\end{assumption}

This condition is slightly stronger than the unichain assumption, but weaker than the recurrent assumption \citep{Puterman1994MDP}.
The relationship among these MDP classes and assumptions is illustrated in Figure~\ref{fig:hierarchy}.

Although the time $t$ in Assumption~\ref{ass:reachability} may depend on the initial state $s$ and the policy $\pi$, since the AMDP has a finite state--action space, the uniform expected hitting time of the state $s^\dagger$ is finite. Specifically, we have
\begin{align*}
    K := \max_{\pi\in \Pi^{\mathrm{SD}}, s\in \cS}
    \mathbb{E}^{\pi}\sqbkcond{\inf\set{t>0 \mid s_t = s^\dagger}}{s_0 = s}<\infty.
\end{align*}
whose proof is deferred to Appendix~\ref{subsec:app:finiteness_of_the_hitting_time}.

Given a policy $\pi\in\Pi^{\mathrm{SD}}$, define the long-run average-reward starting from state $s$ as
\[
g^\pi(s) := \lim_{T\to \infty}\frac{1}{T}\mathbb E^\pi\sqbkcond{\sum_{t=0}^{T-1}r(s_t, a_t)}{s_0 = s}.
\]
The limit always exists and, under Assumption~\ref{ass:reachability}, the average reward is
state independent; that is, there exists a constant $g^\pi\in\mathbb{R}$ such that $g^\pi(s)=g^\pi$ for all
$s\in\cS$ \citep{Puterman1994MDP}. We define the optimal average reward by
$g^* := \max_{\pi\in\Pi^{\mathrm{SD}}} g^\pi$, and let
$\pi^* \in \arg\max_{\pi\in\Pi^{\mathrm{SD}}} g^\pi$ denote an optimal policy.

\paragraph{Bellman operator and optimality equation.}
Define the Bellman operator $\cT_{P} : \R^{|\cS|\times |\cA|}\to \R^{|\cS|\times |\cA|}$ associated with the transition kernel $P$ by
\begin{equation}
\label{equ:bellman_operator}
\cT_P(Q)(s,a) := r(s,a) + \mathbb E\sqbkcond{\max_{a'\in \cA} Q(s', a')}{s,a},
\end{equation}
where $s'\sim p(\cdot | s,a)$. The optimal Q-function is defined as
\begin{equation*}
	Q^*(s,a) := \E^{\pi^*}\sqbkcond{\sum_{t=0}^\infty(r(s_t, a_t) - g^{\pi^*})}{s_0 = s, a_0 = a},
\end{equation*}
for all $(s,a)\in \cS\times \cA$. Then, $(Q^*, g^*)$ satisfies the Bellman equation
\begin{equation}
\label{equ:bellman_equation}
\cT_P(Q^*) - g^*\mathbf 1 = Q^*,
\end{equation}
where $\mathbf 1$ denotes the all-ones vector. An optimal policy $\pi^*\in \Pi^{\mathrm{SD}}$ can be obtained by acting greedily with respect to $Q^*$; that is,
\[
\pi^*(s)\in \argmax{a\in \cA} Q^*(s,a) \text{ for all } s\in \cS.
\] 
Note that the solution to~\eqref{equ:bellman_equation} is unique only up to an additive constant in $Q^*$, and the greedy policy induced by $Q^* + c\mathbf{1}$ is invariant to such shifts. Therefore, to achieve policy learning, it suffices to estimate $Q^*$ accurately up to an additive constant \citep{Puterman1994MDP}.

Due to the lack of discounting, existing analyses of AMDPs rely on the span seminorm $\spannorm{\cdot}$, defined by $sp(x):=\max_i x_i - \min_i x_i$ for $x\in \R^{d}$.



	\section{One-Step Contraction via Lazy Transformation}
\label{sec:one_step_contraction_via_lazy_transformation}

Unlike their discounted counterparts, the Bellman operator for AMDPs is not a contraction mapping under any norm.
As a consequence, existing results either rely on the non-expansive property of the Bellman
operator, which leads to suboptimal rates, or directly impose an assumption on the seminorm
contractivity of the Bellman operator, resulting in implicit and hard-to-verify conditions (Section \ref{sec:literature_review}). 
One of our main contributions is to identify and construct a novel \emph{instance-dependent seminorm} under which the lazy transformed Bellman operator becomes contractive.

Given the broader interest of this result, we devote this section to developing the contraction theory. This is accomplished through the following steps:
\begin{itemize}[leftmargin=*]
\item We first apply a simple \emph{lazy transform} to the transition kernel and obtain the corresponding Bellman operator. This lazy transform leads to a new Q-function that encodes the same optimal policies as the untransformed version.
\item We then introduce an instance-dependent seminorm $\widetilde{sp}$ that shares the same null space as the span seminorm. We further show that the new Bellman operator is a one-step contraction under $\widetilde{sp}$.
\end{itemize}



\subsection{Lazy Transformation}
\label{subsec:lazy_transformation}
Given a transition kernel $P$ and a parameter $\alpha \in (0, 1]$, define the lazy transition kernel $\overline P=\{\overline p(s'|s,a)\}_{(s,a,s')\in\cS\times\cA\times\cS}$ by
\begin{equation}
\label{eq:def_lazy_kernel}
\overline p(s'| s,a) :\;=\; \begin{cases}
		(1-\alpha)+\alpha \, p(s| s,a) & s'=s, \\ 
        \alpha\, p(s'| s,a), & s'\neq s.
	\end{cases}
\end{equation}
Intuitively, the lazy transformation introduces a self-loop with probability $1-\alpha$ at every state, regardless of which action is used. Equivalently, at each transition, the controlled Markov chain remains in the current state with probability $1-\alpha$ and follows the original kernel $p(\cdot \mid s,a)$ with probability $\alpha$.

We note that the lazy transformation is a well-known technique for converting a periodic Markov chain into an aperiodic one without changing its stationary distributions. When applied in the MDP setting, it also preserves the optimal policy, as well as important structural properties of the $Q$ function. This is summarized in the following lemma. 


\begin{lemma}[Lazy Transformation]
    \label{lem:lazy_transformation}
    Let $\cM = (\cS, \cA, P, r)$ be an MDP. For any $\alpha \in (0,1]$, let $\overline{P}$ be defined as in~\eqref{eq:def_lazy_kernel}. Let $(g^*, Q^*)$ be a solution to the average-reward Bellman equation \eqref{equ:bellman_equation} under $P$. Then, for all $(s,a)\in \cS\times \cA$
	\[
    (\overline{g}^*, \overline{Q}^*(s,a)) := \bracket{g^*, Q^*(s,a)+\frac{1-\alpha}{\alpha}\max_{a'\in \cA}Q^*(s,a')}
    \]
	solves the Bellman equation under $\overline P$, i.e. $$\cT_{\overline P}(\overline Q^*) - \overline{g}^*\mathbf 1 = \overline Q^*. $$
    Moreover, the optimal policy set is preserved:
    \[
    \argmax{a\in \cA} Q^*(s,a) = \argmax{a\in \cA} \overline{Q}^*(s,a),\quad \forall s\in \cS.
    \]
\end{lemma}
In the sequel, we choose $\alpha=\tfrac{1}{2}$. The proof of Lemma~\ref{lem:lazy_transformation} is deferred to Appendix~\ref{sec:app:technical_results_proof}. Lemma~\ref{lem:lazy_transformation} establishes an explicit and invertible correspondence between the solutions of the Bellman equations under $P$ and its lazy counterpart $\overline P$. In particular, the optimal average reward is invariant, $\overline g^*=g^*$, and $\overline Q^*$ differs from $Q^*$ only by a state-dependent shift. Consequently, the greedy (optimal) action set is preserved at every state. Therefore, it suffices to learn $\overline Q^*$ and then transform back to $Q^*$. 
 
\subsection{Seminorm Contraction}
\label{subsec:seminorm_contraction}

Remarkably, \citet{CAVAZOSCADENA1998221} observed that lazy-transformed MDP kernels yield lower bounds on hitting probabilities. Building on their result, we show that under Assumption~\ref{ass:reachability}, the transformed kernel induces a contractive Bellman operator under a novel instance-dependent seminorm. 

Before proceeding, we emphasize that even under Assumption~\ref{ass:reachability}, the transformed kernel $\overline{P}$ still induces only a non-expansive Bellman operator $\cT_{\overline{P}}$ under the usual span seminorm $sp(\cdot)$. In particular, one-step Bellman updates do not reduce instantaneous span errors in expectation, which precludes the direct use of standard stochastic approximation (SA) analyses for fixed points of contraction operators.

To address this, we observe that the regularity implied by Assumption~\ref{ass:reachability} is inherently a multi-step property, whereas SA analyses of Q-learning typically rely on one-step contraction of the Bellman operator \citep{chen2025nonasymptotic}. We therefore construct a new seminorm $\widetilde{sp}(\cdot)$\footnote{The notation ``$\widetilde{sp}(\cdot)$'' emphasizes that this seminorm shares the same kernel as the standard span seminorm.} designed to capture the multi-step dynamics of the controlled Markov chain.

We consider Markovian but possibly non-stationary $\pi\in\Pi^{\mathrm{M}}$. Let $\overline{\mathbb{E}}^{\pi}$ denote the expectation on the path space induced by $\overline P$ and $\pi$. 
\begin{definition}[Instance-Dependent Seminorm]
    \label{def:problem_dependent_seminorm}
    Under Assumption~\ref{ass:reachability}, we define the seminorm $\spanstar{\cdot}$ as:
    \begin{equation*}
        \spanstar{Q}:= \max_{0 \le k \le K} 
        \sup_{\pi \in \Pi^{\mathrm{M}}}
        \beta^{-k}\spannorm{ Z_k^{\pi}},
    \end{equation*}
    where $\beta =\bracket{1-\tfrac{1}{K2^K}}^{1/(K+1)}<1$, and
    \begin{equation*}
        Z_k^{\pi}(s,a):= \overline{\mathbb{E}}^{\pi}\sqbkcond{Q(s_k,a_k)}{s_0=s, a_0=a}.
    \end{equation*}
\end{definition}

Proposition~\ref{prop:instance_dependent_seminorm_is_a_seminorm} below establishes that $\spanstar{\cdot}$ is a proper seminorm.

\begin{proposition}
    \label{prop:instance_dependent_seminorm_is_a_seminorm}
    Suppose that Assumption~\ref{ass:reachability} holds. Then, $\spanstar{\cdot}$, as defined in~\ref{def:problem_dependent_seminorm}, is a seminorm whose null space  $\{x:\spanstar{x}=0\}$  is the one-dimensional subspace spanned by the constant vector $\mathbf{1}$.
\end{proposition}
Our construction is inspired by extremal norm theory \citep{WIRTH200217}. We adapt these ideas--in particular, the multistage structure and the maximization step--to our setting to define the seminorm $\widetilde{sp}(\cdot)$, which underpins the main contraction and algorithm analysis results.

Specifically, unlike the standard span seminorm, which captures instantaneous variation, $\widetilde{sp}$ applied to a $Q$ vector can be interpreted as a worst-case discounted span of expected future $Q$-values, propagated by the dynamics under non-stationary policies and the lazy kernel $\overline{P}$. This yields a key error measure--an instance-dependent seminorm for quantifying Bellman error--that links the multi-step regularity in Assumption~\ref{ass:reachability} to a one-step contraction property. In particular, it leads to Theorem~\ref{thm:new_seminorm_having_contraction}, which establishes one-step contractivity of the Bellman operator under $\widetilde{sp}$.

\begin{theorem}[Contraction in $\widetilde{sp}$]
	\label{thm:new_seminorm_having_contraction}
	Suppose that Assumption \ref{ass:reachability} holds. The operator $\cT_{\overline{P}}$ is a $\beta$-contraction with respect to $\spanstar{\cdot}$, i.e., for all $Q_1, Q_2 \in \R^{|\cS||\cA|}$,
    \[
	\spanstar{\cT_{\overline{P}}(Q_1) - \cT_{\overline{P}}(Q_2)} \leq \beta \spanstar{Q_1 - Q_2}.
	\]
\end{theorem}

The proof of Theorem~\ref{thm:new_seminorm_having_contraction} is deferred to Appendix~\ref{subsec:app:proof_of_contraction_theorem}. To establish this result, we build on \citet{CAVAZOSCADENA1998221} and show that, under Assumption~\ref{ass:reachability}, the expected number of visits to the reference state $s^\dagger$ over $K$ consecutive steps is uniformly lower bounded under the original kernel $P$. We then apply the lazy transformation to deduce that, after exactly $K$ steps, $s^\dagger$ is reached with a uniformly lower-bounded probability from any initial state. Combined with a minorization-type argument \citep{wang2023optimalsamplecomplexityreinforcement}, these bounds imply contraction in the seminorm $\widetilde{sp}$.

	\section{Synchronous Q-learning}
\label{sec:sync_q_learning}

In this section, we present two synchronous Q-learning algorithms: an explicit lazy-sampling variant that directly simulates transitions according to $\overline{P}$, and an implicit variant that avoids lazy sampling while preserving the same expected Bellman update. As summarized in Theorem~\ref{thm:sync_q_learning_optimal_rate}, we establish that the two algorithms admit identical sample complexity upper bounds, up to universal constants.

\noindent\textbf{Synchronous Q-learning with Explicit Lazy Sampling (Algorithm~\ref{alg:sync_q_explicit_lazy_sampling}):} At iteration $t$, for every state--action pair $(s,a)$, the agent draws a sample
\begin{equation}
    \label{equ:sync_alg1_sample_update}
    \overline s_t(s,a) = 
        \begin{cases}
            s, & \text{with probability } \tfrac{1}{2}, \\[6pt]
            s'\sim p(\cdot \mid s,a), & \text{with probability } \tfrac{1}{2}.
        \end{cases}
\end{equation}
Define the \emph{explicit} empirical Bellman operator:
\begin{equation}
    \label{equ:sync_alg1_bellman_estimator}
    {\hatT}^{\mathrm{exp}}_{\overline{P}, t}(Q)(s,a) := r(s,a)+\max_{a'\in\cA}Q(\overline s_t(s,a),a').
\end{equation}
The algorithm updates all entries of the $Q$-function estimate synchronously via
\begin{equation}
    \label{equ:sync_alg1_q_update}
    Q_t(s,a)=(1-\lambda)Q_{t-1}(s,a)+\lambda\,{\hatT}^{\mathrm{exp}}_{\overline{P},t}(Q_{t-1})(s,a).
\end{equation}

\noindent\textbf{Synchronous Q-learning with Implicit Lazy Sampling (Algorithm~\ref{alg:sync_q_implicit_lazy_sampling}):}
At iteration $t$, for every state--action pair $(s,a)$, the agent first samples
\begin{equation}
    \label{equ:sync_alg2_sample_update}
    s_t(s,a)\sim p(\cdot\mid s,a).
\end{equation}
It then constructs the \emph{implicit} empirical Bellman operator:
\begin{align}
    {\hatT}^{\mathrm{imp}}_{\overline{P}, t}(Q)(s,a) :=& r(s,a) + \frac{1}{2}\bigl(\max_{a'\in \cA} Q(s,a')\notag\\
    & + \max_{a'\in \cA} Q(s_t(s,a), a')\bigr). \label{equ:sync_alg2_bellman_estimator}
\end{align}
The algorithm updates all entries of $Q_t$ synchronously via
\begin{equation}
    Q_t(s,a)=(1-\lambda)Q_{t-1}(s,a)+\lambda\,{\hatT}^{\mathrm{imp}}_{\overline{P},t}(Q_{t-1})(s,a). \label{equ:sync_alg2_q_update}
\end{equation}
\begin{remark}
    \label{rem:unbiased_estimators}
    Both ${\hatT}^{\mathrm{exp}}_{\overline{P}, t}$ and ${\hatT}^{\mathrm{imp}}_{\overline{P}, t}$ are unbiased estimators of the Bellman operator $\cT_{\overline{P}}$; i.e.
    \[E\sqbk{{\hatT}^{\mathrm{exp}}_{\overline{P}, t}(Q)}=\E\sqbk{{\hatT}^{\mathrm{imp}}_{\overline{P}, t}(Q)}=\cT_{\overline{P}}(Q)\] 
    for all $Q\in\R^{|\cS||\cA|}$. This is shown in Appendix~\ref{sec:app:sync_unbiased_proof}. 
\end{remark}

\noindent\textbf{Output Correction:}
For both Algorithm~\ref{alg:sync_q_explicit_lazy_sampling} and~\ref{alg:sync_q_implicit_lazy_sampling}, after $T$ iterations we obtain an estimate $Q_T$ for the lazy dynamics under $\overline P$. We then construct a corrected estimate of $Q^*$ under $P$ and an approximately optimal policy $\pi_T$ via
 \begin{align}
    Q^{\mathrm{corr}}_T(s,a) :=&\; Q_T(s,a)-\frac{1}{2}\max_{a'\in \cA}Q_T(s, a'),\label{equ:correction_of_q_function_recover}\\
    \pi_T(s) :=&\; \argmax{a\in \cA} Q^{\mathrm{corr}}_T(s,a), \label{equ:output_approximate_policy}
\end{align}
for all $(s,a)\in \cS \times \cA$. Note that \eqref{equ:correction_of_q_function_recover} does not change the greedy action sets. Consequently, \eqref{equ:output_approximate_policy} is equivalent to extracting a greedy policy directly from $Q_T$.

\begin{algorithm}[t]
    \begin{algorithmic}[1]
\STATE \textbf{Inputs:} stepsize $\lambda$, number of iterations $T$, initial estimate $Q_0=\mathbf{0}$.
\FOR{$t = 1,2,\cdots,T$}
    \FOR{each $(s,a) \in \mathcal{S} \times \mathcal{A}$}
        \STATE Sample $\overline s_t(s,a)$ according to~\eqref{equ:sync_alg1_sample_update}
        \STATE Compute $\hatT_{\overline{P}, t}^{\mathrm{exp}}(Q_{t-1})(s,a)$ and set $Q_t(s,a)$ according to~\eqref{equ:sync_alg1_bellman_estimator} and~\eqref{equ:sync_alg1_q_update}.
    \ENDFOR
\ENDFOR
\STATE \textbf{Return} the corrected estimate $Q_T^{\mathrm{corr}}$ and the greedy policy $\pi_T$ defined in~\eqref{equ:correction_of_q_function_recover} and~\eqref{equ:output_approximate_policy}.
\end{algorithmic}
\caption{Synchronous Q-learning with Explicit Lazy Sampling}
\label{alg:sync_q_explicit_lazy_sampling}
\end{algorithm}

\begin{algorithm}[t]
    \begin{algorithmic}[1]
\STATE \textbf{Inputs:} stepsize $\lambda$, number of iterations $T$, initial estimate $Q_0=\mathbf{0}$.
\FOR{$t = 1,2,\cdots,T$}
    \FOR{each $(s,a) \in \mathcal{S} \times \mathcal{A}$}
        \STATE Sample $s_t(s,a)$ according to~\eqref{equ:sync_alg2_sample_update}
        \STATE Compute $\hatT_{\overline{P}, t}^{\mathrm{imp}}(Q_{t-1})(s,a)$ and set $Q_t(s,a)$ according to~\eqref{equ:sync_alg2_bellman_estimator} and~\eqref{equ:sync_alg2_q_update}. 
    \ENDFOR
\ENDFOR
\STATE \textbf{Return} the corrected estimate $Q_T^{\mathrm{corr}}$ and the greedy policy $\pi_T$ defined in~\eqref{equ:correction_of_q_function_recover} and~\eqref{equ:output_approximate_policy}.
\end{algorithmic}
\caption{Synchronous Q-learning with Implicit Lazy Sampling}
\label{alg:sync_q_implicit_lazy_sampling}
\end{algorithm}

\begin{theorem}[Sample Complexity of Synchronous Q-learning]
	\label{thm:sync_q_learning_optimal_rate}
	Suppose Assumption \ref{ass:reachability} holds. Run Algorithm~\ref{alg:sync_q_explicit_lazy_sampling} or~\ref{alg:sync_q_implicit_lazy_sampling} with the constant stepsize:
	\[
	\lambda = \min\left \{1,\frac{K(K+1)2^{K}\ln T}{T} \right \}.
	\]
    There exists a universal constant $C_0 > 0$ such that if
    \[
    T \geq C_0 \frac{{\,2^{6K}\log^3 T}}{\varepsilon^2}\log\left(\frac{|\cS||\cA|T}{\xi}\right),
    \]
    then, with probability at least $1-\xi$, we have
    \[
	\spannorm{Q_T^{\mathrm{corr}} - Q^*} \leq \varepsilon, \text{ and }  g^* - g^{\pi_T} \leq \varepsilon, \text{ simultaneously.}
	\]
\end{theorem}
The proof of Theorem~\ref{thm:sync_q_learning_optimal_rate} is deferred to Appendix~\ref{sec:app_sync_main_theorem}. This theorem implies that synchronous Q-learning attains a total sample complexity of at most $\widetilde{O}(|\cS||\cA|\varepsilon^{-2})$ (up to factors depending on $K$ and logarithmic terms). In particular, the $\widetilde{O}(|\cS||\cA|\varepsilon^{-2})$ dependence is optimal \citep{pmlr-v139-jin21b} up to logarithmic factors. To our knowledge, this is the first finite-sample, last-iterate convergence guarantee for average-reward Q-learning that achieves the optimal $\widetilde{O}(\varepsilon^{-2})$ dependence without imposing any contraction assumption.

Our analysis is based on a recursive error decomposition that upper bounds the last-iterate $Q$-estimation error by three terms: an initial bias that decays geometrically, a cumulative bias that can be controlled by the $\spanstar{\cdot}$-contraction, and a martingale term whose last-iterate variance remains of constant order due to the use of constant stepsizes. Optimally balancing these three terms yields the desired result.
	
	\section{Asynchronous Q-learning}
\label{sec:async_Q_learning}
    In this section, we propose and analyze two asynchronous Q-learning algorithms, where the agent only observes a single Markovian trajectory generated by a fixed behavior policy $\behpi$. We choose $\behpi$ such that $\behpi(a|s)>0$ for all $(s,a)\in\cS\times\cA$. Let $P_{\behpi}$ denote the state-transition kernel induced by $\behpi$. Under Assumption~\ref{ass:reachability}, the Markov chain induced by $P_{\behpi}$ has exactly one recurrent class, and it admits a unique stationary distribution $\rho$ supported on the recurrent class $\cC\subseteq \cS$ with $\rho(s)>0$ for all $s\in\mathcal{C}$. Throughout, let $Q|_{\cC}\in\R^{|\cC||\cA|}$ denote the restriction of $Q$ to $\cC\times \cA$, i.e., $Q|_{\cC}(s,a)=Q(s,a)$ for all $(s,a)\in \cC\times \cA$.


\noindent\textbf{Adaptive stepsize for asynchronous Q-learning:}
Define the visitation counts and state-action dependent stepsizes by
\begin{align}
	N_t(s,a) \;&:=\; \sum_{i=0}^{t-1}\mathbb{I}_{\set{(s_i,a_i)=(s,a)}},\\
	\lambda_t(s,a) \;&:=\; \frac{\lambda^*}{N_t(s,a)+h},\label{equ:async_frequency_and_lr}
\end{align}
where $\lambda^*>0$ and $h\ge \lambda^*$ ensure $\lambda_t(s,a)\in(0,1]$ for all $t$ and all $(s,a)$, in line with the choice in \citet{chen2025nonasymptotic}. In contrast to synchronous Q-learning, the asynchronous setting requires stepsizes to be scaled by visitation counts in order to compensate for imbalances in the expected state-action sampling frequencies along a single trajectory. This rescaling is essential for ensuring convergence to the correct $Q$ target, as noted by \citet{chen2025nonasymptotic}.

Below, we present two variants of asynchronous Q-learning, with explicit and implicit sampling, respectively.

\noindent\textbf{Asynchronous Q-learning with Explicit Lazy Sampling (Algorithm~\ref{alg:async_q_explicit_lazy_sampling}):}
At iteration $t$, the agent is at state $s_{t-1}$, samples an action and generates the next state according to the lazy transition
\begin{align}
	a_{t-1}\sim& \pi_b(\cdot| s_{t-1}),\label{equ:async_alg3_action_update}\\
	s_{t} =& 
	\begin{cases}
		s_{t-1}, & \text{with probability } \tfrac{1}{2}, \\ s\sim p(\cdot | s_{t-1},a_{t-1}), & \text{with probability } \tfrac{1}{2}.
	\end{cases}\label{equ:async_alg3_sample_update}
\end{align}
Compute the explicit temporal-difference $\delta_t^{\mathrm{exp}}$ and update $Q_t$ via
\begin{align}
	&\delta_t^{\mathrm{exp}} = r(s_{t-1}, a_{t-1})\notag\\
	&\quad\quad+ \max_{a'\in \cA} Q_{t-1}(s_{t}, a')- Q_{t-1}(s_{t-1},a_{t-1}) \label{equ:async_alg3_delta_update}\\
	&Q_{t}(s,a) = Q_{t-1}(s,a)\notag\\
	&\qquad\qquad\quad+ \lambda_{t}(s,a)\mathbb I_{\set{(s,a) = (s_{t-1}, a_{t-1})}}\delta_t^{\mathrm{exp}}\label{equ:async_alg3_q_update}
\end{align}
\noindent\textbf{Asynchronous Q-learning with Implicit Lazy Sampling (Algorithm~\ref{alg:async_q_implicit_lazy_sampling}):}
Similar to the synchronous case, we can also design an implicit variant that avoids sampling from the lazy kernel as well. At iteration $t$, sample
\begin{equation}
	\label{equ:async_alg4_sample_update}
	a_{t-1}\sim \pi_b(\cdot| s_{t-1}), \quad s_{t}\sim p(\cdot|s_{t-1}, a_{t-1}).
\end{equation}
Compute the implicit temporal-difference $\delta_t^{\mathrm{imp}}$ and update $Q_t$ via
\begin{align}
	&\delta_t^{\mathrm{imp}}= r(s_{t-1}, a_{t-1}) + \frac{1}{2}\bigl(\max_{a'\in \cA} Q_{t-1}(s_{t-1}, a')\notag\\
	&\quad\quad+ \max_{a'\in \cA} Q_{t-1}(s_{t}, a')\bigr)- Q_{t-1}(s_{t-1},a_{t-1})\label{equ:async_alg4_delta_update}\\
	&Q_{t}(s,a)= Q_{t-1}(s,a)\notag\\
	&\qquad\qquad\quad+\lambda_{t}(s,a)\mathbb I_{\set{(s,a) = (s_{t-1}, a_{t-1})}}\delta_t^{\mathrm{imp}}. \label{equ:async_alg4_q_update}
\end{align}
\noindent\textbf{Output correction:}
As in the synchronous case, after $T$ iterations, we obtain an estimate $Q_T$ for the lazy AMDP. We then form the corrected estimate $Q_T^{\mathrm{corr}}$ and the greedy policy $\pi_T$ via~\eqref{equ:correction_of_q_function_recover} and~\eqref{equ:output_approximate_policy}.

\begin{algorithm}[t]
\begin{algorithmic}[1]
\STATE \textbf{Inputs:} parameter $\lambda^*$, number of iterations $T$, $Q_0 = \mathbf 0$, $s_0\in \cS$, behavior policy $\behpi$.
\FOR{$t = 1,\cdots,T$}
	\STATE{Take $a_{t-1}$, and collect $s_{t}$ according to~\eqref{equ:async_alg3_action_update},~\eqref{equ:async_alg3_sample_update}.
	}
	\STATE {Update stepsize $\lambda_t(s_{t-1}, a_{t-1})$ according to \eqref{equ:async_frequency_and_lr}.}
	\STATE{Compute explicit temporal difference $\delta_{t}^{\mathrm{exp}}$ and update $Q_{t}$ according to~\eqref{equ:async_alg3_delta_update} and ~\eqref{equ:async_alg3_q_update}.
    }
\ENDFOR
\STATE \textbf{Return} the corrected estimate $Q_T^{\mathrm{corr}}$ and the greedy policy $\pi_T$ defined in~\eqref{equ:correction_of_q_function_recover} and~\eqref{equ:output_approximate_policy}.
\end{algorithmic}
\caption{Asynchronous Q-learning with Explicit Lazy Sampling}
\label{alg:async_q_explicit_lazy_sampling}
\end{algorithm}

\begin{algorithm}[t]
\begin{algorithmic}[1]
\STATE \textbf{Inputs:} parameter $\lambda^*$, number of iterations $T$, $Q_0 = \mathbf 0$, $s_0\in \cS$, behavior policy $\behpi$.
\FOR{$t = 1,\cdots,T$}
	\STATE{Take $a_{t-1}$, and collect $s_{t}$ according to~\eqref{equ:async_alg4_sample_update}.
	}
	\STATE {Update stepsize $\lambda_t(s_{t-1}, a_{t-1})$ according to \eqref{equ:async_frequency_and_lr}.}
	\STATE{Compute implicit temporal difference $\delta_{t}^{\mathrm{imp}}$ and update $Q_{t}$ according to \eqref{equ:async_alg4_delta_update} and \eqref{equ:async_alg4_q_update}.}
\ENDFOR
\STATE \textbf{Return} the corrected estimate $Q_T^{\mathrm{corr}}$ and the greedy policy $\pi_T$ defined in~\eqref{equ:correction_of_q_function_recover} and~\eqref{equ:output_approximate_policy}.
\end{algorithmic}
\caption{Asynchronous Q-learning with Implicit Lazy Sampling}
\label{alg:async_q_implicit_lazy_sampling}
\end{algorithm}

Below we present the sample complexity results for the two algorithms, respectively. 
\begin{theorem}[Sample Complexity of Algorithm~\ref{alg:async_q_explicit_lazy_sampling}]
\label{thm:async_q_learning_optimal_rate_explicit}
Suppose Assumption~\ref{ass:reachability} holds and the behavior policy satisfies $\behpi(a|s)>0$ for all $(s,a)\in\cS\times\cA$.
Define
\[
p_\wedge \coloneqq \min_{(s,a)\in\cS\times\cA:\,\rho(s)\behpi(a\mid s)>0}\rho(s)\behpi(a\mid s).
\]
and set $\lambda^* \coloneqq 4K(K+1)2^K$.
There exists a constant $T_{C_1}$, independent of $\varepsilon$, 
such that if  
\[
T \ge \max\left\{T_{C_1},\frac{C_1\,2^{8K}|\cS|^3|\cA|^3\log(|\cS||\cA|)\log^2 T}{p_\wedge^2\,\varepsilon^2} \right\},
\]
we have 
\[
\E\sqbk{\spannorm{Q_T^{\mathrm{corr}}|_{\cC}-Q^*|_{\cC}}^2}\le \varepsilon^2,
\quad\!\text{and}\quad\!
\E\sqbk{g^*-g^{\pi_T}}\le \varepsilon,
\]
where $C_1$ is a universal constant.
\end{theorem}

\begin{theorem}[Sample Complexity of Algorithm~\ref{alg:async_q_implicit_lazy_sampling}]
\label{thm:async_q_learning_optimal_rate_implicit}
Under the setting of Theorem~\ref{thm:async_q_learning_optimal_rate_explicit}, further assume that the Markov chain induced by $P_{\behpi}$ is aperiodic. There exists a constant $T_{C_2}$ (independent of $\varepsilon$) such that if  
\[
T \ge \max\left\{T_{C_2},\frac{C_2\,2^{6K}|\cS|^3|\cA|^3\log(|\cS||\cA|)\log^2 T}{p_\wedge^2\,\varepsilon^2} \right\},
\]
we have
\[
\E\sqbk{\spannorm{Q_T^{\mathrm{corr}}|_{\cC}-Q^*|_{\cC}}^2}\le \varepsilon^2,
\quad\!\text{and}\quad\!
\E\sqbk{g^*-g^{\pi_T}}\le \varepsilon,
\]
where $C_2$ is a constant that depends only on the mixing time of the behavior chain  $P_{\behpi}$.
\end{theorem}

\begin{remark}
    Importantly, although the one-step contraction result (Theorem~\ref{thm:new_seminorm_having_contraction}) is established under the more intricate seminorm $\widetilde{sp}$, Theorems~\ref{thm:async_q_learning_optimal_rate_explicit} and~\ref{thm:async_q_learning_optimal_rate_implicit} measure the $Q$-estimation error in the usual span seminorm. This is achieved by showing the equivalence of the two seminorms up to dimension-independent numerical constants.
\end{remark}

The proofs of Theorem~\ref{thm:async_q_learning_optimal_rate_explicit} and~\ref{thm:async_q_learning_optimal_rate_implicit} are provided in Appendix~\ref{sec:app:asynchronous_case}. Since the average-reward $g^\pi$ of any stationary policy depends only on the stationary distribution over its induced recurrent class $\cC_{\pi}$, the optimal average reward $g^{*}$ is determined solely by states in the recurrent class $\cC_{\pi^*}$, and $\cC_{\pi^*}\subseteq \cC$. Then all states relevant for computing $g^*$ are contained in $\cC$. Consequently, from a policy learning perspective, it suffices to control the span error of $Q_T^{\mathrm{corr}}$ on $\cC$ in bounding the average-reward suboptimality $g^* - g^{\pi_T}$.

The two theorems imply a sample complexity upper bound of $\widetilde{O}(2^{8K}|\cS|^3|\cA|^2p_\wedge^{-2}\varepsilon^{-2})$ to achieve an $\varepsilon$-optimal policy as well as estimation of $Q^*$. Inspired by the approach in \citet{chen2025nonasymptotic}, we construct a Lyapunov function based on our novel instance-dependent seminorm $\widetilde{sp}$. Leveraging the contraction property of the lazy transformed Bellman operator under $\widetilde{sp}$ (c.f. Theorem \ref{thm:new_seminorm_having_contraction}), we establish the convergence rate in mean-squared $sp$-error.


It is worth noting that, compared to~\citet{chen2025nonasymptotic}, Theorem~\ref{thm:async_q_learning_optimal_rate_explicit} does not require the Markov chain induced by the behavior policy $\behpi$ to be irreducible or aperiodic. Theorem~\ref{thm:async_q_learning_optimal_rate_implicit} also relaxes irreducibility but still requires aperiodicity to ensure a finite-sample guarantee.

    \section{Numerical Experiments}
\label{sec:numerical_experiments}
We verify our theoretical results via numerical experiments on a periodic average-reward MDP for which the Bellman operator is not a one-step contraction under the usual span seminorm. Nevertheless, the MDP satisfies Assumption~\ref{ass:reachability} and therefore lies within the scope of our analysis. We consider an instance with $|\cS|=4$ states and $|\cA|=2$ actions, with transition kernel defined as follows:
\begin{align*}
&p(s_j|s_i,a_1)=p,\; (i,j)\!\in\!\{(1,3),(2,4),(3,1),(4,2)\},\\
&p(s_j|s_i,a_1)=1\!-\!p,\; (i,j)\!\in\!\{(1,4),(2,3),(3,2),(4,1)\},\\[3pt]
&p(s_j|s_i,a_2)=q,\; (i,j)\!\in\!\{(1,3),(2,4),(3,1),(4,2)\},\\
&p(s_j|s_i,a_2)=1\!-\!q,\; (i,j)\!\in\!\{(1,4),(2,3),(3,2),(4,1)\}.
\end{align*}
The transition diagram of this kernel is provided in Appendix~\ref{sec:app:numerical_experiments}. The reward function is $r(s_1,a_1)=r(s_1,a_2)=1$ and $r(s,a)=0$ for all other state--action pairs. For the asynchronous setting (Figure~\ref{fig:async}), the behavior policy $\behpi(\cdot| s)$ selects actions uniformly at random from $\cA$ for all $s\in\cS$.

\begin{figure}[ht]
  \centering
  \begin{subfigure}[b]{0.45\linewidth}
    \centering
    \includegraphics[width=\linewidth]{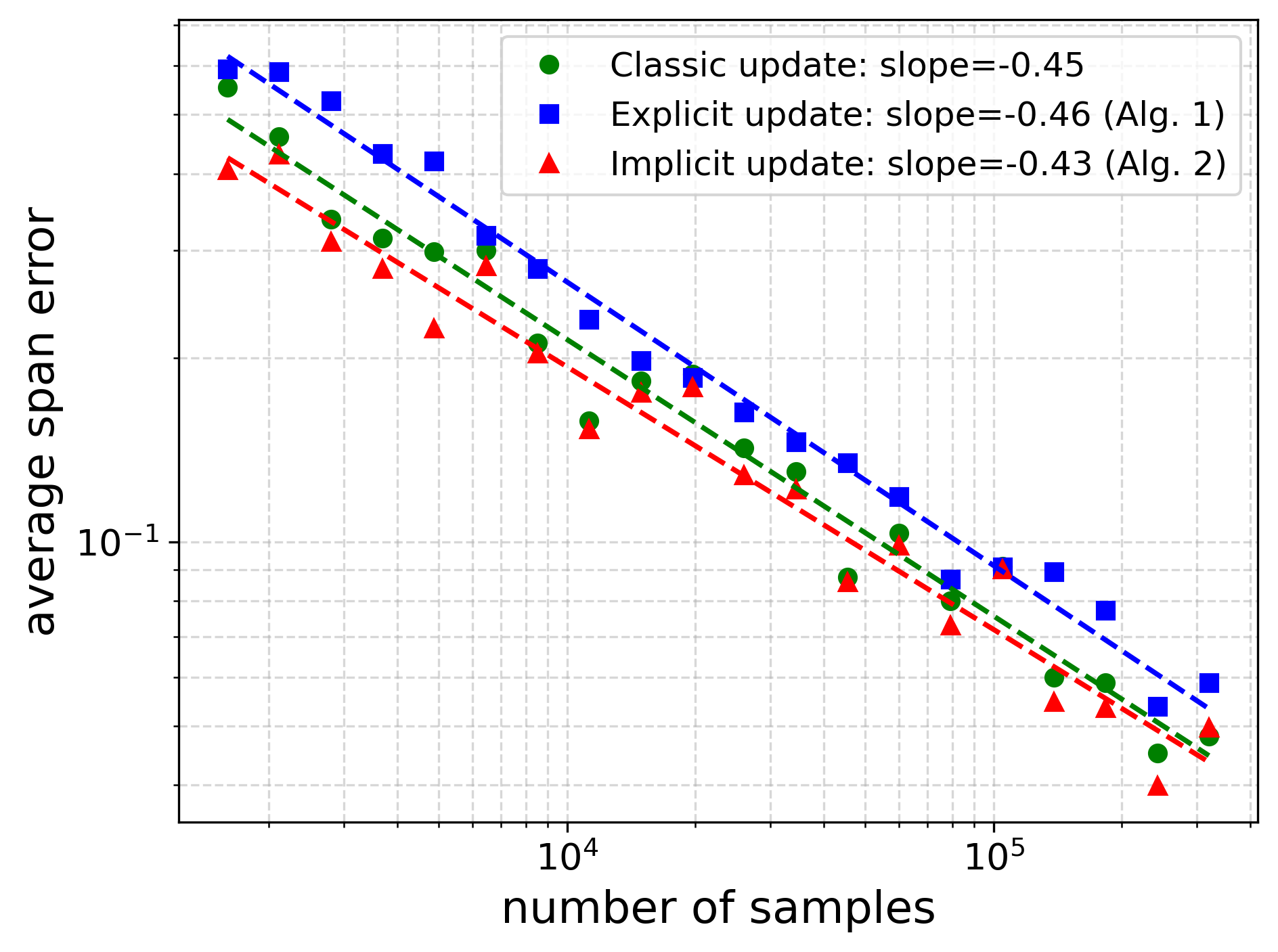}
    \caption{Synchronous case for Algorithms~\ref{alg:sync_q_explicit_lazy_sampling} and \ref{alg:sync_q_implicit_lazy_sampling}}
    \label{fig:sync}
  \end{subfigure}
  \begin{subfigure}[b]{0.45\linewidth}
    \centering
    \includegraphics[width=\linewidth]{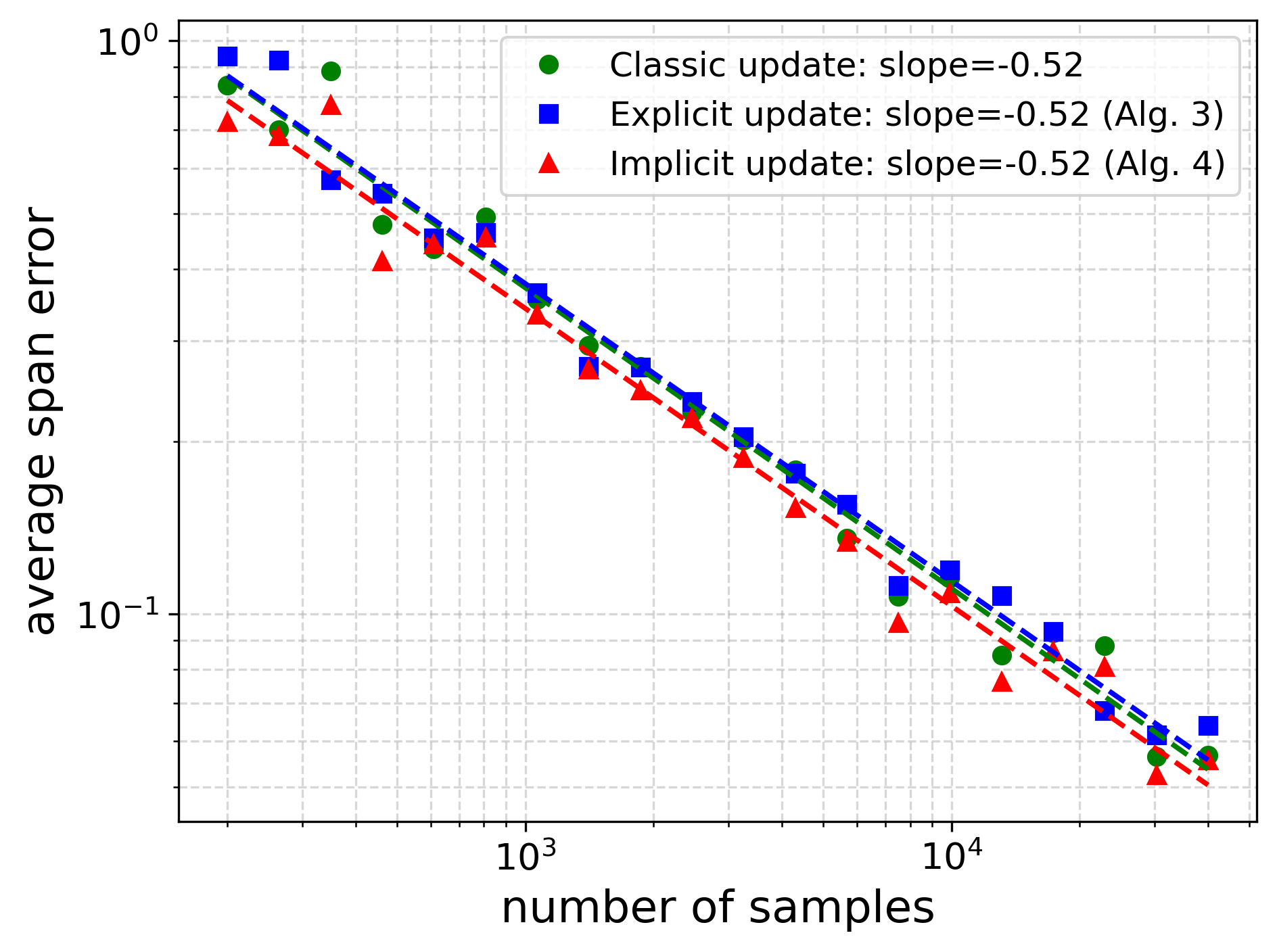}    
    \caption{Asynchronous case for Algorithms~\ref{alg:async_q_explicit_lazy_sampling} and~\ref{alg:async_q_implicit_lazy_sampling}}
    \label{fig:async}
  \end{subfigure}
  \caption{Convergence of lazy Q-learning algorithms with $p=0.3$ and $q=0.7$. The figures show the last-iterate span error against the sample complexity.}
  \label{fig:numerical_results}
\end{figure}

We evaluate the proposed algorithms using the last-iterate span error $\spannorm{Q_T^{\mathrm{corr}}-Q^*}$. Figure~\ref{fig:numerical_results} reports the averaged last-iterate span error over $10$ independent runs as a function of the total number of samples. We observe that the log-log regression slopes of the error curves are approximately $-1/2$ for all tested settings, indicating a convergence rate of $\widetilde O(T^{-1/2})$. This is consistent with our theoretical guarantees: to obtain an $\varepsilon$-accurate estimator of the Q-function, the iterate (and hence sample) complexity of the proposed algorithms scales as $\widetilde{O}(\varepsilon^{-2})$ in both the synchronous and asynchronous settings. Moreover, Algorithms~\ref{alg:sync_q_implicit_lazy_sampling} and~\ref{alg:async_q_implicit_lazy_sampling} seem to exhibit smaller last-iterate errors than classical Q-learning, indicating more favorable finite-sample performance on this instance.

    \section{Conclusion and Future Work}
\label{sec:conclusion}
In this paper, we establish the first non-asymptotic convergence rates for both synchronous and asynchronous lazy Q-learning without one-step contraction assumptions. Our analysis is built on a new contraction principle based on an instance-dependent seminorm $\widetilde{sp}$. This principle yields an optimal $\widetilde{O}(\varepsilon^{-2})$ dependence on the accuracy parameter for both synchronous and asynchronous lazy Q-learning. 
For future work, one direction is to improve the dependence on the hitting-time parameter $K$, potentially by developing sharper instance-dependent bounds or alternative constructions that better exploit the recurrence structure of the MDP. Another direction is to investigate whether the contraction principle can be extended to weaker structural assumptions, such as weakly communicating or multichain MDPs.
	
	\section*{Acknowledgement}
The work described in this paper was partially supported by a grant from the ANR/RGC Joint Research Scheme sponsored by the Research Grants Council of the Hong Kong Special Administrative Region, China and French National Research Agency (Project No. A-HKUST603/25). N. Si also acknowledges support from the Early Career Scheme (Grant No. 26210125) of the Hong Kong Research Grants Council.
	
	\bibliographystyle{apalike}
	\bibliography{reference.bib}
	
	\appendix
	\appendixpage
	
	\section{Notation and Basic Properties}
\label{sec:app:notation}
In this section, we summarize the notation and basic properties used throughout the technical proofs. We begin by presenting a hierarchy of MDP classes and assumptions in Figure~\ref{fig:hierarchy}, which illustrates the relationships among various MDP classes and assumptions commonly used in the literature.

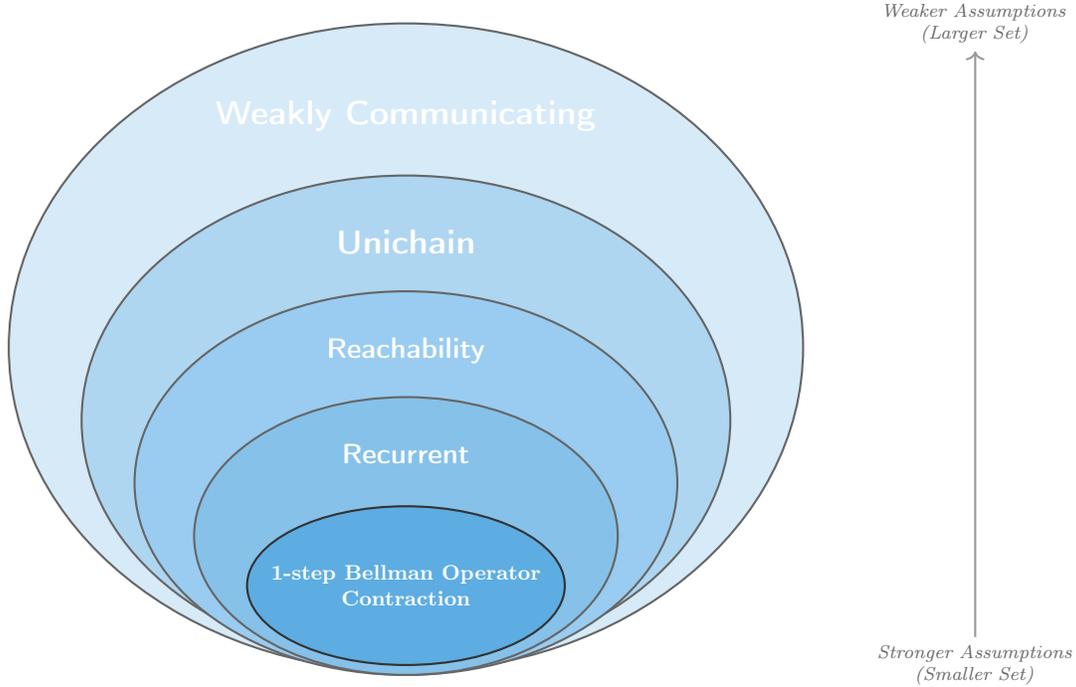
\begin{figure}[!htbp]
\begin{center}
\begin{tikzpicture}[font=\sffamily\bfseries, scale=0.88, transform shape]

    \definecolor{cWeak}{RGB}{214, 234, 248}
    \definecolor{cUni}{RGB}{174, 214, 241}
    \definecolor{cUR}{RGB}{153, 204, 240} 
    \definecolor{cIrr}{RGB}{133, 193, 233}
    \definecolor{cCont}{RGB}{93, 173, 226}

    \def\yTop{0.4}
    \def\aTop{6.0}
    \def\bTop{4.9}
    \def\yBot{-3.2}
    \def\bBot{1.2}

    \pgfmathsetmacro{\YTOP}{\yTop + \bTop}
    \pgfmathsetmacro{\YBOT}{\yBot - \bBot}

    \draw[fill=cWeak, draw=black!60, thick] (0, \yTop) ellipse (\aTop cm and \bTop cm);
    \node[text=white] at (0, 3.9) {\Large Weakly Communicating};

    \draw[fill=cUni, draw=black!60, thick] (0, -0.7) ellipse (4.9cm and 3.7cm);
    \node[text=white] at (0, 2.0) {\Large Unichain};

    \draw[fill=cUR, draw=black!60, thick] (0, -1.65) ellipse (4.1cm and 2.9cm);
    \node[text=white] at (0, 0.35) {\large Reachability};

    \draw[fill=cIrr, draw=black!60, thick] (0, -2.45) ellipse (3.2cm and 2.1cm);
    \node[text=white] at (0, -1.2) {\large Recurrent};

    \draw[fill=cCont, draw=black!80, thick] (0, \yBot) ellipse (2.4cm and \bBot cm);
    \node[text=white, align=center, font=\small\bfseries] at (0, \yBot)
        {1-step Bellman Operator\\Contraction};

    \def\xArrow{8.6}

    \node (strong) at (\xArrow, \YBOT) [text width=3cm, align=center,
        font=\footnotesize\itshape, text=black!60]
        {Stronger Assumptions\\(Smaller Set)};

    \node (weak) at (\xArrow, \YTOP) [text width=3cm, align=center,
        font=\footnotesize\itshape, text=black!60]
        {Weaker Assumptions\\(Larger Set)};

    \draw[->, thick, black!40] (strong) -- (weak);

\end{tikzpicture}
\end{center}
\caption{A hierarchy of MDP classes and assumptions.}
\label{fig:hierarchy}
\end{figure}
This is the list of frequently used notations throughout the paper:
\begin{itemize}
	\item $\mathbb I_{\set{\cdot}}$:indicator function
	\item $\mathbf{1}$: all-ones vector with adaptive dimension
	\item $\mathbb I_{s}$: indicator vector with $1$ at state $s$ and $0$ otherwise
	\item $\norm[p]{\cdot}$: $l_p$-norm defined as $\norm[p]{\mu} = \bracket{\sum_{i} |\mu_i|^p}^{1/p}$ for $p\geq 1$
	\item $\norm[p,*]{\cdot}$: dual norm of $l_p$-norm defined as $\norm[p,*]{\mu} = \sup_{\nu}\set{\mu^\top \nu|\norm[p]{\nu}\leq 1}$ for $p\geq 1$
	\item $\spannorm{\cdot}$: span seminorm defined as $\spannorm{\mu} = \max_{i} \mu_i - \min_{i} \mu_i$ 
	\item $\spanstar{\cdot}$: span-envelope seminorm defined later, this seminorm is the key to our analysis, where the Bellman operator is a single step contraction under this seminorm.
	\item $\cC$: recurrent class induced by the behavior policy $\behpi$.
	\item $\cC_{\pi}$: recurrent class induced by the policy $\pi$.
	\item $\Pi^{\mathrm{SD}}$, $\Pi^{\mathrm{S}}$, $\Pi^{\mathrm{M}}$, $\Pi^{\mathrm{HD}}$: the set of stationary deterministic policies, stationary policies, Markovian policies and history dependent policies respectively.
\end{itemize}

Given a policy $\pi$, the associated relative value function and $Q$-function of $\pi$ are defined, respectively, by
\begin{align*}
	V^\pi(s) := & \E^\pi \sqbkcond{\sum_{t=0}^\infty(r(s_t, a_t) - g^\pi)}{s_0 = s}\\
	Q^\pi(s,a) :=& \E^\pi\sqbkcond{\sum_{t=0}^\infty(r(s_t, a_t) - g^\pi)}{s_0 = s, a_0 = a},
\end{align*}

We introduce the matrix form $\Pp \in \R^{|\cS||\cA| \times |\cS|}$ to represent the probability transition kernel $\Pp$, where the $(s,a)$-th row corresponds to the transition probability distribution $p(\cdot | s,a)$ when taking action $a$ at state $s$. Additionally, for a stationary policy $\pi \in \Pi^{\mathrm{S}}$, we define the policy matrix $\bPi^\pi \in [0,1]^{|\cS| \times |\cS||\cA|}$ as a row-stochastic matrix, where each row corresponding to state $s$ is given by the action distribution $\pi(\cdot| s)$.
\[
\bPi^\pi=\left[\begin{array}{cccc}
\pi(\cdot | s_1)^{\top} & & & \\
& \pi(\cdot | s_2)^{\top} & & \\
& & \ddots & \\
& & & \pi(\cdot | s_{|\cS|})^{\top}
\end{array}\right]
\]

We abbreviate $\bPi^\pi$ as $\bPi$ when there is no ambiguity. Thus, the transition kernel under policy $\pi$ can be represented as $\Pp_\pi := \bPi \Pp \in \R^{|\cS|\times |\cS|}$, and $\Pp^\pi := \Pp \bPi \in \R^{|\cS||\cA|\times |\cS||\cA|}$.

Moreover, $Q_t$ denotes the estimated Q-function at iteration $t$ in the algorithms, and $\bPi_t, \bPi^* \in \Pi^{\mathrm{SD}}$ denote the greedy policy matrices with respect to $Q_t$ and $Q^*$, respectively. That is,
\[
\bPi_t=\left[\begin{array}{cccc}
\mathbb I_{\pi_{Q_t}(s_1)}^{\top} & & & \\
& \mathbb I_{\pi_{Q_t}(s_2)}^{\top} & & \\
& & \ddots & \\
& & & \mathbb I_{\pi_{Q_t}(s_{|\cS|})}^{\top}
\end{array}\right]
\qquad \text{and }\quad 
\bPi^*=\left[\begin{array}{cccc}
\mathbb I_{\pi_{Q^*}(s_1)}^{\top} & & & \\
& \mathbb I_{\pi_{Q^*}(s_2)}^{\top} & & \\
& & \ddots & \\
& & & \mathbb I_{\pi_{Q^*}(s_{|\cS|})}^{\top}
\end{array}\right].
\]
where $\pi_{Q}(s) \in \argmax{a\in \cA} Q(s,a)$ and $\mathbb I_{a_i} = \mathbb I_i$. We say that $\bPi$ is a greedy policy matrix with respect to $Q$ if $\bPi = \bPi^{\pi}$ for some policy $\pi$ satisfying $\pi(s) \in \argmax{a\in \cA} Q(s,a)$ for all $s \in \cS$. For simplicity, we denote the value function during the update by
\[
V_t (s) := \max_{a\in \cA} Q_t (s,a)
\]
It is immediate that 
\[
V_t = \bPi_t Q_t \quad \text{and} \quad V^* = \bPi^* Q^*
\]
Then, we define the empirical kernel $\Pp_t \in \set{0,1}^{|\cS||\cA|\times |\cS|}$ at the $t$-th iteration as the empirical transition kernel constructed from the samples collected at iteration $t$:
\[
\Pp_t((s,a), s') =
\begin{cases}
    1 & \quad \mathrm{if} \quad s' = s_t(s,a),\quad s_t(s,a) \sim p(\cdot|s,a), \\
    0 & \quad \mathrm{otherwise}.
\end{cases}
\quad 
\Pp_t =
\begin{bmatrix}
	\mathbb I_{s_t\bracket{s_1, a_1}}^\top \\
    \vdots \\
    \mathbb I_{s_t\bracket{s_{|\cS|}, a_{|\cS|}}}^\top
\end{bmatrix}.
\]
Denote the lazy transition kernel as 
\[
\overline{\Pp} := \bracket{1-\alpha}\Kk + \alpha \Pp,
\]
where $\Kk$ is the Kronecker delta kernel defined as:
\[
\Kk\bracket{(s,a),s'} = \delta_{s,s'} = \begin{cases}
	1 & \quad \mathrm{if} \quad s' = s, \\
	0 & \quad \mathrm{otherwise}
\end{cases}
\]
Then the empirical lazy kernel at the $t$-th iteration is defined as:
\[
\overline{\Pp}_t((s,a), s') =
\begin{cases}
    1 & \quad \mathrm{if} \quad s' = s_t(s,a),\quad s_t(s,a) \sim \overline{p}(\cdot|s,a), \\
    0 & \quad \mathrm{otherwise}.
\end{cases}
\]

Furthermore, the span of $Q^*$ can be bounded as follows.
\begin{lemma}
    \label{lem:app:q_star_span_bound}
Suppose Assumption~\ref{ass:reachability} holds and $\|r\|_\infty \le 1$. Then $\spannorm{Q^*} \le 2K + 1$.
\end{lemma}

\begin{proof}
The optimal $Q$-function satisfies the average-reward Bellman equation
\[
Q^* + g^*\mathbf 1 = \cT_{P}(Q^*) = r + P \bPi^* Q^* = r + P V^*,
\]
where $g^*$ is a constant and $V^*$ is the optimal bias function. Taking span on both sides and using $\spannorm{f+c\mathbf 1}=\spannorm{f}$, $\spannorm{Pf}\le \spannorm{f}$, $\spannorm{r}\le 1$, and $\spannorm{V^*}\le 2K$ (Puterman, 1994), we obtain
\[
\spannorm{Q^*} \le \spannorm{r} + \spannorm{P V^*} \le 1 + 2K.
\]
This completes the proof.
\end{proof}

\subsection{Finiteness of the Hitting Time}
\label{subsec:app:finiteness_of_the_hitting_time}

\begin{lemma}[Uniform finiteness of the hitting time constant]
\label{lem:K_finite_under_reachability}
Suppose Assumption~\ref{ass:reachability} holds and $|\cS|<\infty$.
For any $\pi\in\Pi^{\mathrm{SD}}$ and $s\in\cS$, define the hitting time $T_s := \inf\{t\ge 0: s_t = s^\dagger\}$. Then $\mathbb{E}^{\pi}\sqbkcond{T_s}{s_0=s}<\infty$ for all $s\in\cS$. Consequently,
\[
K := \max_{\pi\in \Pi^{\mathrm{SD}},\, s\in \cS}
\mathbb{E}^{\pi}\sqbkcond{T_s}{s_0=s}<\infty.
\]
\end{lemma}

\begin{proof}
Fix $\pi\in\Pi^{\mathrm{SD}}$ and write $\mathbb{P}:=\mathbb{P}^\pi$, $\mathbb{E}:=\mathbb{E}^\pi$.
Assumption~\ref{ass:reachability} implies that $s^\dagger$ belongs to the unique recurrent class of the finite Markov chain induced by $\pi$; otherwise,
$\mathbb{P}\sqbkcond{s_t=s^\dagger}{s_0=s}=0$ for all $t$ and some $s$, contradicting reachability.
Hence, for every $s\in\cS$,
\[
\mathbb{P}\sqbkcond{T_s<\infty}{s_0=s}=1.
\]

Since the chain is finite and $s^\dagger$ is recurrent, there exist $M\ge 1$ and $p\in(0,1)$ such that
\[
\inf_{x\in\cS}\mathbb{P}\sqbkcond{s_M=s^\dagger}{s_0=x}\ge p .
\]
By the Markov property, for all $k\ge 0$,
\[
\mathbb{P}\sqbkcond{T_s>(k+1)M}{s_0=s}
\le (1-p)\,\mathbb{P}\sqbkcond{T_s>kM}{s_0=s},
\]
and hence $\mathbb{P}\sqbkcond{T_s>kM}{s_0=s}\le (1-p)^k$.
Using the tail-sum formula,
\[
\mathbb{E}\sqbkcond{T_s}{s_0=s}
\le \sum_{k=0}^\infty M\,\mathbb{P}\sqbkcond{T_s>kM}{s_0=s}
\le \sum_{k=0}^\infty M(1-p)^k
= \frac{M}{p}<\infty .
\]
Finally, since $|\cS|<\infty$ and $|\Pi^{\mathrm{SD}}|<\infty$, taking the maximum over $(\pi,s)$ yields $K<\infty$.
\end{proof}

	\section{Technical Results in Section~\ref{sec:one_step_contraction_via_lazy_transformation}}
\label{sec:app:technical_results_proof}

\subsection{Proof of Lemma \ref{lem:lazy_transformation}}
Denote
\[
V^*(s) := \max_{a\in \cA}Q^*(s,a),\quad\text{ and }\quad \overline{V}^*(s) := \max_{a\in \cA}\overline{Q}^*(s,a).
\]
From the construction, we have 
    \[
    \max_{a\in \cA}\overline Q^*(s,a) = \max_{a\in \cA}Q^*(s,a) + \frac{1-\alpha}{\alpha}V^*(s) = V^*(s) + \frac{1-\alpha}{\alpha} V^*(s) = \frac{1}{\alpha}V^*(s).
    \]
    For the right-hand side of the Bellman equation, we have
	\begin{equation}
		\label{equ:app:part_b_bellman_equation}
		\begin{aligned}
			&r(s,a) + \sum_{s'}\overline{p}(s'|s,a)\max_{a'\in \cA}\overline{Q}^*(s', a')\\
			=& r(s,a) + \sum_{s'}((1-\alpha)\delta_{s, s'} + \alpha p(s'|s,a))\frac{1}{\alpha}V^*(s')\\
			=& r(s,a) + \frac{1-\alpha}{\alpha}V^*(s) + \sum_{s'}p(s'|s,a)V^*(s')\\
			\overset{(i)}{=}& g^* + \frac{1-\alpha}{\alpha}V^*(s) + Q^*(s,a)\\
			=& g^* + \overline{Q}^*(s,a).
		\end{aligned}
	\end{equation}
	Here, equality~(i) follows from the Bellman equation for $\cM$, 
    \begin{equation*}
		g^* + Q^*(s,a) = r(s,a) + \sum_{s'}p(s'|s,a)V^*(s')
	\end{equation*}
	Thus, we conclude that $(\overline g^*, \overline Q^*)$ satisfies the Bellman equation~\eqref{equ:app:part_b_bellman_equation} for the MDP $\overline{\cM}$ with lazy kernel $\overline P$:
	\begin{equation*}
		\overline g^* + \overline Q^*(s,a) = r(s,a) + \sum_{s'}\overline p(s'|s,a)\max_{a'\in \cA}\overline Q^*(s', a').
	\end{equation*}
    Further, $\overline{Q}^*$ and $Q^*$ can be related as:
    \begin{lemma}
	\label{lem:inverse_between_q_functions}
	Suppose $Q^*$ and $\overline{Q}^*$ are the optimal Q-functions for MDPs $\cM$ and $\overline{\cM}$, respectively. Then
	\[
	Q^* = \bracket{\Ii - (1-\alpha)\Kk \bPi^{\overline{\pi}^*}} \overline{Q}^*,
	\]
	where $\overline{\pi}^*$ is the greedy policy for $\overline{Q}^*$.
\end{lemma}
    \begin{proof}
        For all policy matrix $\bPi$: $\bPi \Kk = \Ii$, which further implies:
		\begin{align*}
    \bracket{\Ii + \frac{1-\alpha}{\alpha}\Kk\bPi}\bracket{\Ii - \bracket{1-\alpha}\Kk\bPi} =& \Ii - (1-\alpha)\Kk\bPi + \frac{1-\alpha}{\alpha}\Kk\bPi - \frac{(1-\alpha)^2}{\alpha}\Kk\bPi \Kk\bPi\\
    =& \Ii - \sqbk{\bracket{1-\alpha} - \frac{1-\alpha}{\alpha} + \frac{(1-\alpha)^2}{\alpha}}\Kk\bPi\\
    =& \Ii\\
	\bracket{\Ii + \frac{1-\alpha}{\alpha}\Kk\bPi}^{-1} =& \Ii - (1-\alpha)\Kk\bPi.
\end{align*}
        Then from Lemma \ref{lem:lazy_transformation}, in matrix product form: 
		\[
		\overline{Q}^* = (\Ii + \frac{1-\alpha}{\alpha}\Pp \bPi^{\overline{\pi}^*})Q^*,
		\]
		and:
        \[
        Q^* = \bracket{\Ii + \frac{1-\alpha}{\alpha}\Kk \bPi^{\overline{\pi}^*}}^{-1} \overline{Q}^*= \bracket{\Ii - (1-\alpha)\Kk \bPi^{\overline{\pi}^*}} \overline{Q}^*.
        \]
        This completes the proof.
\end{proof}

\subsection{Proof of Proposition~\ref{prop:instance_dependent_seminorm_is_a_seminorm} and Theorem~\ref{thm:new_seminorm_having_contraction}}
\label{subsec:app:proof_of_contraction_theorem}

Before proving Theorem~\ref{thm:new_seminorm_having_contraction}, we first present the property of the nominal kernel $P$ and the lazy kernel $\overline{P}$ that lay the foundation of our new seminorm construction.

\begin{lemma}
	\label{lem:better_bound_on_span_hat_h}
    Under Assumption \ref{ass:reachability}, consider the modified MDP $\hat \cM = (\cS, \cA, P, \hat r)$ with $\hat r (s,a) = -\mathbb I_{\set{s=s^\dagger}}$, let $(\hat g^*, \hat V^*)$ be the optimal average reward and relative value function for $\hat \cM$:
	\[
	\hat V^*(s) + \hat g^* = \max_{a\in \cA}\set{\hat r(s,a) + \sum_{s'\in \cS}p(s'|s,a)\hat V^*(s')},
	\]
	then $\hat V^*$ satisfies:
    \[
    sp(\hat V^*) \leq -\hat g^* K.
    \]
\end{lemma}
\begin{proof}
	Note that $\hat g^* \le 0$ since $\hat r(s,a) \le 0$ for all $(s,a)$ and $\hat r(s^\dagger,a) < 0$ for all $a\in\cA$. Let $\hat \pi^*$ be an optimal stationary deterministic policy for $\hat \cM$ (since $\hat \cM$ is an unichain, by \citet{Puterman1994MDP}, such optimal policy $\hat \pi^*\in \Pi^{\mathrm{SD}}$ always exist). Define:
    \[
    \hat f(s):=\hat V^*(s)-\hat V^*(s^\dagger),\quad\text{so that}\quad\hat f(s^\dagger)=0\text{ and }sp(\hat f)=sp(\hat V^*).
    \]
    Let hitting time to $s^\dagger$ be $T_{s} = \inf\set{t > 0 | s_t = s^\dagger, s_0 = s }$, then we have:
	\begin{align*}
		\hat f(s) \overset{(i)}{=}& \hat V^*(s) - \hat V^*(s^\dagger)\\
        =& \mathbb E^{\hat\pi^*}\sqbkcond{\sum_{t=0}^{T_{s}-1}(\hat r(s_t, a_t) - \hat g^*)}{s_0 = s} + \mathbb E^{\hat\pi^*}\sqbkcond{\sum_{t=T_{s}+1}^\infty (\hat r(s_t, a_t) - \hat g^*)}{s_0 = s^\dagger} - \hat V^*(s^\dagger)\\
		=& \mathbb {E}^{\hat\pi^*}\sqbkcond{\sum_{t=0}^{T_{s}-1}(\hat r(s_t, a_t) - \hat g^*)}{s_0 = s} + \hat V^*(s^\dagger) - \hat V^*(s^\dagger)\\
		=& \mathbb {E}^{\hat\pi^*}\sqbkcond{\sum_{t=0}^{T_{s}-1}(\hat r(s_t, a_t) - \hat g^*)}{s_0 = s},
	\end{align*}
	where the equality~$(i)$ follows from the definition of the relative value function and the strong Markov property at the hitting time $T_s$. Since $T_s$ is the first hitting time on $s^\dagger$, and $\hat r (s,a) = -\mathbb I_{\set{s=s^\dagger}}$, when $s\neq s^\dagger$, when $t = 0, \cdots, T_s -1$, the Markov Chain is not at $s^\dagger$, thus $\hat r(s_t, a_t) = 0$. Therefore:
	\[
	0\leq\hat f(s) = \mathbb E^{\hat\pi^*}\sqbkcond{\sum_{t=0}^{T_{s}-1} -\hat g^*}{s_0 = s} = -\hat g^* \mathbb E^{\hat\pi^*}\sqbkcond{T_{s}}{s_0 = s} \leq -\hat g^* K.
	\]
	When $s = s^\dagger$, we have $\hat f(s^\dagger) = 0$, $\inf_{s\in \cS} \hat f(s) = 0$. Thus:
	\[
	sp(\hat V^*) = sp(\hat f) = \sup_{s\in \cS} \hat f(s) - \inf_{s\in \cS} \hat f(s) \leq (-\hat g^* K) - 0 = - \hat g^* K.
	\]
	This completes the proof.
\end{proof}

\begin{lemma}[\citet{CAVAZOSCADENA1998221}]
	\label{lem:visitation_lower_bound}
	Suppose Assumption \ref{ass:reachability} holds. Then, for any \textbf{Markovian} policy $\pi\in \Pi^{\mathrm{M}}$ (no need to be deterministic), if the horizon $N$ satisfies $N \geq K$, we have:
	\[
		\frac{1}{N+1}\sum_{t=0}^{N} \mathbb{P}^\pi\sqbkcond{s_t = s^\dagger}{s_0=s} \geq \frac{1}{K(K+1)}, \quad \forall s \in \mathcal{S}.
	\]
\end{lemma}
\begin{proof}
	Consider a new MDP $\hat{\cM} = (\cS, \cA, P, \hat r)$, that is same in Lemma \ref{lem:better_bound_on_span_hat_h}, where the transition kernel is unchanged, and the reward is $\hat r(s,a) = -\mathbb{I}_{\set{s=s^\dagger}}$. Let $(\hat g^*, \hat V^*)$ be the solution to the average reward Bellman equation:
	\[
	\hat g^* + \hat V^*(s) = \max_{a\in \cA} \set{\hat r(s,a) + \sum_{s'\in \cS}p(s'|s,a)\hat V^*(s')}.
	\]
	Let $\hat\pi^*$ be the optimal policy for $\hat{\cM}$. Since $\hat{\cM}$ is unichain, there exists $\hat g^*\in [-1,0]$ where the optimal average reward $\hat g^*(s) \equiv \hat g^*$ is constant for all $s\in \cS$. By standard MDP theory \citep{Puterman1994MDP}, $\hat{g}^*$ corresponds to the maximum long-term average reward, and there exists a stationary and deterministic policy $\hat{\pi}^*$ such that $\hat{g}^* = g^{\hat{\pi}^*}$. By the definition of $K$ in Assumption~\ref{ass:reachability}, for any stationary deterministic policy the expected return time to $s^\dagger$ is at most $K$. Let $T_s$ as the first hitting time to $s^\dagger$ starting at $s$ then, under an optimal stationary and deterministic policy $\hat{\pi}$, the visits to $s^\dagger$ form a renewal process, the average reward equals minus the reciprocal of the expected return time.
	\begin{equation}
		\label{equ:app:section_b_hat_g_star}
		\hat g^* = -\frac{1}{\mathbb E^{\hat\pi^*}\sqbk{T_{s}}} \leq -\frac{1}{K}.
	\end{equation}
	Furthermore, for any state $s\in \cS$, and $\set{s_0, s_1,\cdots, s_{N}}$ is the trajectory. Bellman optimality implies:
	\begin{equation*}
		\hat g^* + \hat V^*(s) \geq \hat r (s,a) + \sum_{s'\in \cS}p(s'| s,a) \hat V^*(s').
	\end{equation*}
	Consider policy $\pi\in\Pi^{\mathrm{M}}$:
	\[
	\hat g^* + \hat V^*(s_t) \geq  \hat r(s_t, a_t) + \mathbb E^{\pi}\sqbkcond{\hat V^*(s_{t+1})}{s_t, a_t}.
	\]
	And further Taking the expectation on both sides with respect to the trajectory generated by $\pi$ starting from $s_0=s$:
	\[
	\hat g^* + \mathbb E^{\pi}\sqbkcond{\hat V^*(s_t)}{s_0=s} \geq \mathbb E^{\pi} \sqbkcond{\hat r(s_t, a_t) + \hat V^*(s_{t+1})}{s_0=s}.
	\]
	Summing over $t=0,1,\cdots, N$:
	\begin{align*}
		\sum_{t=0}^{N} \bracket{\hat g^* + \mathbb E^{\pi}\sqbkcond{\hat V^*(s_t)}{s_0=s}} \geq& \sum_{t=0}^N \mathbb E^{\pi} \sqbkcond{\hat r(s_t, a_t) + \hat V^*(s_{t+1})}{s_0=s}\\
		(N+1)\hat g^* + \sum_{t=0}^N \mathbb E^{\pi}\sqbkcond{\hat V^*(s_t)}{s_0=s} \geq& \mathbb E^{\pi} \sqbkcond{\sum_{t=0}^N \hat r(s_t, a_t) + \sum_{t=1}^{N+1}\hat V^*(s_t)}{s_0=s}.
	\end{align*}
	By canceling the overlapping terms $\mathbb E_\pi \sqbkcond{\sum_{t=1}^N\hat V^*(s_t)}{s_0=s}$ on both sides, we obtain:
	\[
		(N+1)\hat g^* + \hat V^*(s) \geq \mathbb E_\pi\sqbkcond{\sum_{t=0}^N \hat r(s_t, a_t) + \hat V^*(s_{N+1})}{s_0=s}.
	\]
	Rearranging the terms yields:
	\begin{align}
		\frac{\mathbb E^{\pi}\sqbkcond{\sum_{t=0}^N\hat r(s_t, a_t)}{s_0=s}}{N+1} \leq& \hat g^* + \frac{\hat V^*(s) - \mathbb E^{\pi}\sqbkcond{\hat V^*(s_{N+1})}{s_0=s}}{N+1}\notag\\
		\leq& \hat g^* + \frac{sp(\hat V^*)}{N+1}\label{equ:app:section_b_1998}.
	\end{align}
	Substituting $\hat r(s_t, a_t) = -\mathbb{I}_{\set{s_t = s^\dagger}}$, and multiplying by $-1$ on both sides in~\eqref{equ:app:section_b_1998}, we have:
	\[
	\frac{1}{N+1}\sum_{t=0}^N \mathbb P^\pi\sqbkcond{s_t = s^\dagger}{s_0=s} \geq -\hat g^* - \frac{sp(\hat V^*)}{N+1}.
	\]
	Apply Lemma \ref{lem:better_bound_on_span_hat_h}, since $sp(\hat V^*) \leq -\hat g^* K$, then we have:
	\[
	\frac{1}{N+1}\sum_{t=0}^N \mathbb P^\pi\sqbkcond{s_t = s^\dagger}{s_0=s} \geq -\hat g^* - \frac{-\hat g^* K}{N+1} = -\hat g^* \bracket{1 - \frac{K}{N+1}}.
	\]
	With the selection of $N \geq K$ and $\hat g^* \leq -\tfrac{1}{K}$ by~\eqref{equ:app:section_b_hat_g_star}, we conclude that:
	\[
	\frac{1}{K+1}\sum_{t=0}^K \mathbb P^\pi\sqbkcond{s_t = s^\dagger}{s_0=s} \geq -\hat g^* \bracket{1 - \frac{K}{K+1}} = -\hat g^* \frac{1}{K+1} \geq \frac{1}{K(K+1)}.
	\]
	This completes the proof.
\end{proof}

\begin{lemma}[\citet{CAVAZOSCADENA1998221}]
	\label{lem:lazy_kernel_lower_bound}
	For any \textbf{Markovian} policy $\pi\in \Pi^{\mathrm{M}}$ and any $(x,y)\in \cS\times \cS$, we have, for all $t\geq 0$
	\[
	\overline{\mathbb{P}}^{\pi}\sqbkcond{s_t = y}{s_0 = x} \geq \sum_{i=0}^t (1-\alpha)^{t-i}\alpha^{i} \mathbb{P}^{\pi}\sqbkcond{s_i = y}{s_0 = x}
	\]
	where $\overline{\mathbb P}^{\pi}$ denotes the probability where the randomness is taken over lazy kernel $\overline{P}$ induced by policy $\pi$.
\end{lemma}

\begin{proof}

	Recall that $\overline{p}(s'|s,a) = (1-\alpha)\delta_{s,s'} + \alpha p(s'|s,a)$. We have
	\begin{align*}
		\overline{\mathbb{P}}^\pi[s_{t+1} = y | s_0 = x]
		=& \sum_{s\in\cS}\overline{\mathbb P}^{\pi}\sqbkcond{s_1 = y}{s_0=s}
		\overline{\mathbb P}^{\pi}\sqbkcond{s_t = s}{s_0=x} \\
		=& \sum_{s\in \cS}\bracket{(1-\alpha)\delta_{s,y}
		+ \alpha \mathbb{P}^{\pi}\sqbkcond{s_1 = y}{s_0=s}}
		\overline{\mathbb P}^{\pi}\sqbkcond{s_t = s}{s_0=x} \\
		=& (1-\alpha)\overline{\mathbb{P}}^\pi\sqbkcond{s_t = y}{s_0 = x}
		+ \alpha \sum_{s\in\cS}
		\mathbb{P}^{\pi}\sqbkcond{s_1 = y}{s_0 = s}
		\overline{\mathbb P}^\pi\sqbkcond{s_t = s}{s_0 = x}.
	\end{align*}
	We proceed by induction on $t$.
	\begin{itemize}
		\item \textbf{Base case.} When $t=0$,
		\[
		\overline{\mathbb{P}}^{\pi}[s_1 = y | s_0 = x]
		= (1-\alpha)\delta_{x,y} + \alpha \mathbb{P}^{\pi}[s_1 = y | s_0 = x]
		\geq \sum_{i=0}^1 (1-\alpha)^{1-i}\alpha^i
		\mathbb{P}^{\pi}[s_i = y | s_0 = x].
		\]

		\item \textbf{Induction step.} Suppose the claim holds for time $t$. Then
		\[
		\overline{\mathbb P}^\pi\sqbkcond{s_t = y}{s_0 = x}
		\geq \sum_{i=0}^t (1-\alpha)^{t-i}\alpha^i
		\mathbb P^\pi\sqbkcond{s_i = y}{s_0 = x}
		\geq \alpha^t \mathbb P^\pi\sqbkcond{s_t = y}{s_0 = x}.
		\]
	\end{itemize}	
	Therefore,
	\begin{align*}
			\alpha \sum_{s\in\cS} \mathbb{P}^{\pi}[s_{1}=y|s_0 = s]\overline{\mathbb P}^\pi \sqbk{s_t=s | s_0 = x} \geq& \alpha \sum_{s\in\cS}\alpha^t \mathbb P^\pi \sqbk{s_1=y | s_0 = s}\mathbb P^\pi \sqbk{s_t=s | s_0 = x}\\
			=& \alpha^{t+1}\mathbb{P}^\pi\sqbkcond{s_{t+1}=y}{s_0 = x}.
		\end{align*}
	Combining the above inequalities yields
	\begin{align*}
		\overline{\mathbb{P}}^\pi\sqbkcond{s_{t+1}=y}{s_0 = x} \geq & (1-\alpha)\sum_{i=0}^t (1-\alpha)^{t-i}\alpha^i \mathbb P^\pi \sqbkcond{s_i = y}{s_0 = x} + \alpha^{t+1}\mathbb P^\pi \sqbkcond{s_{t+1}=y}{s_0 = x}\\
		=&\sum_{i=0}^{t}(1-\alpha)^{t+1-i}\alpha^i\mathbb P^\pi \sqbkcond{s_i = y}{s_0 = x} + \alpha^{t+1}\mathbb P^\pi \sqbkcond{s_{t+1}=y}{s_0 = x}\\
		=& \sum_{i=0}^{t+1}(1-\alpha)^{t+1-i}\alpha^i \mathbb P^\pi \sqbkcond{s_i = y}{s_0 = x}.
	\end{align*}
	This holds for all $(x,y)\in \cS\times \cS$, completing the proof.
\end{proof}

Combining Lemma~\ref{lem:visitation_lower_bound} and Lemma~\ref{lem:lazy_kernel_lower_bound} with $\alpha = \frac{1}{2}$, we obtain that for any Markovian policy $\pi \in \Pi^{\mathrm{M}}$ and any $s\in\cS$, the $K$-step transition probability to $s^\dagger$ under the lazy kernel satisfies
\begin{equation}
\label{equ:lazy_kernel_recurrent_lowerbound}
\begin{aligned}
\overline{\mathbb P}^\pi \sqbkcond{s_K = s^\dagger}{s_0 = s}
\geq& \sum_{i=0}^K (1-\alpha)^{K-i}\alpha^i \mathbb P^{\pi}\sqbkcond{s_i = s^\dagger}{s_0 = s} \\
\geq& \frac{1}{2^K}\sum_{i=0}^K \mathbb P^{\pi}\sqbkcond{s_i = s^\dagger}{s_0 = s}
\geq \frac{1}{K2^K}.
\end{aligned}
\end{equation}
Recall that $\overline{\mathbb P}^\pi$ denotes the probability measure induced by the lazy transition kernel $\overline{P}$ under policy $\pi$.

The following lemma shows that the lazy transformation preserves the reachability property up to a constant factor.
\begin{lemma}[Lazy kernel doubles the expected hitting time]
	\label{lem:lazy_doubles_hitting_time}
	For the lazy kernel, the upper bound on the hitting time:
	\[
	\overline K := \max_{\pi\in \Pi^{\mathrm{SD}},\, s\in \cS}
    \overline{\mathbb{E}}^{\pi}\sqbkcond{\inf\set{t>0 \mid s_t = s^\dagger}}{s_0 = s}
	\]
	satisfies $\overline{K} = 2K$.
\end{lemma}
The proof is provided in Appendix~\ref{sec:app:proof_of_lem:lazy_doubles_hitting_time}.

\begin{lemma}
\label{lem:convex_combination_of_q_value}
Let $\bPi_1$, $\bPi_2$ denote greedy policy matrices corresponding to $Q_1$ and $Q_2$ respectively. Then there exists a policy $\widetilde{\pi}\in \Pi^{\mathrm{S}}$ with policy matrix $\bPi^{\widetilde{\pi}}$ such that
\[
\bPi_1 Q_1 - \bPi_2 Q_2 \;=\; \bPi^{\widetilde{\pi}}\,(Q_1 - Q_2).
\]
\end{lemma}

\begin{proof}
Fix $s\in\cS$, let $V_i(s):=\max_{a\in\cA} Q_i(s,a)$ and define
\[
\Delta(s,a):=Q_1(s,a)-Q_2(s,a).
\]
Pick any $a_1\in\arg\max_{a\in\cA}Q_1(s,a)$ and $a_2\in\arg\max_{a\in\cA}Q_2(s,a)$. Then
\[
V_1(s)-V_2(s)
= Q_1(s,a_1)-V_2(s)
\le Q_1(s,a_1)-Q_2(s,a_1)
= \Delta(s,a_1)
\le \max_{a\in\cA}\Delta(s,a),
\]
and
\[
V_1(s)-V_2(s)
= V_1(s)-Q_2(s,a_2)
\ge Q_1(s,a_2)-Q_2(s,a_2)
= \Delta(s,a_2)
\ge \min_{a\in\cA}\Delta(s,a).
\]
Therefore,
\[
V_1(s)-V_2(s)\in\Bigl[\min_{a\in\cA}\Delta(s,a),\ \max_{a\in\cA}\Delta(s,a)\Bigr].
\]
Hence there exists an action distribution $\widetilde{\pi}(\cdot| s)\in \Delta(\cA)$ such that
\[
V_1(s)-V_2(s)=\sum_{a\in\cA}\widetilde{\pi}(a| s)\,\Delta(s,a).
\]
For instance, if $\max_{a\in\cA}\Delta(s,a)>\min_{a\in\cA}\Delta(s,a)$ (the inequality is strict), choose any
$a_{\min}\in\arg\min_{a\in\cA}\Delta(s,a)$ and $a_{\max}\in\arg\max_{a\in\cA}\Delta(s,a)$ and set
\[
\widetilde{\pi}(a_{\min}| s)
=\frac{\max_{a'}\Delta(s,a')-(V_1(s)-V_2(s))}{\max_{a'}\Delta(s,a')-\min_{a'}\Delta(s,a')},
\qquad
\widetilde{\pi}(a_{\max}| s)
=\frac{(V_1(s)-V_2(s))-\min_{a'}\Delta(s,a')}{\max_{a'}\Delta(s,a')-\min_{a'}\Delta(s,a')},
\]
and $\widetilde{\pi}(a| s)=0$ for all other $a$. If $\max_a\Delta(s,a)=\min_a\Delta(s,a)$, then $\Delta(s,a)$ is constant in $a$ and any distribution $\widetilde{\pi}(\cdot| s)$ works.

Let $\bPi^{\widetilde{\pi}}$ be the policy matrix induced by $\widetilde{\pi}$. Then
\[
(\bPi_1 Q_1)(s)-(\bPi_2 Q_2)(s)
= V_1(s)-V_2(s)
= \sum_{a\in\cA}\widetilde{\pi}(a| s)\bigl(Q_1(s,a)-Q_2(s,a)\bigr)
= \bigl(\bPi^{\widetilde{\pi}}(Q_1-Q_2)\bigr)(s).
\]
Since $s$ is arbitrary, we conclude for some $\widetilde{\pi}$
\[
\bPi_1 Q_1 - \bPi_2 Q_2 \;=\; \bPi^{\widetilde{\pi}}\,(Q_1 - Q_2).
\]
Proved.
\end{proof}

Now, we are ready to show Proposition~\ref{prop:instance_dependent_seminorm_is_a_seminorm} and Theorem~\ref{thm:new_seminorm_having_contraction}.

\begin{proposition}[Restatement of Proposition~\ref{prop:instance_dependent_seminorm_is_a_seminorm}]
    $\spanstar{\cdot}$ defined in~\ref{def:problem_dependent_seminorm} is a seminorm whose null space consists precisely of constant vectors.
\end{proposition}
\begin{proof}
    From the definition, $\spanstar{\cdot}$ can be written as
	\begin{align*}
		\spanstar{Q}
		:=& \max_{0 \le k \le K} 
        \sup_{\pi \in \Pi^{\mathrm{M}}}
        \beta^{-k}\spannorm{\overline{\mathbb{E}}^{\pi}\sqbkcond{Q(s_k,a_k)}{s_0=\cdot, a_0=\cdot}} \\
		=& \max_{0\leq k\leq K}\max_{\pi_1, \cdots, \pi_k \in \Pi^{\mathrm{S}}}
        \beta^{-k}\spannorm{\prod_{i=1}^k (\overline{\Pp} \bPi^{\pi_i})Q},
	\end{align*}
	where the second equality follows since any Markovian policy can be represented as a sequence of stationary decision rules, and
	$\beta = \bracket{1-\frac{1}{K2^K}}^{1/(K+1)}$.
	The homogeneity follows for any $c \in \R$,
	\[
	\spanstar{c Q}
	= |c| \spanstar{Q}.
	\]
	Similarly, non-negativity and subadditivity follow directly from the definition of $\spannorm{\cdot}$.
	For the null space, when $Q = c \mathbf 1$ for some $c\in \R$, since
	$\prod_{i=1}^k(\overline{\Pp}\bPi^{\pi_i})$ is a stochastic matrix for all
	$0\leq k\leq K$ and $\pi_1, \cdots, \pi_k \in \Pi^{\mathrm{S}}$, we have
    \[
    \spannorm{\prod_{i=1}^k(\overline{\Pp}\bPi^{\pi_i}) Q}
    = c\,\spannorm{\prod_{i=1}^k(\overline{\Pp}\bPi^{\pi_i}) \mathbf 1}
    = 0,
    \]
    which implies $\spanstar{Q} = 0$.
    Conversely, if $\spanstar{Q} = 0$, since
    $\prod_{i=1}^k(\overline{\Pp}\bPi^{\pi_i})$ is a stochastic matrix and hence nonexpansive, we have
    \[
    0 = \spanstar{Q}
    \geq \spannorm{Q}
    \implies Q=c\mathbf 1.
    \]
    We have shown that $\spanstar{Q}=0$ if and only if $Q= c\mathbf 1$ for some $c\in \R$, and the null space consists precisely of constant vectors.
\end{proof}

We further show that the newly defined seminorm $\spanstar{\cdot}$ is equivalent to the standard span seminorm $\spannorm{\cdot}$.
\begin{lemma}
	\label{lem:new_seminorm_equivalent_to_span}
	$\spanstar{\cdot}$ is equivalent to $\spannorm{\cdot}$ with
	\[
	\spannorm{Q}\leq \spanstar{Q} \leq 2 \spannorm{Q},
	\]
	for all $Q\in \R^{|\cS||\cA|}$.
\end{lemma}
\begin{proof}
	Since for any row-stochastic matrix $\mathbf R$, the span seminorm is nonexpansive, i.e.,
	$\spannorm{\mathbf R Q}\leq \spannorm{Q}$ for all $Q\in \R^{|\cS||\cA|}$,
	let $\beta = \bracket{1-\frac{1}{K2^K}}^{1/(K+1)}$. We have
	\begin{align*}
	\spanstar{Q}
	=& \max_{0\leq k\leq K} \sup_{\pi_1, \cdots, \pi_k \in \Pi^{\mathrm{S}}}
	\beta^{-k}\spannorm{\prod_{i=1}^k\bracket{\overline{\Pp} \bPi^{\pi_i}} Q} \\
	\leq& \max_{0\leq k\leq K}\beta^{-k}\spannorm{Q} \\
	=& \beta^{-K}\spannorm{Q} \\
    \leq& \frac{1}{1-\frac{1}{K2^{K}}}\spannorm{Q}
    \leq 2\spannorm{Q},
	\end{align*}
	where the last inequality holds since $K\geq 1$.
	Moreover,
	\[
	\spanstar{Q} \geq \beta^{0}\spannorm{Q} = \spannorm{Q}.
	\]
	This establishes the seminorm equivalence between $\spanstar{\cdot}$ and $\spannorm{\cdot}$.
\end{proof}

\begin{theorem}[Restatement of Theorem~\ref{thm:new_seminorm_having_contraction}]
	Under Assumption~\ref{ass:reachability}, the Bellman operator $\cT_{\overline{P}}$ associated with $\overline{P}$,
\[
\cT_{\overline{P}}(Q)(s,a) = r(s,a) + \E_{s'\sim \overline p(\cdot|s,a)}[\max_{a'\in \Aa}Q(s', a')],
\]
is a $\beta$-contraction with respect to $\spanstar{\cdot}$, where
	\[
	\spanstar{\cT_{\overline{P}}(Q_1) - \cT_{\overline{P}}(Q_2)} \leq \beta \spanstar{Q_1 - Q_2},
	\]
	where $\beta = \bracket{1-\frac{1}{K2^K}}^{1/(K+1)}$
\end{theorem}
\begin{proof}
	From the definition, $\spanstar{\cdot}$ can be written as:
	\[
    \spanstar{Q} = \max_{0\leq k\leq K}\max_{\pi_1, \cdots, \pi_k \in \Pi^{\mathrm{S}}}\beta^{-k}\spannorm{\prod_{i=1}^k (\overline{\Pp} \bPi^{\pi_i})Q}
    \]
	For any $(s,a)\in \cS\times \cA$, let $\pi(s)\in\argmax{a\in \cA}Q(s,a)$, where $\pi$ denotes the stationary greedy policy induced by $Q$.,
	\[
	\E_{s'\sim \overline{p}(\cdot | s,a)}\sqbk{\max_{a' \in \Aa}Q(s', a')} = \sum_{s'}\overline{p}(s'| s,a)\max_{a'\in \Aa}Q(s', a') = \sum_{s'}\overline{p}(s'|s,a)Q(s', \pi(s')) = \overline{\Pp}\bPi^{\pi} Q(s,a)
	\]
	Thus, letting $\bPi_1$ and $\bPi_2$ denote the greedy policy matrices corresponding to $Q_1$ and $Q_2$, respectively, then
	\[
	\cT_{\overline{P}}(Q_1) = r + \overline \Pp\bPi_1 Q_1 \quad \text{and}\quad \cT_{\overline{P}}(Q_2) = r + \overline \Pp\bPi_2 Q_2
	\]
	and
	\[
    \cT_{\overline{P}}(Q_1) - \cT_{\overline{P}}(Q_2) = \overline \Pp\bPi_1 Q_1 - \overline \Pp\bPi_2 Q_2.
	\]
	then we have:
	\begin{align*}
		&\spanstar{\cT_{\overline{P}}(Q_1) - \cT_{\overline{P}}(Q_2)} \\
		=&\spanstar{\overline \Pp\bPi_1 Q_1 - \overline \Pp\bPi_2 Q_2}\\
		=& \max_{0\leq k\leq K}\max_{\pi_1, \cdots, \pi_k \in \Pi^{\mathrm{S}}}\beta^{-k}\spannorm{\prod_{i=1}^{k}(\overline \Pp\bPi^{\pi_i})\overline \Pp(\bPi_1 Q_1 - \bPi_2 Q_2)}\\
		=& \max\Biggl\{\max_{0\leq k\leq K-1}\max_{\pi_1, \cdots, \pi_k \in \Pi^{\mathrm{S}}}\beta^{-k}\spannorm{\prod_{i=1}^k(\overline \Pp\bPi^{\pi_i})\overline \Pp (\bPi_1 Q_1 - \bPi_2 Q_2)}, \\
		&\qquad\qquad\qquad\qquad\qquad\qquad\beta^{-K}\spannorm{\prod_{i=1}^K(\overline \Pp\bPi^{\pi_i})\overline \Pp(\bPi_1 Q_1 - \bPi_2 Q_2)}\Biggr\}\\
		\overset{(i)}{\leq}& \max\bigl\{\max_{0\leq k\leq K-1}\max_{\pi_1, \cdots, \pi_k \in \Pi^{\mathrm{S}}}\max_{\pi\in \Pi^{\mathrm{S}}}\beta^{-k}\spannorm{\prod_{i=1}^k(\overline \Pp\bPi^{\pi_i})\overline \Pp\bPi^\pi(Q_1 - Q_2)},\\
		&\qquad\qquad\qquad\qquad\qquad\qquad \beta^{-K}\spannorm{\overline \Pp\prod_{i=1}^K (\bPi^{\pi_i}\overline \Pp)(\bPi_1 Q_1 - \bPi_2 Q_2)}\bigr\}\\
		\overset{(ii)}{\leq}& \max\set{\max_{1\leq k\leq K} \max_{\pi_1, \cdots, \pi_k \in \Pi^{\mathrm{S}}}\beta\cdot \beta^{-k}\spannorm{\prod_{i=1}^k(\overline \Pp\bPi^{\pi_i})(Q_1 - Q_2)}, \beta^{-K}\cdot \beta^{K+1} \spannorm{Q_1 - Q_2}}\\
		=& \beta \spanstar{Q_1 - Q_2}
	\end{align*}
	The first term in inequality (i) is derived by Lemma \ref{lem:convex_combination_of_q_value} where there exists a $\widetilde{\pi}$ such that:
    \[
    \spannorm{\bPi_1 Q_1 - \bPi_2 Q_2} = \spannorm{\bPi^{\widetilde{\pi}}(Q_1 - Q_2)} \leq \max_{\pi \in \Pi^{\mathrm{S}}}\spannorm{\bPi^\pi (Q_1 - Q_2)}.
    \]
    And the second term in inequality (ii), let $\pi = \pi_{1:K}:= (\pi_1, \pi_2, \cdots, \pi_K)$, $\pi$ can be formulated by a markovian policy that depends only on the current timestamp and state, then equation \eqref{equ:lazy_kernel_recurrent_lowerbound} implies that for any $s \in \cS$,:
    \[
    \overline{\mathbb{P}}^{\pi_{1:K}}\sqbkcond{s_K = s^\dagger}{s_0 = s} = \prod_{i=1}^K(\bPi^{\pi_i}\overline{\Pp}) \bracket{s, s^\dagger} \geq \frac{1}{K2^K}.
    \]
    Combining the fact that $\overline{\bPi}$ is nonexpansive with respect to $\spannorm{\cdot}$,
    \begin{equation}
		\label{equ:app:span_iteration_contract}
		\spannorm{\overline{\Pp}\prod_{i=1}^K(\bPi^{\pi_i}\overline{\Pp})(\bPi_1 Q_1 - \bPi_2 Q_2)} \leq \bracket{1 - \frac{1}{K2^K}}\spannorm{\bPi_1 Q_1 - \bPi_2 Q_2} \leq \beta^{K+1} \spannorm{Q_1 - Q_2}.
	\end{equation}
	By equation~\ref{equ:app:span_iteration_contract}, we prove the second term in inequality (ii) holds. Thus, we prove the theorem.
\end{proof}
Combining the lazy-kernel visitation lower bounds with the convex-combination property of greedy policies, we establish that the lazy Bellman operator admits a strict one-step contraction under the instance-dependent seminorm $\spanstar{\cdot}$, thereby completing the proof of Proposition~\ref{prop:instance_dependent_seminorm_is_a_seminorm} and Theorem~\ref{thm:new_seminorm_having_contraction}.

This corollary is the matrix form of Theorem~\ref{thm:new_seminorm_having_contraction}.

\begin{corollary}
    \label{cor:lem:new_seminorm_q_function}
    For any $\pi \in \Pi^{\mathrm{S}}$, the operator $\overline{\Pp}\bPi^{\pi}$ is a contraction with respect to $\spanstar{\cdot}$ defined in Definition~\ref{def:problem_dependent_seminorm}, in the sense that
    \[
    \spanstar{\overline{\Pp}\bPi^\pi Q} \leq \bracket{1-\frac{1}{K2^K}}^{1/(K+1)} \,\spanstar{Q},
    \]
	for all $Q\in \R^{|\cS||\cA|}$.
\end{corollary}

\begin{proof}
    By Bernoulli's inequality, $(1-x)^r \leq 1-rx$ for all $x\in [0,1]$ and $r\leq 1$. Taking $x = \frac{1}{K2^K}$ and $r = \frac{1}{K+1}$ yields
    \[
    \bracket{1-\frac{1}{K2^K}}^{1/(K+1)} \leq 1-\frac{1}{K(K+1)2^K}.
    \]
    The remaining argument follows the same steps as in the proof of Theorem~\ref{thm:new_seminorm_having_contraction}.
\end{proof}

This corollary implies that for any stationary policy $\pi\in \Pi^{\mathrm{S}}$, the operator $\overline{\Pp}\bPi^\pi$ is a strict contraction under the seminorm $\spanstar{\cdot}$.

	\section{Synchronous Case, Proof of Theorem~\ref{thm:sync_q_learning_optimal_rate}}
\label{sec:app_sync_main_theorem}

\subsection{Proof of Remark~\ref{rem:unbiased_estimators}}
\label{sec:app:sync_unbiased_proof}
Obviously,
\[
\E\sqbk{{\hatT}^{\mathrm{exp}}_{\overline{P}, t}(Q)} = \cT_{\overline{P}}(Q).
\]
The operator ${\hatT}^{\mathrm{imp}}_{\overline{P}, t}$ is also unbiased since
\begin{align*}
    \E\sqbk{{\hatT}^{\mathrm{imp}}_{\overline{P}, t}(Q)}(s,a)
    =&\, r(s,a) + \frac{1}{2}\max_{a'\in \cA}Q(s,a')
    + \frac{1}{2}\E_{s'\sim p(\cdot|s,a)}\sqbk{\max_{a'\in \cA}Q(s', a')}\\
    =&\, r(s,a) + \E_{s'\sim \overline p(\cdot|s,a)}\sqbk{\max_{a'\in \cA}Q(s', a')}\\
    =&\, \cT_{\overline{P}}(Q).
\end{align*}

\subsection{Auxiliary Lemmas}
\label{sec:app_sync_auxiliary_lemma}
This section presents the detailed proofs of Theorem~\ref{thm:sync_q_learning_optimal_rate}. The analysis relies on Azuma’s inequality; for ease of presentation, we include a standard version here.

\begin{theorem}[Azuma’s inequality]
	\label{thm:azuma_inequality}
	Suppose that $Y_n = \sum_{i=1}^n X_i \in \R$, where $\{X_i\}$ is a martingale difference sequence with respect to the filtration $\{\cF_i\}$, i.e., $\E[X_i \mid \cF_{i-1}] = 0$. Further, suppose that $|X_i|\leq c_i$ almost surely for constants $c_i>0$ and all $1\leq i \leq n$. Then, for any $\tau > 0$,
	\[
	\mathbb P\sqbk{|Y_n| \geq \tau} \leq 2 \exp\bracket{-\frac{\tau^2}{2\sum_{i=1}^n c_i^2}}.
	\]
\end{theorem}

In Algorithms~\ref{alg:sync_q_explicit_lazy_sampling} and~\ref{alg:sync_q_implicit_lazy_sampling}, controlling the growth of the $Q$-function estimator $Q_t$ and its corresponding value function $V_t$ is important for bounding error propagation. The following lemma shows that the $\ell_\infty$ norm of $Q_t$ grows at most linearly.

\begin{lemma}
	\label{lem:bounded_l_infty}
	The following inequality holds for all $t\geq 1$:
	\[
	\norm[\infty]{Q_{t}} \leq \norm[\infty]{Q_{t-1}} + \lambda 
	\]
\end{lemma}
\begin{proof}
	Since $\linftynorm{r}\leq 1$, for all $(s,a)\in \cS\times \cA$ we have
	\begin{align*}
		Q_{t}(s,a)
		=&\, (1-\lambda) Q_{t-1}(s,a) + \lambda\bigl(r(s,a) + \max_{a'\in \cA}Q_{t-1}(s', a')\bigr)\\
		\leq&\, (1-\lambda)\linftynorm{Q_{t-1}} + \lambda\bigl(1 + \linftynorm{Q_{t-1}}\bigr)\\
		=&\, \linftynorm{Q_{t-1}} + \lambda.
	\end{align*}
	Since this holds for all $(s,a)$, it follows that $\linftynorm{Q_t}\leq \linftynorm{Q_{t-1}} + \lambda$, which establishes the claim.
\end{proof}

\begin{lemma}
	\label{lem:bound_q_error_with_lazy_q}
	Suppose $Q_1^{\mathrm{corr}}, Q_2^{\mathrm{corr}}, Q_1, Q_2 \in \R^{|\cS||\cA|}$ where
	\[
	Q_1^{\mathrm{corr}} = \bigl(\Ii - \tfrac{1}{2}\Kk\bPi_1\bigr)Q_1,
	\quad
	Q_2^{\mathrm{corr}} = \bigl(\Ii - \tfrac{1}{2}\Kk\bPi_2\bigr)Q_2,
	\]
	and $\bPi_1, \bPi_2$ are the greedy policy matrices for $Q_1$ and $Q_2$, respectively. Then
	\[
	\spannorm{Q_1^{\mathrm{corr}} - Q_2^{\mathrm{corr}}} \leq 2\,\spannorm{Q_1 - Q_2}.
	\]
\end{lemma}
\begin{proof}
	\begin{align*}
		\spannorm{Q_1^{\mathrm{corr}} - Q_2^{\mathrm{corr}}}
		=&\, \spannorm{(\Ii - \tfrac{1}{2}\Kk\bPi_1)Q_1 - (\Ii - \tfrac{1}{2}\Kk\bPi_2)Q_2} \\
		=&\, \spannorm{Q_1 - Q_2 - \tfrac{1}{2}\Kk(\bPi_1 Q_1 - \bPi_2 Q_2)} \\
		\leq&\, \spannorm{Q_1 - Q_2}
		+ \tfrac{1}{2}\spannorm{\Kk(\bPi_1 Q_1 - \bPi_2 Q_2)} \\
		\overset{(i)}{\leq}&\, \spannorm{Q_1 - Q_2}
		+ \tfrac{1}{2}\spannorm{\bPi_1 Q_1 - \bPi_2 Q_2} \\
		\overset{(ii)}{\leq}&\, \spannorm{Q_1 - Q_2}
		+ \tfrac{1}{2}\spannorm{Q_1 - Q_2}
		\leq 2\,\spannorm{Q_1 - Q_2}.
	\end{align*}
	where (i) holds since $\Kk$ is a row-stochastic matrix and hence nonexpansive in the span seminorm, and (ii) follows from Lemma~\ref{lem:convex_combination_of_q_value}.
\end{proof}

\subsection{Proof of Theorem \ref{thm:sync_q_learning_optimal_rate}}
\label{sec:app:theorem_alg1}
For Algorithms~\ref{alg:sync_q_explicit_lazy_sampling} and~\ref{alg:sync_q_implicit_lazy_sampling}, the corresponding lazy kernel and empirical kernels at the $t$-th iteration are defined as
\begin{equation}
	\label{equ:app:lazy_kernel_and_lazy_sampling}
	\begin{aligned}
		\overline{\Pp} \;:=\; \frac{1}{2}\Kk + \frac{1}{2}\Pp,
\quad\text{and}\quad\overline{\Pp}_t \;\sim&\; \overline{\Pp}\quad\text{Explicit Lazy Sampling (Algorithm~\ref{alg:sync_q_explicit_lazy_sampling})}\\
\Pp_t\;\sim& \; \overline{\Pp}\quad\text{Implicit Lazy Sampling (Algorithm~\ref{alg:sync_q_implicit_lazy_sampling})}.
	\end{aligned}
\end{equation}

We give the general proof framework for Theorem~\ref{thm:sync_q_learning_optimal_rate} of both Algorithm~\ref{alg:sync_q_explicit_lazy_sampling} and~\ref{alg:sync_q_implicit_lazy_sampling}. Let 
\[
\hatT_{\overline{P}, t}(\cdot) = \hatT_{\overline{P}, t}^{\mathrm{exp}}(\cdot)\quad \text{or}\quad \hatT_{\overline{P}, t}^{\mathrm{imp}}(\cdot)
\]
be the empirical Bellman operator.

\textbf{Step 1: Decomposition.} Denote the estimate error at timestamp $t$ as
\begin{align*}
    \Delta_t =& Q_t - \overline{Q}^*\\
    \overset{(i)}{=}& (1-\lambda)\Delta_{t-1} + \lambda \hatT_{\overline{P}, t}(Q_{t-1}) - \lambda\bracket{ \cT_{\overline{P}}(\overline{Q}^*) - g^*\mathbf 1}\\
    =& (1-\lambda)\Delta_{t-1} + \lambda \bracket{\hatT_{\overline{P}, t}(Q_{t-1}) - \cT_{\overline{P}}(Q_{t-1}) + \cT_{\overline{P}}(Q_{t-1}) - \cT_{\overline{P}}(\overline{Q}^*)} + \lambda g^*\mathbf 1.
\end{align*}
Where $(i)$ follows from the Bellman equation of the lazy MDP $\overline{\cM}$,
\[
g^* \mathbf{1} + \overline{Q}^* = \cT_{\overline{P}}(\overline{Q}^*).
\]

We start by decomposing the estimate error $\Delta_t$, let $\overline{\pi}^*$ be the optimal policy of $\overline{\cM}$, and $\overline{\bPi}^*:=\bPi^{\overline{\pi}^*}$. From the Q-learning update rule, we have:
\begin{align}
    \Delta_t =& (1-\lambda)\Delta_{t-1} + \lambda \bracket{\hatT_{\overline{P}, t}(Q_{t-1}) - \cT_{\overline{P}}(Q_{t-1}) + \overline{\Pp}\,\overline{\bPi}_{t-1} Q_{t-1} - \overline{\Pp}\,\overline{\bPi}^* \overline{Q}^*} + \lambda g^*\mathbf 1\notag\\
    =& (1-\lambda)\Delta_{t-1} + \lambda\overline{\Pp}(\bPi_{t-1} Q_{t-1} - \overline{\bPi}^* \overline{Q}^*) +  \lambda \bracket{\hatT_{\overline{P}, t}(Q_{t-1}) - \cT_{\overline{P}}(Q_{t-1})} + \lambda g^*\mathbf 1.\label{equ:iteration_prototype}
\end{align}
Then, using Lemma \ref{lem:convex_combination_of_q_value}, there exists a policy $\widetilde{\pi}_{t-1}$, such that:
\begin{equation}
	\label{equ:sync_main_delta_q_combination}
\bPi_{t-1} Q_{t-1} - \overline \bPi^* \overline Q^* = \bPi^{\widetilde{\pi}_{t-1}} (Q_{t-1} - \overline Q^*) = \bPi^{\widetilde{\pi}_{t-1}} \Delta_{t-1}.
\end{equation}
There, substitute \eqref{equ:sync_main_delta_q_combination} into \eqref{equ:iteration_prototype} to obtain:
\begin{equation*}
	\Delta_t = (1-\lambda)\Delta_{t-1} + \lambda \overline \Pp\,\bPi^{\widetilde{\pi}_{t-1}}\Delta_{t-1} + \lambda (\hatT_{\overline{P}, t}(Q_{t-1}) - \cT_{\overline{P}}(Q_{t-1})) + \lambda g^* \mathbf{1}.
\end{equation*}
Applying this relation recursively, we obtain
\begin{equation*}
	\Delta_t = (1-\lambda)^t\Delta_0 + \lambda\sum_{i=1}^t (1-\lambda)^{t-i}\overline \Pp\,\bPi^{\widetilde{\pi}_{i-1}}\Delta_{i-1} + \lambda \sum_{i=1}^t (1-\lambda)^{t-i}\bracket{\hatT_{\overline{P}, i}(Q_{i-1}) - \cT_{\overline{P}}(Q_{i-1})} + \lambda \sum_{i=1}^t (1-\lambda)^{t-i}g^*\mathbf{1}
\end{equation*}
Since the iteration corresponds to the lazy MDP $\overline{\cM}$ with transition kernel $\overline{\Pp}$, by Corollary~\ref{cor:lem:new_seminorm_q_function} there exists a seminorm $\spanstar{\cdot}$ such that the operator $\overline{\Pp}\,\bPi^{\widetilde{\pi}_{i-1}}$ is a contraction under $\spanstar{\cdot}$ with contraction factor $\beta = \bigl(1-\tfrac{1}{K2^K}\bigr)^{1/(K+1)}$. Note that $\widetilde{\pi}_{i-1}$ is stationary at each iteration, so the corollary applies pointwise in time. Applying $\spanstar{\cdot}$ to both sides and noting that $g^*\mathbf{1}$ is constant (and hence lies in the null space of $\spanstar{\cdot}$), we have $\spanstar{g^*\mathbf{1}}=0$. Therefore,
\begin{align}
    \spanstar{\Delta_t} \leq& (1-\lambda)^t\spanstar{\Delta_0} + \spanstar{\lambda\sum_{i=1}^t(1-\lambda)^{t-i} \overline \Pp\,\bPi^{\widetilde{\pi}_{i-1}}\Delta_{i-1}} \notag\\
	&+ \spanstar{\lambda\sum_{i=1}^t(1-\lambda)^{t-i}\bracket{\hatT_{\overline{P}, i}(Q_{i-1}) - \cT_{\overline{P}}(Q_{i-1})}} \notag\\
    \leq& \underbrace{(1-\lambda)^t \spanstar{\Delta_0}}_{\text{initial bias}:= \phi_t} + \underbrace{\beta \lambda\sum_{i=1}^t (1-\lambda)^{t-i}\spanstar{\Delta_{i-1}}}_{\text{cumulative bias}} \notag\\
	&+ \underbrace{\spanstar{\lambda\sum_{i=1}^t(1-\lambda)^{t-i}\bracket{\hatT_{\overline{P}, i}(Q_{i-1}) - \cT_{\overline{P}}(Q_{i-1})}}}_{\text{martingale error}:= \psi_t}.\label{equ:sync_seminorm_iteration}
\end{align}

\textbf{Step 2: Bounding the martingale error $\psi_t$.}
Consider the martingale term:
\begin{enumerate}
\item When $\hatT_{\overline{P}, i}(\cdot)=\hatT^{\mathrm{exp}}_{\overline{P}, i}(\cdot)$, for any $Q\in \R^{|\cS||\cA|}$,
\begin{align*}
\hatT_{\overline{P}, i}(Q) - \cT_{\overline{P}}(Q)
=&\, r + \overline{\Pp}_{i} \bPi Q - \bigl(r + \overline{\Pp}\bPi Q\bigr)\\
=&\, (\overline{\Pp}_i - \overline{\Pp}) \bPi Q .
\end{align*}
Define
\[
W^{\mathrm{exp}}_t := \lambda\sum_{i=1}^t (1-\lambda)^{t-i}\zeta_i^{\mathrm{exp}},
\qquad
\zeta^{\mathrm{exp}}_i := (\overline \Pp_i - \overline \Pp) V_{i-1}.
\]
By Lemma~\ref{lem:bounded_l_infty},
\[
\linftynorm{Q_t}\leq \lambda t
\quad\Longrightarrow\quad
\linftynorm{V_t}\leq \lambda t.
\]
Under the filtration
\[
\cF^{\mathrm{exp}}_{i} = \sigma(\overline{\Pp}_1,\ldots,\overline{\Pp}_i, Q_1,\ldots,Q_i),
\]
$V_{i-1}$ is $\cF^{\mathrm{exp}}_{i-1}$-measurable and
\[
\E\!\left[\zeta^{\mathrm{exp}}_i\mid \cF^{\mathrm{exp}}_{i-1}\right]=0,
\qquad
\linftynorm{\zeta^{\mathrm{exp}}_i}
\le (\|\overline{\Pp}_i\|_1+\|\overline{\Pp}\|_1)\|V_{i-1}\|_\infty
\le 2\lambda(i-1)\le 2\lambda i .
\]
Thus, letting
\[
c_i := 2\lambda^2 (1-\lambda)^{t-i} i,
\]
we obtain
\begin{align*}
\sum_{i=1}^t c_i^2
\le&\, 4\lambda^4 t^2 \sum_{i=1}^t (1-\lambda)^{2(t-i)}
\le \frac{4\lambda^4 t^2}{1-(1-\lambda)^2}
\le 4\lambda^3 t^2 .
\end{align*}
Applying Azuma’s inequality coordinate-wise and taking a union bound, for any $\xi>0$,
\[
\mathbb P\!\left(\linftynorm{W^{\mathrm{exp}}_t}
\ge \sqrt{8\lambda^3 t^2 \log\frac{2|\cS||\cA|T}{\xi}}\right)
\le \xi,\qquad \forall\,1\le t\le T .
\]
By Lemma~\ref{lem:new_seminorm_equivalent_to_span}, with probability at least $1-\xi$,
\begin{equation}
\label{equ:app:spanstar_martingale_bound}
\spanstar{W^{\mathrm{exp}}_t}
\le 2\spannorm{W^{\mathrm{exp}}_t}
\le 4\linftynorm{W^{\mathrm{exp}}_t}
\le 8\sqrt{2\lambda^3 t^2 \log\frac{2|\cS||\cA|T}{\xi}},
\end{equation}
holds for all $t\in[1,T]$.

\item When $\hatT_{\overline{P}, i}(\cdot)=\hatT^{\mathrm{imp}}_{\overline{P}, i}(\cdot)$,
\begin{align*}
\hatT_{\overline{P}, i}(Q) - \cT_{\overline{P}}(Q)
=&\, r + \tfrac12 \Kk\bPi Q + \tfrac12 \Pp_i\bPi Q
- \bigl(r + \tfrac12 \overline{\Pp}\bPi Q\bigr)\\
=&\, \tfrac12(\Pp_i-\Pp)\bPi Q .
\end{align*}
Similarly to case~(1), define
\[
W^{\mathrm{imp}}_t := \lambda\sum_{i=1}^t (1-\lambda)^{t-i}\zeta^{\mathrm{imp}}_i,
\qquad
\zeta^{\mathrm{imp}}_i := \tfrac12(\Pp_i-\Pp)V_{i-1}.
\]
Then, with probability at least $1-\xi$, for all $t\in[1,T]$,
\[
\spanstar{W_t^{\mathrm{imp}}}
\le 4\sqrt{2\lambda^3 t^2\log\frac{2|\cS||\cA|T}{\xi}} .
\]
\end{enumerate}
Therefore,
\[
\psi_t
\le \max\{\spanstar{W_t^{\mathrm{exp}}},\,\spanstar{W_t^{\mathrm{imp}}}\}
\le 8\sqrt{2\lambda^3 t^2\log\frac{2|\cS||\cA|T}{\xi}},
\]
holds for all $t\in[1,T]$.

\textbf{Step 3: putting martingale error into iteration.} 
By~\eqref{equ:app:spanstar_martingale_bound}, the recursion~\eqref{equ:sync_seminorm_iteration} implies that, with probability at least $1-\xi$, for all $1\le t\le T$,
\begin{align*}
\spanstar{\Delta_t}
\leq&\, (1-\lambda)^t\spanstar{\Delta_0} + \psi_t
+ \beta \lambda\sum_{i=1}^t (1-\lambda)^{t-i}\spanstar{\Delta_{i-1}}\\
\leq&\, (1-\lambda)^t\spanstar{\Delta_0}
+ 8\sqrt{2\lambda^3 t^2 \log\frac{2|\cS||\cA|T}{\xi}}
+ \beta \lambda\sum_{i=1}^t (1-\lambda)^{t-i}\spanstar{\Delta_{i-1}}.
\end{align*}

\textbf{Step 4: solving the iteration bound.} Define an auxiliary sequence $u_t$ as an upper bound for $\spanstar{\Delta_t}$ that satisfies:
\begin{equation*}
\begin{aligned}
u_0 &:= \spanstar{\Delta_0},\\
u_t &:= (1-\lambda)^t u_0
      + 8\sqrt{2\lambda^3 t^2 \log\!\Bigl(\frac{2|\cS||\cA|T}{\xi}\Bigr)}
      + \beta\lambda \sum_{i=1}^t (1-\lambda)^{t-i} u_{i-1},
\qquad t\ge 1.
\end{aligned}
\end{equation*}
By construction, $u_0=\spanstar{\Delta_0}$. Suppose inductively that $\spanstar{\Delta_k}\le u_k$ holds for all $0\le k\le t-1$. Then, at time $t$,
\begin{align*}
	\spanstar{\Delta_t} \leq& (1-\lambda)^t \spanstar{\Delta_0} + 8\sqrt{2\lambda^3 t^2 \log\frac{2|\cS||\cA|T}{\xi}} + \beta\lambda\sum_{i=1}^t (1-\lambda)^{t-i}\spanstar{\Delta_{i-1}}\notag\\
	\overset{(i)}{\leq}& (1-\lambda)^t u_0 + 8\sqrt{2\lambda^3 t^2 \log\frac{2|\cS||\cA|T}{\xi}} + \beta\lambda\sum_{i=1}^t (1-\lambda)^{t-i} u_{i-1}\notag\\
	=& u_t\label{equ:app:spanstar_u_t}
\end{align*}
where (i) holds by the induction hypothesis holds for all $0\leq k\leq t-1$.

Further, define

\begin{equation}
\label{equ:app:u_t_u_0_v_t}
\begin{aligned}
v_0 &:= 0,\\
v_t &:= u_t - \bigl(1-\lambda(1-\beta)\bigr)^t u_0,\qquad t\ge 1.
\end{aligned}
\end{equation}

Recall that for all $t\ge 1$,
\begin{equation}\label{equ:app:ut-rec}
u_t
= (1-\lambda)^t u_0
+ \beta\lambda\sum_{i=1}^t (1-\lambda)^{t-i}u_{i-1}
+ 8\sqrt{2\lambda^3 t^2 \log\frac{2|\cS||\cA|T}{\xi}} .
\end{equation}
Moreover, we use the identity
\[
(1-\lambda)^t u_0
+ \beta\lambda\sum_{i=1}^t (1-\lambda)^{t-i}\bigl(1-\lambda(1-\beta)\bigr)^{i-1}u_0
= \bigl(1-\lambda(1-\beta)\bigr)^t u_0.
\]
Substituting $u_{i-1}=\bigl(1-\lambda(1-\beta)\bigr)^{i-1}u_0+v_{i-1}$ into \eqref{equ:app:ut-rec} yields
\begin{align*}
u_t
=&\ (1-\lambda)^t u_0
+ \beta\lambda\sum_{i=1}^t (1-\lambda)^{t-i}
\Bigl[\bigl(1-\lambda(1-\beta)\bigr)^{i-1}u_0 + v_{i-1}\Bigr]
+ 8\sqrt{2\lambda^3 t^2 \log\frac{2|\cS||\cA|T}{\xi}}\\
=&\ \bigl(1-\lambda(1-\beta)\bigr)^t u_0
+ \beta\lambda\sum_{i=1}^t (1-\lambda)^{t-i} v_{i-1}
+ 8\sqrt{2\lambda^3 t^2 \log\frac{2|\cS||\cA|T}{\xi}}.
\end{align*}
Therefore, subtracting $\bigl(1-\lambda(1-\beta)\bigr)^t u_0$ from both sides gives the recursion
\begin{equation}
	\label{equ:app:u_t_u_0_martingale}
	v_t
= 8\sqrt{2\lambda^3 t^2 \log\frac{2|\cS||\cA|T}{\xi}}
+ \beta\lambda\sum_{i=1}^t (1-\lambda)^{t-i} v_{i-1}.
\end{equation}
Define the maximum sequence:
\begin{equation}
	\label{equ:app:bar_v_t_bound}
\bar{v}_t := \max_{0\leq i\leq t} v_i,
\end{equation}
we claim:
\begin{equation}
	\label{equ:sync_top_up_sequence_for_mg}
\bar{v}_t \leq \frac{8}{1-\beta}\sqrt{2\lambda^3 t^2 \log\frac{2|\cS||\cA|T}{\xi}}.
\end{equation}
When $t=0$, we have $\bar v_0 = 0\leq \frac{1}{1-\beta}\cdot 0$, and the induction holds to $0\leq k\leq t-1 $, when $k=t$,
then:
\begin{align*}
    v_t =& 8\sqrt{2\lambda^3 t^2\log\frac{2|\cS||\cA|T}{\xi}} + \beta\lambda\sum_{i=1}^{t}(1-\lambda)^{t-i} v_{i-1}\\
    \overset{(i)}{\leq}& 8\sqrt{2\lambda^3 t^2\log\frac{2|\cS||\cA|T}{\xi}} + \beta\lambda\sum_{i=1}^{t}(1-\lambda)^{t-i} \bar v_{t-1}\\
    \overset{(ii)}{\leq}& 8\sqrt{2\lambda^3 t^2\log\frac{2|\cS||\cA|T}{\xi}} + \beta \bar v_{t-1}\\
    \overset{(iii)}{\leq}& 8\sqrt{2\lambda^3 t^2\log\frac{2|\cS||\cA|T}{\xi}} + \frac{8\beta}{1-\beta}\sqrt{2\lambda^3 (t-1)^2\log\frac{2|\cS||\cA|T}{\xi}}\\
    \leq & \frac{8}{1-\beta}\sqrt{2\lambda^3 t^2\log\frac{2|\cS||\cA|T}{\xi}}
\end{align*}
(i) follows $\bar v_t$ is monotone increasing by definition. (ii) is derived due to $\lambda\sum_{i=1}^t (1-\lambda)^{t-i}< 1$. And (iii) is because of the induction hypothesis holds for all $0\leq k\leq t-1$. Since
\begin{align*}
	\bar v_t =& \max\set{v_t, \bar v_{t-1}}\\
	=& \max\set{\frac{8}{1-\beta}\sqrt{2\lambda^3 t^2\log\frac{2|\cS||\cA|T}{\xi}}, \bar v_{t-1}}\\
	\leq& \frac{8}{1-\beta}\sqrt{2\lambda^3 t^2\log\frac{2|\cS||\cA|T}{\xi}}.
\end{align*}
We proved the claim of \eqref{equ:sync_top_up_sequence_for_mg}. Here $\beta<1$ plays an important role in the convergence. Without $\beta < 1$, the martingale error is not guaranteed to converge.

\textbf{Step 5: Put the results together.} Combine~\eqref{equ:app:u_t_u_0_v_t},~\eqref{equ:app:u_t_u_0_martingale},~\eqref{equ:app:bar_v_t_bound},~\eqref{equ:sync_top_up_sequence_for_mg}, we have with probability $1-\xi$, for all $1\leq t\leq T$:
\begin{align}
	\spanstar{\Delta_t} \leq& u_t \notag\\
	\leq& (1-\lambda(1-\beta))^t u_0 + v_t\notag\\
	\leq& (1-\lambda(1-\beta))^t \spanstar{\Delta_0} + \frac{8}{1-\beta}\sqrt{2\lambda^3 t^2 \log\frac{2|\cS||\cA|T}{\xi}}\label{equ:sync_error_bound_before_parameter_choice}.
\end{align}
Finally, take parameter
\[
\lambda = \frac{K(K+1)2^{K}\ln T}{T}\quad \text{and} \quad \beta = \bracket{1-\frac{1}{K2^K}}^{1/(K+1)}
\]
and $\spanstar{\cdot} \leq 2\spannorm{\cdot}$ into the inequality \eqref{equ:sync_error_bound_before_parameter_choice}, we have with probability $1-\xi$, for all $1\leq t\leq T$:
\[
\spannorm{\Delta_t} \leq \spanstar{\Delta_t} \leq (1-\ln T/ T)^t + 8\sqrt{\frac{2K^5(K+1)^52^{5K} \ln^3T}{T}\log\frac{2|\cS||\cA|T}{\xi}}.
\]
Apply Lemma \ref{lem:bound_q_error_with_lazy_q}, with $t=T$ and $Q_T^{\mathrm{corr}} = (\Ii - \frac{1}{2}\Kk\bPi_T)Q_T$ and $Q^* = (\Ii - \frac{1}{2}\Kk \overline \bPi^*)\overline{Q}^*$, we got:
\begin{align*}
	\spannorm{Q_T^{\mathrm{corr}} - Q^*} \leq& 2\spannorm{Q_T - \overline Q^*}\\
	=& 2\spannorm{\Delta_T}\\
	\leq& 2\bracket{1-\ln T/ T}^T\spannorm{\Delta_0} + 16\sqrt{\frac{2K^5(K+1)^52^{5K}\ln^3T}{T}\log\frac{2|\cS||\cA|T}{\xi}}\\
	\leq& \frac{4K+2}{T} + 16\sqrt{\frac{2K^5(K+1)^5 2^{5K}\log^3T}{T}\log\frac{2|\cS||\cA|T}{\xi}}
\end{align*}
Where the last inequality holds since $\spannorm{\Delta_0}=\spannorm{Q^*} \leq 2K+1$ by Lemma~\ref{lem:app:q_star_span_bound}.

Consequently, if
\[
T\geq \max\set{\frac{40K}{\varepsilon},\;
\frac{640\,K^5 (K+1)^5 2^{5K}\,\log^3 T}{\varepsilon^2}\,
\log\!\frac{2|\cS||\cA|T}{\xi}},
\]
then, with probability at least $1-\xi$,
\[
\spannorm{Q_T^{\mathrm{corr}} - Q^*} \leq \varepsilon,
\]
By Lemma~\ref{lem:approximate_on_union_recurrent_class}, this further implies
\[
g^* - g^{\pi_T} \leq \varepsilon,
\]
with probability $1-\xi$.

In other words, the synchronous Q-learning algorithms~\ref{alg:sync_q_explicit_lazy_sampling} and~\ref{alg:sync_q_implicit_lazy_sampling} achieve $\varepsilon$-optimality with sample complexity
\[
N =  O\bracket{\frac{2^{6K}|\cS||\cA|\log^3T}{\varepsilon^2}\log\frac{|\cS||\cA|T}{\xi}}.
\]

	\section{Asynchronous Case, Proof of Theorem~\ref{thm:async_q_learning_optimal_rate_explicit} and~\ref{thm:async_q_learning_optimal_rate_implicit}}
\label{sec:app:asynchronous_case}

\subsection{Overview and Proof Sketch of Theorem~\ref{thm:async_q_learning_optimal_rate_explicit} and~\ref{thm:async_q_learning_optimal_rate_implicit}}
Here we provide a brief proof skecth of Theorem~\ref{thm:async_q_learning_optimal_rate_explicit}
and Theorem~\ref{thm:async_q_learning_optimal_rate_implicit} under asynchronous setting.
\begin{itemize}
	\item In Section~\ref{app:async:reduction_to_the_recurrent_class}, we bound the average-reward suboptimality gap $g^* - g^{\pi}$ by the Q-function approximation error restricted to the recurrent class $\cC$ induced by the behavior policy $\behpi$. Consequently, it suffices to analyze $Q_t$ only after the trajectory enters $\cC$.
	\item In Section~\ref{app:async:mixing_properties}, we show that Algorithm~\ref{alg:async_q_explicit_lazy_sampling} (analyzed in Theorem~\ref{thm:async_q_learning_optimal_rate_explicit}) is geometrically ergodic with stationary distribution $\rho$, since it samples from a lazy kernel. In contrast, Algorithm~\ref{alg:async_q_implicit_lazy_sampling} (analyzed in Theorem~\ref{thm:async_q_learning_optimal_rate_implicit}) achieves geometric ergodicity under an additional aperiodicity assumption, and hence also converges geometrically to $\rho$. As a consequence, the expected squared error $\E\sqbk{\spannorm{Q_t|_{\cC} - \overline{Q}^*|_{\cC}}^2}$ can be decomposed into a term after the trajectory enters $\cC$ and a term before it enters $\cC$, where the latter can be controlled via the geometric convergence to $\rho$.
	\item In Section~\ref{app:async:unified_asynchronous_stochastic_approximation_formulation}, we analysis the expected squred error $\E\sqbk{\spanstar{Q_t|_{\cC} - \overline{Q}^*|_{\cC}}^2}$ following the similar high-level framework to that of~\citet{chen2025nonasymptotic}, with necessary adjustments to accommodate the instance-dependent seminorm $\spanstar{\cdot}$ and the corresponding instance-dependent generalized Moreau envelope $M_{q,\theta}$.
	\item In Section~\ref{app:async:aggregate_the_proof}, we aggregate the results from previous sections to complete the proofs of Theorem~\ref{thm:async_q_learning_optimal_rate_explicit} and Theorem~\ref{thm:async_q_learning_optimal_rate_implicit}.
\end{itemize}

\subsection{Reduction to the Recurrent Class \texorpdfstring{$\mathcal{C}$}{C}}
\label{app:async:reduction_to_the_recurrent_class}

\begin{lemma}[Recurrent-class coverage]
	\label{lem:recurrent_class_inclusion}
	Under Assumption~\ref{ass:reachability}, for every
	$\pi\in \Pi^{\mathrm{SD}}$, we have $\cC_\pi \subseteq \cC$.
\end{lemma}
\begin{proof}
	Fix an arbitrary stationary deterministic policy $\pi\in\Pi^{\mathrm{SD}}$.
	By Assumption~\ref{ass:reachability}, the reference state
	$s^\dagger$ belongs to the recurrent class $\cC_\pi$.
	For any $s\in \cC_\pi$, define the hitting time
	\[
	T_s := \inf\{t>0 \mid s_t = s^\dagger,\ s_0 = s\}.
	\]
	Then
	\[
	\mathbb{P}^\pi\sqbk{T_s < \infty \mid s_0 = s} = 1,
	\]
	which implies that every state in $\cC_\pi$ reaches $s^\dagger$
	almost surely under $\pi$.

	Moreover, since $s^\dagger$ and $s$ both belong to the same recurrent
	class $\cC_\pi$, there exists a finite path
	\[
	s^\dagger = s_0 \xrightarrow{a_0} s_1 \xrightarrow{a_1} \cdots
	\xrightarrow{a_{m-1}} s_m = s,
	\]
	such that $p(s_{i+1}\mid s_i,a_i) > 0$ and $a_i = \pi(s_i)$ for all
	$i=0,\ldots,m-1$.

	Since the behavior policy $\behpi$ is fully supported,
	i.e., $\behpi(a\mid s)>0$ for all $(s,a)$, we have
	\[
	\mathbb{P}^{\behpi}\sqbkcond{s_m = s}{s_0 = s^\dagger}
	\ge \prod_{i=0}^{m-1} \behpi(a_i\mid s_i)\, p(s_{i+1}\mid s_i,a_i)
	> 0.
	\]
	Therefore, $s^\dagger$ and $s$ can reach each other with positive
	probability under the behavior policy, implying that they belong to
	the same communicating class induced by $\behpi$.
	Hence $s\in \cC$, and since $s$ was arbitrary in $\cC_\pi$, we conclude
	that $\cC_\pi \subseteq \cC$.

	Since the argument holds for all $\pi\in\Pi^{\mathrm{SD}}$, the claim
	follows.
\end{proof}

\begin{lemma}
	For any stationary deterministic policies $\pi_1, \pi_2\in \Pi^{\mathrm{SD}}$, if
	\[
	\pi_1(\cdot \mid s) = \pi_2(\cdot \mid s), \quad \forall s\in \cC,
	\]
	then $g^{\pi_1} = g^{\pi_2}$.
\end{lemma}
\begin{proof}
	The average reward $g^{\pi}$ depends only on the recurrent class
	$\cC_{\pi}$ of the Markov chain induced by $\pi$.
	Let $\rho_\pi$ denote the stationary distribution of this chain.
	Then
	\[
	g^{\pi} = \sum_{s\in \cS} \rho_\pi(s)\, r_\pi(s),
	\]
	where $\rho_\pi(s)=0$ for all $s\notin \cC_\pi$.

	By Lemma~\ref{lem:recurrent_class_inclusion}, we have
	$\cC_{\pi_1}\subseteq \cC$ and $\cC_{\pi_2}\subseteq \cC$.
	Therefore,
	\begin{align*}
		g^{\pi_1}
		&= \sum_{s\in \cC_{\pi_1}} \rho_{\pi_1}(s)\, r(s,\pi_1(s)) \\
		&= \sum_{s\in \cC} \rho_{\pi_1}(s)\, r(s,\pi_1(s)) \\
		&= \sum_{s\in \cC} \rho_{\pi_2}(s)\, r(s,\pi_2(s)) \\
		&= \sum_{s\in \cC_{\pi_2}} \rho_{\pi_2}(s)\, r(s,\pi_2(s)) \\
		&= g^{\pi_2},
	\end{align*}
	where the third equality holds since $\pi_1$ and $\pi_2$ coincide on $\cC$.
	This completes the proof.
\end{proof}

Intuitively, actions taken on transient states do not affect the long-run average reward.

\begin{lemma}[Sufficiency of accuracy on $\cC$]
\label{lem:approximate_on_union_recurrent_class}
Let $Q\in \R^{|\cS||\cA|}$ and let $\cC\subseteq\cS$.
Denote by $Q|_{\cC}$ the restriction of $Q$ to $\cC\times\cA$, and let
$\spannorm{\cdot}$ be the span seminorm over $(s,a)\in\cC\times\cA$.
Then any greedy policy $\pi$ satisfying
$\pi(s)\in\arg\max_{a\in\cA}Q(s,a)$ for all $s\in\cC$
(defined arbitrarily on $\cS\setminus\cC$) guarantees
\[
0 \le g^* - g^\pi \le \spannorm{Q|_{\cC}-Q^*|_{\cC}},
\qquad
\spannorm{V^\pi|_{\cC} - V^*|_{\cC}}
\le 2K \cdot \spannorm{Q|_{\cC} - Q^*|_{\cC}}.
\]
\end{lemma}
\begin{proof}
Denote $\Delta|_{\cC}(s,a):=Q|_{\cC}(s,a)-Q^*|_{\cC}(s,a)$. For any $s\in\cC$, we have
\begin{align}
	&Q^*|_{\cC}(s,\pi^*(s)) - Q^*|_{\cC}(s,\pi(s)) \notag\\
	=&\, Q|_{\cC}(s,\pi^*(s)) - Q|_{\cC}(s,\pi(s))
	+ Q|_{\cC}(s,\pi(s)) - Q^*|_{\cC}(s,\pi(s)) \notag\\
	&\quad - \bigl(Q|_{\cC}(s,\pi^*(s)) - Q^*|_{\cC}(s,\pi^*(s))\bigr) \notag\\
	\leq&\, 0 + \max_{(s,a)\in\cC\times\cA}\Delta|_{\cC}(s,a)
	- \min_{(s,a)\in\cC\times\cA}\Delta|_{\cC}(s,a) \notag\\
	=&\, \spannorm{\Delta|_{\cC}},
	\label{equ:optimal_q_residual_difference_bounded_by_span}
\end{align}
where the inequality follows from the fact that $\pi$ is greedy with respect to $Q$.

Recall the Bellman optimality equations:
\begin{align}
	g^* + Q^*(s,a)
	=&\, r(s,a) + \sum_{s'} p(s'|s,a)\max_{a'} Q^*(s',a'), \notag\\
	g^* + V^*(s)
	=&\, \max_{a}\Bigl\{r(s,a) + \sum_{s'} p(s'|s,a)V^*(s')\Bigr\} \notag\\
	=&\, r(s,\pi^*(s)) + \sum_{s'} p(s'|s,\pi^*(s))V^*(s').
	\label{equ:residual_difference_optiamlity}
\end{align}
Define the residual
\[
\delta_\pi(s):= g^* + V^*(s)
- \Bigl(r(s,\pi(s)) + \sum_{s'} p(s'|s,\pi(s))V^*(s')\Bigr).
\]
By optimality, $\delta_\pi(s)\ge \delta_{\pi^*}(s)=0$ for all $s\in\cS$. Consequently,
\[
g^* - r(s,\pi(s))
= \sum_{s'} p(s'|s,\pi(s))V^*(s') - V^*(s) + \delta_\pi(s).
\]

Since the stationary distribution $\rho_\pi$ is supported on $\cC$, we obtain
\begin{align*}
	g^* - g^\pi
	=&\, \sum_{s\in\cS}\rho_\pi(s)(g^* - r(s,\pi(s))) \\
	=&\, \sum_{s\in\cS}\rho_\pi(s)\delta_\pi(s)
	= \sum_{s\in\cC}\rho_\pi(s)\delta_\pi(s)
	\le \sup_{s\in\cC}\delta_\pi(s).
\end{align*}
Combining this with~\eqref{equ:optimal_q_residual_difference_bounded_by_span} yields
\[
g^* - g^\pi \le \spannorm{Q|_{\cC}-Q^*|_{\cC}}.
\]

Next, note that $Q^*$ and $V^\pi$ satisfy (see \citealp{kumar2025global})
\begin{align}
	g^* + Q^*(s,\pi(s))
	=&\, r(s,\pi(s)) + \sum_{s'} p(s'|s,\pi(s))Q^*(s',\pi^*(s')), \label{equ:opt_q_lemma_asyn}\\
	g^\pi + V^\pi(s)
	=&\, r(s,\pi(s)) + \sum_{s'} p(s'|s,\pi(s))V^\pi(s'),
	\quad s\in\cC.
	\label{equ:policy_q_lemma_async}
\end{align}
Define
\[
\eta(s):= (g^*-g^\pi) + (Q^*(s,\pi(s)) - V^*(s)),
\qquad
h(s):=V^\pi(s)-V^*(s).
\]
Then
\[
\eta(s)= h(s) - \sum_{s'} p(s'|s,\pi(s))h(s').
\]

Let $(s_t)_{t\ge0}$ be the Markov chain induced by $\pi$ with filtration
$\mathcal F_t=\sigma(s_0,\ldots,s_t)$. Then
\[
\E[h(s_{t+1})\mid\mathcal F_t]
= \sum_{s'} p(s'|s_t,\pi(s_t))h(s'),
\]
and hence $\eta(s_t)=h(s_t)-\E[h(s_{t+1})\mid\mathcal F_t]$.
For $s_0=s\in\cC$, let $T_s:=\inf\{t>0:s_t=s^\dagger\}$. By Assumption~\ref{ass:reachability},
$\E[T_s]\le K$, and telescoping yields
\begin{align*}
\E\Bigl[\sum_{t=0}^{T_s-1}\eta(s_t)\,\Big|\,s_0=s\Bigr]
&= \E[h(s_0)-h(s_{T_s})\mid s_0=s]
= h(s)-h(s^\dagger).
\end{align*}

Finally, using the bounds
\[
0\le g^*-g^\pi \le \spannorm{Q|_{\cC}-Q^*|_{\cC}},
\qquad
-\spannorm{Q|_{\cC}-Q^*|_{\cC}}
\le Q^*(s,\pi(s)) - V^*(s)\le 0,
\]
we obtain $\|\eta|_{\cC}\|_\infty\le \spannorm{Q|_{\cC}-Q^*|_{\cC}}$, and therefore
\[
|h(s)-h(s^\dagger)|
\le \E[T_s]\cdot \|\eta|_{\cC}\|_\infty
\le K\spannorm{Q|_{\cC}-Q^*|_{\cC}}.
\]
Taking the span over $s\in\cC$ completes the proof:
\[
\spannorm{V^\pi|_{\cC}-V^*|_{\cC}}
\le 2K\spannorm{Q|_{\cC}-Q^*|_{\cC}}.
\]
\end{proof}

Lemma~\ref{lem:approximate_on_union_recurrent_class} implies that all recurrent states that may arise under stationary deterministic policies (in particular, under an optimal policy) are contained in $\cC$. Hence, it suffices to obtain an accurate approximation of $Q^*$ on $\cC$: once $\spannorm{Q|_{\cC}-Q^*|_{\cC}}$ is small, taking a greedy policy on $\cC$ immediately translates the $Q$-function approximation guarantee into an average-reward performance guarantee.

\subsection{Mixing Properties for Theorem~\ref{thm:async_q_learning_optimal_rate_explicit} and~\ref{thm:async_q_learning_optimal_rate_implicit}}
\label{app:async:mixing_properties}

The proofs of Theorem~\ref{thm:async_q_learning_optimal_rate_explicit}
and Theorem~\ref{thm:async_q_learning_optimal_rate_implicit}
both rely on the same structural property of the Markov chain, namely \emph{uniform ergodicity} under the behavior policy. Once this property holds, the subsequent error decomposition and martingale arguments are identical for the two algorithms.

\begin{definition}[Uniform Ergodicity \citep{wang2023optimalsamplecomplexityreinforcement}]
	\label{def:uniform_ergodicity}
	A Markov chain with transition matrix $\mathbf{R}$ is called uniformly ergodic if there exists probability measure $\nu$ with $\nu \mathbf{R} = \nu$ ($\nu$ is the stationary distribution) and a constant $r$ such that
	\[
	\sup_{s\in\cS}\|\mathbf{R}^{t}(s, \cdot) - \nu(\cdot)\|_{\mathrm{TV}} \to 0.
	\]
\end{definition}

Then, by Theorem 16.0.2 of~\citet{meyn2012markov}, when the chain induced by $\mathbf{R}$ is uniformly ergodic, there exist constants $C<\infty$ and $r<1$ such that
\begin{equation}
	\label{equ:geometric_radius}
	r:=\inf\set{r', \exists C < \infty, \forall t\geq 1, \sup_{s\in\cS}\|\mathbf{R}^t(s, \cdot) - \nu(\cdot)\|_{\mathrm{TV}}\leq C (r')^{t}} < 1.
\end{equation}

Under Assumption~\ref{ass:reachability}, the lazy transition kernel $\overline{P}_{\behpi}$ used in Algorithm~\ref{alg:async_q_explicit_lazy_sampling} induces a Markov chain with a unique recurrent class $\cC$ and is aperiodic by lazy construction. Consequently, the chain with $\overline{P}_{\behpi}$ is uniformly ergodic.

For Algorithm~\ref{alg:async_q_implicit_lazy_sampling}, the Markov chain induced by the original kernel $P_{\behpi}$ shares the same recurrent class $\cC$. In this case, we additionally assume aperiodicity of $P_{\behpi}$ in Theorem~\ref{thm:async_q_learning_optimal_rate_implicit}, which ensures uniform ergodicity and allows the same geometric mixing bounds to be invoked. Apart from this aperiodicity requirement, the two analyses proceed in an entirely parallel manner. 

By Definition~\ref{def:uniform_ergodicity} and equation~\eqref{equ:geometric_radius},there exist constants $C<\infty$, and $r^{\mathrm{exp}}, r^{\mathrm{imp}}<1$ such that, for all
$t\ge 0$,
\begin{align*}
	\sup_{s\in \cS}
	\norm[\mathrm{TV}]{\overline{\Pp}_{\behpi}^t(s,\cdot)-\rho(\cdot)}
	\leq C (r^{\mathrm{exp}})^{-t},
	&\quad \text{for Algorithm~\ref{alg:async_q_explicit_lazy_sampling}},\\
	\sup_{s\in \cS}
	\norm[\mathrm{TV}]{\Pp_{\behpi}^t(s,\cdot)-\rho(\cdot)}
	\leq C (r^{\mathrm{imp}})^{-t},
	&\quad \text{for Algorithm~\ref{alg:async_q_implicit_lazy_sampling}}.
\end{align*}

\begin{lemma}
	\label{lem:async:lazy_radius}
	Under Assumption~\ref{ass:reachability}, the lazy kernel
	$\overline{P}_{\behpi}$ satisfies
	\[
	\frac{1}{1-r^{\mathrm{exp}}} \le (2K+2)^2\,2^{2K+2}.
	\]
\end{lemma}
The proof is provided in Section~\ref{sec:app:proof_of_lem:async:lazy_radius}.

Let $\eta\in\{\mathrm{exp},\mathrm{imp}\}$ denote the explicit or implicit lazy sampling case, respectively, and write $r^\eta$ for the corresponding geometric mixing rate.

Define
\begin{align}
	\tau_t^\eta :=& \min\set{k: C (r^{\eta})^{k} \leq \frac{\lambda^*}{t + h}}\label{app:async:tau_t}\\
	b_t :=& \lambda^*|\cS||\cA|\log\bracket{\frac{t/(|\cS||\cA|) + h}{h}}\label{equ:app:async:b_t}.
\end{align}
We next present two auxiliary lemmas that are useful for the analysis of asynchronous Q-learning. These lemmas control the almost-sure growth rate of the iterates $Q_t$ for both Algorithm~\ref{alg:async_q_explicit_lazy_sampling} and Algorithm~\ref{alg:async_q_implicit_lazy_sampling}.
\begin{lemma}
	\label{lem:async_q_function_growth_rate}
	For all $t\ge 1$,
	\[
	\spannorm{Q_{t+1}} \le \spannorm{Q_t} + \lambda_t(s_t,a_t).
	\]
\end{lemma}
The proof of Lemma~\ref{lem:async_q_function_growth_rate} is given in Section~\ref{sec:app:proof_of_lem:async_q_function_growth_rate}.
\begin{lemma}[Lemma B.1 in \citet{chen2025nonasymptotic}]
	\label{lem:async_q_span_bound}
	The following inequality holds for all $t\geq 1$:
	\[
	\spannorm{Q_t} \leq \lambda^*|\cS||\cA|\log\bracket{\frac{t/(|\cS||\cA|) + h}{h}}
	\]
\end{lemma}

Suppose the Markov chain starts from an arbitrary state $s\in\cS$, and let $\tau$ denote the first time at which the trajectory enters the recurrent class $\cC$, that is,
\begin{align*}
	s_t\notin \cC, \quad & \forall t < \tau,\\
	s_t \in \cC, \quad & \forall t\geq \tau.
\end{align*}
Once the chain enters $\cC$, it remains in $\cC$ thereafter. Consequently, for any $t$ we may decompose
\begin{align*}
	\E\sqbk{\spannorm{Q_t|_{\cC} - \overline{Q}|_{\cC}^*}^2} =& \E\sqbk{\spannorm{Q_t|_{\cC} - \overline{Q}|_{\cC}^*}^2 \mathbb I_{\{2\tau < t\}}} + \E\sqbk{\spannorm{Q_t|_{\cC} - \overline{Q}|_{\cC}^*}^2 \mathbb I_{\{2\tau \geq t\}}}\\
	\leq& \E\sqbk{\spannorm{Q_t|_{\cC} - \overline{Q}|_{\cC}^*}^2 \mathbb I_{\{2\tau < t\}}} + \E\sqbk{\spannorm{Q_t - \overline{Q}^*}^2 \mathbb I_{\{2\tau \geq t\}}}
\end{align*}
For Algorithm~\ref{alg:async_q_explicit_lazy_sampling}, we have
\begin{equation*}
	\E\sqbk{\spannorm{Q_t|_{\cC}-\overline{Q}^*|_{\cC}}^2}
	= \E\sqbk{\spannorm{Q_t|_{\cC}-\overline{Q}^*|_{\cC}}^2
	\mathbb{I}_{\{\tau< t/2\}}}
	+ \E\sqbkcond{\spannorm{Q_t-\overline{Q}^*}^2}{\tau>t/2}\,
	\overline{\mathbb P}\sqbk{\tau>t/2}.
\end{equation*}
Similarly, for Algorithm~\ref{alg:async_q_implicit_lazy_sampling},
\begin{equation*}
	\E\sqbk{\spannorm{Q_t|_{\cC}-\overline{Q}^*|_{\cC}}^2}
	= \E\sqbk{\spannorm{Q_t|_{\cC}-\overline{Q}^*|_{\cC}}^2
	\mathbb{I}_{\{\tau< t/2\}}}
	+ \E\sqbkcond{\spannorm{Q_t-\overline{Q}^*}^2}{\tau>t/2}\,
	\mathbb P\sqbk{\tau>t/2}.
\end{equation*}
By Lemma~\ref{lem:async_q_span_bound}, $\spannorm{Q_t}$ is almost surely bounded by
\[
\spannorm{Q_t} \le \lambda^*|\cS||\cA|
\log\!\bracket{\frac{t/(|\cS||\cA|)+h}{h}}.
\]
Moreover, in Algorithm~\ref{alg:async_q_explicit_lazy_sampling}, the Markov chain induced by $\behpi$ under $\overline{P}$ has an unique recurrent class and aperiodic, its hitting time to $\cC$ has an exponential tail. In particular,
\[
\overline{\mathbb{P}}\sqbk{\tau > \frac{t}{2}} \leq \overline{\mathbb{P}}\sqbk{s_{\frac{t}{2}}\notin \cC} =  \overline{\Pp}_{\behpi}^{t/2}(s, \cS\backslash \cC) \leq \rho(\cS\backslash \cC) + \|\overline{\Pp}_{\behpi}^{t/2}(s_0, \cdot) - \rho(\cdot)\|_{\mathrm{TV}}\leq C(r^{\mathrm{exp}})^{t/2}.
\]
Moreover, in Algorithm~\ref{alg:async_q_implicit_lazy_sampling}
\[
\mathbb P\sqbk{\tau > \frac{t}{2}} = \overline{\mathbb{P}}\sqbk{s_{\frac{t}{2}}\notin \cC} = \Pp_{\behpi}^{t/2}(s, \cS\backslash \cC) \leq \rho(\cS\backslash \cC) + \|\Pp_{\behpi}^{t/2}(s_0, \cdot) - \rho(\cdot)\|_{\mathrm{TV}} \leq C(r^{\mathrm{imp}})^{t/2}.
\]
Therefore, for $\eta \in \{\mathrm{exp}, \mathrm{imp}\}$,
\begin{equation}
	\label{equ:mc_step_to_recurrent_class_decomposition}
	\begin{aligned}
		\E\sqbk{\spannorm{Q_t|_{\cC}-\overline{Q}^*|_{\cC}}^2}
	\leq& \E\sqbk{\spannorm{Q_t|_{\cC}-\overline{Q}^*|_{\cC}}^2
	\mathbb{I}_{\{\tau<t/2\}}}\\
	\quad &+ C(r^\eta)^{t/2}
	\Bigl(\lambda^*|\cS||\cA|
	\log\!\bracket{\frac{t/(|\cS||\cA|)+h}{h}}
	+ \spannorm{\overline{Q}^*}\Bigr)^2.
	\end{aligned}
\end{equation}
When $\tau<t/2$, the trajectory has entered the recurrent class $\cC$ and remains in $\cC$ thereafter. Thus, the subsequent analysis can be restricted to the recurrent class $\cC$.

\begin{lemma}[Restart after entering the recurrent class]
\label{lem:restart_in_class}
Let $\tau := \inf\{t\geq 0: s_t\in\cC\}$ be the entrance time of the recurrent class $\cC$, i.e. once $s_\tau\in\cC$ then $s_{\tau+\ell}\in\cC$ for all $\ell\ge 0$ almost surely.
Fix any $t\ge 1$. On the event $\{\tau<t\}$, conditional on $\cF_\tau$, the post-$\tau$ process evolves as the same Markov chain started from $s_\tau\in\cC$.
Consequently, if the ``in-class'' mean-square bound holds for any horizon $m\ge 1$ whenever the chain starts in $\cC$, namely
\begin{equation}
\label{eq:in_class_bound_old}
\sup_{s\in\cC}\E_s\sqbk{\spannorm{Q_m|_{\cC}-\overline Q|_{\cC}^*}^2}\le f(m),
\end{equation}
then we have, for all $t\ge 1$,
\begin{equation}
\label{eq:in_class_bound_shifted}
\E\sqbk{\spannorm{Q_t|_{\cC}-\overline Q|_{\cC}^*}^2\,\Big|\,\cF_\tau}\le f(t - \tau)
\qquad\text{on }\{\tau<t\}.
\end{equation}
\end{lemma}
\begin{proof}
On event $\{\tau<t\}$ we have $s_\tau\in\cC$. By the strong Markov property at time $\tau$ and the closedness of $\cC$,
the trajectory $(s_{\tau+k})_{k\ge 0}$ is (conditionally on $\cF_\tau$) distributed as the same Markov chain started from $s_\tau\in\cC$ and remaining in $\cC$ thereafter.
Therefore, the evolution of the algorithm from time $\tau$ to time $t$ is equivalent (given $\cF_\tau$) to running the in-class algorithm for $m=t-\tau$ steps from an initial state in $\cC$.
Applying the in-class bound \eqref{eq:in_class_bound_old} with horizon $m=t-\tau$ yields \eqref{eq:in_class_bound_shifted}.
\end{proof}

Since after time $\tau$ the agent remains in $\cC$, we may restrict the subsequent analysis to the recurrent class $\cC$, on which the behavior policy $\behpi$ is irreducible. Equivalently, this corresponds to running the original algorithm starting from state $s_\tau\in\cC$ for $T-\tau$ steps, with initial value $Q_\tau$. For notational convenience, we henceforth restrict attention to state--action pairs in $\cC\times\cA$, and simply write $Q$ in place of $Q|_{\cC}$.

\subsection{Unified Asynchronous Stochastic Approximation Formulation}
\label{app:async:unified_asynchronous_stochastic_approximation_formulation}

\paragraph{Convention.} In this subsection, we focus exclusively on the behavior of the algorithm \emph{after} the trajectory has entered the recurrent class induced by the behavior policy. Accordingly, all states considered in the analysis belong to $\cC$, and all operators are evaluated on the state--action space $\cC\times\cA$. For notational simplicity, we continue to denote this effective state space by $\cS$, and we write $Q$ in place of its restriction $Q|_{\cC}$.

\paragraph{Time re-indexing.} Time is re-indexed so that $t=0$ corresponds to the first entrance into the recurrent class $\cC$. Under this convention, $Q_0 = \mathbf 0$ denotes the $Q$-function estimate at the moment the process enters $\cC$. By the strong Markov property, the post-entrance evolution is independent of the transient history, and the subsequent analysis concerns only the dynamics within the recurrent class.

A key technical device in \citet{chen2025nonasymptotic} is to compensate for the fact
that a single trajectory visits different state--action pairs at highly non-uniform frequencies. This can be viewed as an
importance-weighted normalization of the asynchronous update. Recall the visit counts
\[
N_t(s,a) \;:=\; \sum_{i=0}^{t-1}\mathbb{I}_{\{(s_i,a_i)=(s,a)\}},\qquad (s,a)\in\cS\times\cA.
\]
Define the smoothed empirical frequency matrix $D_t$ by
\[
D_t (s,a) \;=\; \frac{N_t(s,a) + h}{t + h},\qquad h \geq 1.
\]
Under the behavior policy $\behpi$, $D_t(s,a)$ is a consistent estimator of the limiting frequency $D(s,a):=\rho(s)\behpi(a\mid s)$, where $\rho$ is the stationary distribution of $P_{\behpi}$. Moreover, $D_t$ satisfies the recursion
\[
D_{t}(s,a) \;=\; \frac{(t-1+h)D_{t-1}(s,a)+\mathbb{I}_{\{(s_{t-1},a_{t-1})=(s,a)\}}}{t+h}.
\]
Next, define the augmented state process
\[
z_t \;:=\; (D_{t-1},\, s_{t-1},\, a_{t-1},\, s_{t}) \in \cZ:=\cD\times\cS\times\cA\times\cS,
\]
where $\cD$ denotes the set of feasible $D$-matrices (diagonal matrices corresponding to probability distributions over $\Delta(\cS\times \cA)$). All states in $z_t$ are in recurrent class $\cC$ since the process has step in $\cC$ and never get out. 

The invariant distribution of the augmented process is concentrated on the limiting occupancy $D$ is given by\[
\mu_z(\tilde D,s,a,s') \;=\; \mathbb{I}_{\set{\tilde D=D}}\;\rho(s)\behpi(a | s)\;\overline p(s'| s,a),
\]
where $\overline p$ is the lazy transition kernel.

We now define the operators $F^{\mathrm{exp}},F^{\mathrm{imp}} :\R^{|\cS||\cA|}\times\cZ\to \R^{|\cS||\cA|}$ as follows. For
$z=(D,s_0,a_0,s_1)\in\cZ$, the $(s,a)$-th entry of the output $F^{\mathrm{exp}}(Q, z),F^{\mathrm{imp}}(Q, z)$ are given by
\begin{align}
    F^{\mathrm{exp}}(Q, z) (s,a) =& \frac{\mathbb{I}_{\set{(s_0, a_0) = (s,a)}}}{D(s,a)} \bracket{r(s_0,a_0) + \max_{a'\in \cA} Q(s_1, a') - Q(s_0, a_0)} + Q(s,a)\label{equ:app:async_f_exp}\\
    F^{\mathrm{imp}}(Q, z)(s,a) =& \frac{\mathbb{I}_{\set{(s_0, a_0) = (s,a)}}}{D(s,a)}\bracket{r(s_0, a_0) + \frac{1}{2}\max_{a'\in \cA}Q(s_0, a') + \frac{1}{2}\max_{a'\in \cA}Q(s_1, a') - Q(s_0, a_0)} + Q(s,a)\label{equ:app:async_f_imp}
\end{align}
Here $F^{\mathrm{exp}}(Q, z)$ is used in the analysis of Algorithm~\ref{alg:async_q_explicit_lazy_sampling}, while $F^{\mathrm{imp}}(Q, z)$ is used for Algorithm~\ref{alg:async_q_implicit_lazy_sampling}. Intuitively, $F(Q,z)$ updates only on the visited coordinate $(s_0,a_0)$ and leaves all other entries unchanged. 

With equations~\eqref{equ:app:async_f_exp} and \eqref{equ:app:async_f_imp}, the asynchronous update for $Q_t$ can be written in the stochastic-approximation form
\begin{equation}
	\label{equ:async_q_learning_f_update}
Q_{t} = (1-\lambda_{t}) Q_{t-1} + \lambda_{t} F^{\eta}(Q_{t-1}, z_{t})\qquad
\lambda_t := \frac{\lambda^*}{t+h}.
\end{equation}
where $\eta\in\set{\mathrm{exp}, \mathrm{imp}}$ indicates whether $Q_t$ is updated using the explicit or implicit algorithms.

Define
\begin{align*}
    \delta_t^{\mathrm{exp}} :=& r(s_{t-1},a_{t-1})+\max_{a'\in\cA}Q_{t-1}(s_{t},a')-Q_{t-1}(s_{t-1},a_{t-1})\\
    \delta_t^{\mathrm{imp}} :=& r(s_{t-1},a_{t-1})+\frac{1}{2}\max_{a'\in\cA}Q_{t-1}(s_{t-1},a')+\frac{1}{2}\max_{a'\in\cA}Q_{t-1}(s_{t},a')-Q_{t-1}(s_{t-1},a_{t-1}),
\end{align*}
we have
\begin{align*}
	Q_{t}(s_{t-1},a_{t-1})
=& Q_{t-1}(s_{t-1},a_{t-1}) + \frac{\lambda^*}{t+h}\cdot\frac{1}{D_{t-1}(s_{t-1},a_{t-1})}\,\delta^{\eta}_{t}\\
=& Q_{t-1}(s_{t-1},a_{t-1}) + \frac{\lambda}{N_{t-1}(s_{t-1},a_{t-1})+h}\,\delta_{t}^{\eta},
\end{align*}
and $Q_{t}(s,a)=Q_{t-1}(s,a)$ for all $(s,a)\neq(s_{t-1},a_{t-1})$.

The operator $F^\eta$ satisfies the following properties.
\begin{lemma}[Lemm 6.1 in \citet{chen2025nonasymptotic}]
	\label{lem:f_operator_properties}
	Thee following properties holds regarding the operator $F^{\eta}(\cdot)$ for $\eta\in\set{\mathrm{exp}, \mathrm{imp}}$.
	\begin{enumerate}
		\item (\emph{Lipschitz in span.}) For any $Q_1,Q_2$ and any $z=(\tilde D,s,a,s')\in\cZ$,
\[
\spannorm{F^{\eta}(Q_1,z)-F^{\eta}(Q_2,z)}\le \frac{2}{\tilde D(s,a)}\,\spannorm{Q_1-Q_2}.
\]
		\item (\emph{Linear growth.}) For any $Q$ and any $z=(\tilde D,s,a,s')\in\cZ$,
\[
\spannorm{F^{\eta}(Q,z)}\le \frac{2}{\tilde D(s,a)}\bigl(\spannorm{Q}+1\bigr).
\]
		\item (\emph{Unbiased estimator.}) The operator $F^\eta$ is an unbiased estimator for the Bellman opertor under $\overline{P}$ with for all $Q$,
\begin{equation}
	\label{equ:f_operator_unbiased}
	\mathbb E_{z\sim\mu_z}\big[F^{\eta}(Q,z)(s,a)\big] = \cT_{\overline P}(Q)(s,a).
\end{equation}
	\end{enumerate}
\end{lemma}
The proof for~\ref{lem:f_operator_properties} is provided in~\ref{sec:app:proof_of_lem:f_operator_properties}.

We will employ a Lyapunov--drift approach to perform the finite-time analysis. Since we have shown that there exists a seminorm $\spanstar{\cdot}$ under which the lazy Bellman operator $\cT_{\overline{P}}$ is a contraction, we introduce an associated smooth Lyapunov function via the infimal convolution.

\begin{definition}
	\label{def:infimal_convolution}
	Let $f_1,f_2:\cD\to(-\infty,\infty]$ be two proper functions defined on
	$\cD\subseteq\R^d$. The \emph{infimal convolution} of $f_1$ and $f_2$ is defined
	as
	\[
	(f_1\square f_2)(x)
	:= \inf_{y\in\cD}\bigl\{f_1(y)+f_2(x-y)\bigr\},
	\qquad \forall x\in\cD.
	\]
\end{definition}

We define an instance-dependent Lyapunov function as a generalized Moreau envelope, given by the infimal convolution of the squared seminorm $\spanstar{\cdot}^2$ and the squared $\ell_q$-norm:
\begin{equation}
	\label{equ:lyapunov_function}
	M_{q,\theta}(Q)
	:= \inf_{\mu}
	\Bigl\{\frac{1}{2}\spanstar{\mu}^2
	+ \frac{1}{2\theta}\|Q-\mu\|_q^2\Bigr\},
	\qquad Q\in\R^{|\cS||\cA|},
\end{equation}
where $q\ge 1$ and $\theta>0$ are tunable parameters to be specified later. When there is no ambiguity, we write $M(\cdot)$ in place of $M_{q,\theta}(\cdot)$.

Let $l_q=(|\cS||\cA|)^{-1/q}$. Then for all $Q$,
\[
l_q\,\|Q\|_q \le \|Q\|_\infty \le \|Q\|_q.
\]

\begin{lemma}
	\label{lem:sp_env_seminorm_with_min_norm_version}
	There exists a norm
	\[
	\norm[\widetilde{sp}]{Q}
	:= \spanstar{Q} + \frac{1}{|\cS||\cA|}\bigl|\mathbf 1^\top Q\bigr|
	\]
	such that
	\[
	\spanstar{Q} = \min_{c\in\R}\norm[\widetilde{sp}]{Q-c\mathbf 1}.
	\]
\end{lemma}

\begin{lemma}
	\label{lem:sp_env_equivlent_to_l_q}
	The norm $\norm[\widetilde{sp}]{\cdot}$ is equivalent to $\|\cdot\|_q$, in the
	sense that $l_q = (|\cS||\cA|)^{-1/q}$, $u_q = 5$ satisfy
	\[
	l_q\|Q\|_q
	\le \norm[\widetilde{sp}]{Q}
	\le u_q\|Q\|_q,
	\qquad \forall Q\in\R^{|\cS||\cA|}.
	\]
\end{lemma}
The proof of Lemma~\ref{lem:sp_env_seminorm_with_min_norm_version} and Lemma~\ref{lem:sp_env_equivlent_to_l_q} are provided in~\ref{sec:app:proof_of_lem:sp_env_seminorm_with_min_norm_version} and~\ref{sec:app:proof_of_lem:sp_env_equivlent_to_l_q} respectively.

\begin{lemma}
	\label{lem:lyapunov_property}
	The instance-dependent Lyapunov function $M(\cdot)$ has the following properties:
	\begin{enumerate}
		\item The function $M(\cdot)$ is convex and differentiable, and for all $Q_1,Q_2\in\R^{|\cS||\cA|}$,
		\[
		M(Q_2) \leq M(Q_1) + \innerprod{\nabla M(Q_1)}{Q_2 - Q_1} + \frac{L}{2}\spanstar{Q_2 - Q_1}^2
		\]
		where $L=(q-1)/(\theta l_q^2)$.
		\item There exists a seminorm $\seminorm[M]{\cdot}$ such that $M(Q)=\frac12\seminorm[M]{Q}^2$ and
		\[
		\seminorm[M]{Q}
		= \min_{c\in\R}\norm[M]{Q-c\mathbf 1}
		\]
		for some norm $\norm[M]{\cdot}$. Moreover, letting
		$l_q = (|\cS||\cA|)^{-1/q}$, $u_q = 5$, $l_*=\sqrt{1+\theta l_q^2}$ and $u_*=\sqrt{1+\theta u_q^2}$, we have
		\[
		l_*\seminorm[M]{Q}\le \spanstar{Q}\le u_*\seminorm[M]{Q},
		\qquad
		l_*^2 M(Q)\le \tfrac12\spanstar{Q}^2\le u_*^2 M(Q).
		\]
		\item It holds for all $Q\in \R^{|\cS||\cA|}$ and $c\in \R$ that $\innerprod{\nabla M(Q)}{c\mathbf 1} = 0$
		\item It holds for all $Q_1, Q_2, Q_3\in \R^{|\cS||\cA|}$ that 
		\[
		\innerprod{\nabla M(Q_1) - \nabla M(Q_2)}{Q_3}\leq L\spanstar{Q_1 - Q_2} \spanstar{Q_3}
		\]
	\end{enumerate}
\end{lemma}
The proof for Lemma~\ref{lem:lyapunov_property} is provided in~\ref{sec:app:proof_of_lem:lyapunov_property}

\paragraph{Lyapunov drift decomposition.} Using item~(1) of Lemma~\ref{lem:lyapunov_property}, we obtain the following Lyapunov drift decomposition. For all $t\ge 1$,
\begin{equation}
    \label{equ:lyapunov_drift_decomposition}
    \begin{aligned}
        \E\sqbk{M(Q_{t} - \overline{Q}^*)} \leq& \E \sqbk{M(Q_{t-1} - \overline{Q}^*)} + \E\sqbk{\innerprod{\nabla M(Q_{t-1} - \overline Q^*)}{Q_{t} - Q_{t-1}}}\\
	& + \frac{L}{2} \E\sqbk{\spanstar{Q_t - Q_{t-1}}^2}\\
	\overset{(i)}{=}& \E\sqbk{M(Q_{t-1} - \overline{Q}^*)} + \lambda_{t-1} \E\sqbk{\innerprod{\nabla M(Q_{t-1} - \overline Q^*)}{F^{\eta}(Q_{t-1}, z_{t-1}) - Q_{t-1}}}\\
	& + \frac{L\lambda_{t-1}^2}{2}\E\sqbk{\spanstar{F^{\eta}(Q_{t-1}, z_{t-1}) - Q_{t-1}}^2}\\
	=& \E\sqbk{M(Q_{t-1} - \overline{Q}^*)} + \lambda_{t-1}\underbrace{ \E\sqbk{\innerprod{\nabla M(Q_{t-1} - \overline{Q}^*)}{\cT_{\overline{P}}(Q_{t-1}) - Q_{t-1}}}}_{T_1} \\
	& + \lambda_{t-1}\underbrace{ \E\sqbk{\innerprod{\nabla M(Q_{t-1} - \overline{Q}^*)}{F^{\eta}(Q_{t-1}, D, y_{t-1}) - \cT_{\overline{P}}(Q_{t-1})}}}_{T_2^{\eta}} \\
	& + \lambda_{t-1}\underbrace{\E\sqbk{\innerprod{\nabla M(Q_{t-1} - \overline{Q}^*)}{F^{\eta}(Q_{t-1}, D_{t-1}, y_{t-1}) - F^{\eta}(Q_{t-1}, D, y_{t-1})}}}_{T_3^{\eta}}\\
	& + \frac{L\lambda_{t-1}^2}{2}\underbrace{\E\sqbk{\spanstar{F^{\eta}(Q_{t-1}, z_{t-1}) - Q_{t-1}}^2}}_{T_4^{\eta}}
    \end{aligned}
\end{equation}
where $(i)$ is by equation~\ref{equ:async_q_learning_f_update}, $z_{t-1} = (D_{t-1}, y_{t-1}) = (D_{t-1}, s_{t-1}, a_{t-1}, s_{t})$ and $\eta\in \set{\mathrm{exp}, \mathrm{imp}}$. The following parts bound the four subterms $T_1, T^{\eta}_2, T^{\eta}_3, T^{\eta}_4$ respectively.

\subsubsection{Bounding \texorpdfstring{$T_1$}{T1}}
\begin{lemma}
	\label{lem:async_T_1_bound}
	For all $t\ge 1$, it holds that
	\[
	T_1 \le -2\bracket{1-\frac{\beta u_*}{l_*}}\,
	\E\sqbk{M(Q_t-\overline Q^*)}.
	\]
\end{lemma}
The proof of Lemma~\ref{lem:async_T_1_bound} is provided in
Section~\ref{sec:app:proof_of_lem:async_T_1_bound}.

\subsubsection{Bounding \texorpdfstring{$T_2^{\eta}$}{T2 (eta)}}
\begin{lemma}
    \label{lem:async_T_2_bound}
	There exists a constant $T_{C_2}$ depends on $\lambda^*$, $p_\wedge$, $r^{\eta}$, when $t\geq T_{C_2}$, the subterm $T_{2}^{\eta}$ can be bounded as:
	\[
	T_2^{\eta} \leq \frac{112 L \tau_{t-1}^{\eta}|\cS||\cA|\bracket{b_{t-1} +\spannorm{\overline{Q}^*} + 1}^2}{p_\wedge^2}\lambda_{t-1}
	\]
\end{lemma}
The proof of Lemma~\ref{lem:async_T_2_bound} is provided in Section~\ref{sec:app:proof_of_lem:async_T_2_bound}.

\subsubsection{Bounding \texorpdfstring{$T_3^{\eta}$}{T3 (eta)}}
\begin{lemma}
	\label{lem:async_T_3_bound}
	There exists a constant $T_{C_3}$ depends on $\lambda^*$, $p_\wedge$, $r^{\eta}$, when $t\geq T_{C_3}$, the subterm $T_3^{\eta}$ satisfies
	\[
	T_3^{\eta} \leq \bracket{1-\frac{\beta u_*}{l_*}} \E\sqbk{M(Q_{t-1} - \overline{Q}^*)} + \frac{63C|\cS||\cA|(b_{t-1} + 1)^2}{l_*^2 (1-\beta u_*/l_*) (1-r^{\eta})\lambda^*p_\wedge^2}\lambda_{t-1}.
	\]
\end{lemma}
The proof of Lemma~\ref{lem:async_T_3_bound} is provided in Section~\ref{sec:app:proof_of_lem:async_T_3_bound}.

\subsubsection{Bounding \texorpdfstring{$T_4^{\eta}$}{T4 (eta)}}
\begin{lemma}
	\label{lem:async_T_4_bound}
	The subterm $T_4^{\eta}$ can be bounded as:
	\[
	T_4^{\eta} \leq \frac{12 |\cS||\cA|(b_{t-1}^2 + 1)}{p_\wedge^2}.
	\]
\end{lemma}
The proof of Lemma~\ref{lem:async_T_4_bound} is provided in Section~\ref{sec:app:proof_of_lem:async_T_4_bound}.

We now combine the bounds for $T_1$, $T_2^{\eta}$, $T_3^{\eta}$, and $T_4^{\eta}$ to derive an overall recursion for the Lyapunov error $\E\sqbk{M(Q_t - \overline{Q}^*)}$.
\begin{align*}
	T_1 \leq& -2\bracket{1-\frac{\beta u_*}{l_*}} M(Q_{t-1} - \overline{Q}^*) \\
	T_2^{\eta} \leq& \frac{112 L \tau_{t-1}^{\eta}|\cS||\cA|\bracket{b_{t-1} +\spannorm{\overline{Q}^*} + 1}^2}{p_\wedge^2}\lambda_{t-1}\\
	T_3^{\eta} \leq& \bracket{1-\frac{\beta u_*}{l_*}} \E\sqbk{M(Q_{t-1} - \overline{Q}^*)} + \frac{63C|\cS||\cA|(b_{t-1} + 1)^2}{l_*^2 (1-\beta u_*/l_*) (1-r^{\eta})\lambda^*p_\wedge^2}\lambda_{t-1}\\
	T_4^{\eta} \leq& \frac{12 |\cS||\cA| (b_{t-1} + 1)^2}{p_\wedge^2}
\end{align*}

Define the constants
\[
\phi_1 = 1-\frac{\beta u_*}{l_*}, \quad \phi_2 = \frac{|\cS||\cA|}{p_\wedge^2}\bracket{118L + \frac{63C}{l_*^2(1-r^{\eta})\phi_1 \lambda^*}}
\]
Then, for all $t\geq \max\set{T_{C_2}^{\eta}, T_{C_4}^{\eta}, T_{C_4}^{\eta}}$, we have
\begin{align*}
	&\E\sqbk{M(Q_{t} - \overline{Q}^*)}\\
	\leq& \E\sqbk{M(Q_{t-1} - \overline{Q}^*)} + \lambda_{t-1} (T_1 + T_2^{\eta} + T_3^{\eta}) + \frac{L\lambda_{t-1}^2}{2}T_4^{\eta}\\
	\leq& \bracket{1 - \phi_1\lambda_{t-1}} \ E\sqbk{M(Q_{t-1} - \overline{Q}^*)}\\
	& + \frac{|\cS||\cA|}{p_\wedge^2}\bracket{112L\tau_{t-1}^{\eta} + \frac{63C}{l_*^2\phi_1(1-r^{\eta})\lambda^*} + 6}\bracket{b_{t-1} + \spannorm{\overline{Q}^*} + 1}^2\lambda_{t-1}^2\\
	\leq& \bracket{1 - \phi_1\lambda_{t-1}} \ E\sqbk{M(Q_{t-1} - \overline{Q}^*)}\\
	& + \frac{|\cS||\cA|}{p_\wedge^2}\bracket{118L + \frac{63C}{l_*^2\phi_1(1-r^{\eta})\lambda^*}}\tau_{t-1}^{\eta}\bracket{b_{t-1} + \spannorm{\overline{Q}^*} + 1}^2\lambda_{t-1}^2\\
	=& \bracket{1-\phi_1 \lambda_{t-1}}\E\sqbk{M(Q_{t-1} - \overline{Q}^*)} + \phi_2^{\eta}\tau_{t-1}^{\eta}\bracket{b_{t-1} + \spannorm{\overline{Q}^*} + 1}^2 \lambda_{t-1}^2.
\end{align*}

\paragraph{Solving the Recursion.} Repeatedly applying the above inequality, we have for all $t\geq T_C^{\eta}:=\max\set{T_{C_2}^{\eta}, T_{C_3}^{\eta}, T_{C_4}^{\eta}}$ that:
\begin{equation*}
	\E \sqbk{M(Q_t - \overline{Q}^*)} \leq \prod_{j=T_{C}^{\eta}}^{t-1}(1-\phi_1 \lambda_j)\E\sqbk{M(Q_{T_{C}^{\eta}} - \overline{Q}^*)}+ \phi_2^{\eta} \sum_{i=T_{C}^{\eta}}^{t-1}\tau_{i}^{\eta} \bracket{b_i + \spannorm{\overline{Q}^*} + 1}^2 \lambda_i^2 \prod_{j=i+1}^{t-1}(1-\phi_1 \lambda_j).
\end{equation*}
We next translate the error measure into the seminorm $\spanstar{\cdot}$. By item~(2) of Lemma~\ref{lem:lyapunov_property},

\[
l_*^2 M(\cdot)\leq\tfrac{1}{2}\spanstar{\cdot}^2\leq u_*^2M(\cdot),
\]
which yields
\begin{align*}
	&\E\sqbk{\spanstar{Q_t - \overline Q^*}^2}\\
	\leq&2u_*^2\E\sqbk{M(Q_t - \overline Q^*)}\\
	\leq& 2u_*^2 \prod_{j=T_{C}^{\eta}}^{t-1}(1-\phi_1 \lambda_j)\E\sqbk{M(Q_{T_{C}^{\eta}} - \overline Q^*)} + 2u_*^2\phi_2^{\eta} \sum_{i=T_{C}^{\eta}}^{t-1}\tau_i^{\eta} (b_i + \spannorm{Q^*} + 1)^2 \lambda_i^2 \prod_{j=i+1}^{t-1}(1-\phi_1 \lambda_j)\\
	\leq& \frac{u_*^2}{l_*^2}\prod_{j=T_{C}^{\eta}}^{t-1}(1-\phi_1 \lambda_j)\E\sqbk{\spanstar{Q_{T_{C}^{\eta}} - \overline Q^*}^2}\\
	&+ 2u_*^2\phi_2^{\eta} \sum_{i=T_{C}^{\eta}}^{t-1}\tau_i^{\eta} \bracket{b_i + \spannorm{\overline Q^*} + 1}^2 \lambda_i^2 \prod_{j=i+1}^{t-1}(1-\phi_1 \lambda_j)\\
	\leq& \frac{4u_*^2}{l_*^2}\bracket{b_{T_{C}^{\eta}} + \spannorm{\overline Q^*}}^2 \prod_{j=T_{C}^{\eta}}^{t-1}\bracket{1-\phi_1 \lambda_j}\\
	&+ 2u_*^2\phi_2^{\eta}\tau_{t-1}^{\eta} \bracket{b_{t-1} + \spannorm{\overline Q^*} +1}^2 \sum_{i=T_{C}^{\eta}}^{t-1}\lambda_i^2 \prod_{j=i+1}^{t-1}\bracket{1- \phi_1 \lambda_j}.
\end{align*}
Recall $l_q = (|\cS||\cA|)^{-1/q}$ and $u_q = 5$, we choose
\[
\begin{cases}
	\theta =& \frac{1-\beta}{100}\\
	q =& 2\log |\cS||\cA|,
\end{cases}
\]
which implies $l_q = \tfrac{1}{\sqrt{e}}$ and $u_q = 5$. Consequently,
\begin{align}
	\theta \bracket{u_q^2 - \bracket{\frac{1+\beta}{2\beta}}^2 l_q^2} \leq& 25\theta = \frac{1-\beta}{4} \leq \frac{(1-\beta)(1+3\beta)}{4\beta^2} = \bracket{\frac{1+\beta}{2\beta}}^2 - 1\notag\\
	1 + \theta u_q^2 - \theta\bracket{\frac{1+\beta}{2\beta}}^2l_q^2\leq& \bracket{\frac{1+\beta}{2\beta}}^2\notag\\
	1 + \theta u_q^2 \leq& \bracket{1 + \theta l_q^2}\bracket{\frac{1+\beta}{2\beta}}^2\notag\\
	\frac{u_*^2}{l_*^2} = \frac{1 + \theta u_q^2}{1 + \theta l_q^2} \leq \bracket{\frac{1+\beta}{2\beta}}^2 \implies& 1- \frac{\beta \sqrt{1 + \theta u_q^2}}{\sqrt{1 + \theta l_q^2}} \geq 1-\frac{1+\beta}{2}\geq \frac{1-\beta}{2}\label{equ:app:u_star_to_l_star}\\
	\implies& \phi_1 \geq \frac{1-\beta}{2}.\notag
\end{align}
Moreover,
\begin{align*}
	L =& \frac{q- 1}{\theta l_q^2} \leq \frac{600 \log(|\cS||\cA|)}{1-\beta}\\
	2u_*^2 \phi_2^{\eta} \leq& 2\bracket{1 + \frac{(1-\beta)}{25}}\frac{|\cS||\cA|}{p_\wedge^2}\bracket{\frac{118\cdot 600 \log(|\cS||\cA|)}{1-\beta} + \frac{63C}{l_*^2(1-r^{\eta})\phi_1 \lambda^*}}\\
	\leq&\frac{3|\cS||\cA|}{p^2_\wedge(1-\beta)}\bracket{118\cdot 600 \log(|\cS||\cA|) + \frac{63C}{(1-r^{\eta})\lambda^*}}
\end{align*}
Let $C^{\eta} = \max\set{3\cdot118\cdot 600 (1-r^{\eta}), \frac{189C}{\lambda^*}}$, we have:
\[
2u_*^2 \phi_2^{\eta} \leq \frac{C^{\eta}|\cS||\cA|\log(|\cS||\cA|)}{p_\wedge^2(1-\beta)(1-r^{\eta})}.
\]
Then, since $\beta \geq 1/2$, 
\[
\frac{u_*^2}{l_*^2} \leq \frac{9}{4},
\]
Therefore,
\begin{align*}
	\E\sqbk{\spanstar{Q_t - \overline{Q}^*}^2}\leq& 9\bracket{b_{T_{C}^{\eta}} + \spannorm{\overline{Q}^*}}^2 \underbrace{\prod_{j=T_{C}^{\eta}}^{t-1}\bracket{1 - \frac{1-\beta}{2}\lambda_j}}_{E_1}\\
	& + \frac{C^{\eta} |\cS||\cA|\log(|\cS||\cA|)}{(1-\beta)(1-r^{\eta})p_\wedge^2 }\tau_{t-1}^{\eta}\bracket{b_{t-1} + \spannorm{\overline{Q}^*} + 1}^2\underbrace{\sum_{i=T_{C}^{\eta}}^{t-1}\lambda_i^2 \prod_{j=i+1}^{t-1}\bracket{1 - \frac{1-\beta}{2}\lambda_j}}_{E_2}. 
\end{align*}
We now bound the two terms $E_1$ and $E_2$ separately. For the first product,
\begin{equation*}
	E_1 = \prod_{j=T_{C}^{\eta}}^{t-1}\bracket{1 - \frac{1-\beta}{2}\lambda_j}
	\leq \exp\bracket{-\frac{1-\beta}{2}\sum_{j=T_{C}^{\eta}}^{t-1}\lambda_j}
	\leq \exp\bracket{-\frac{1-\beta}{2}\int_{T_{C}^{\eta}}^{t} \frac{\lambda}{x + h} dx}
	\leq \bracket{\frac{T_{C}^{\eta} + h}{t + h}}^{\frac{1-\beta}{2}\lambda^*}
\end{equation*}
Similarly, w bound the second term $E_2$. By the inequality $1-x \leq e^{-x}$ and the definition of $\lambda_t = \lambda^* / (t + h)$, we have
\begin{align}
	E_2 =& \sum_{i=T_{C}^{\eta}}^{t-1}\lambda_i^2 \prod_{j=i+1}^{t-1}\bracket{1 - \frac{1-\beta}{2}\lambda_j}\notag\\
	\leq& \sum_{i=T_{C}^{\eta}}^{t-1}\lambda_i^2 \exp\bracket{-\frac{1-\beta}{2}\sum_{j=i+1}^{t-1}\lambda_j}\notag\\
	\leq& \sum_{i=T_{C}^{\eta}}^{t-1}\lambda_i^2 \bracket{\frac{i+1 + h}{t+h}}^{\frac{1-\beta}{2}\lambda^*}\notag\\
	\leq& \frac{4\lambda^{*2}}{(t + h)^{\frac{1-\beta}{2}\lambda^*}}\sum_{i=T_{C}^{\eta}}^{t-1}\frac{1}{(i + 1 + h)^{2 - \frac{1-\beta}{2}\lambda^*}}\notag\\
	\leq& \frac{4\lambda^{*2}}{(t + h)^{\frac{1-\beta}{2}\lambda^*}} \left\{
	\begin{aligned}
		&\frac{1}{\frac{1-\beta}{2}\lambda^*- 1}\frac{1}{(t + 1 + h)^{1-\frac{1-\beta}{2}\lambda^*}}
		&& \text{if } 2 - \frac{1-\beta}{2}\lambda^*< 1\\
		&\log\bracket{\frac{t + h}{T_{C}^{\eta} + 1 + h}}
		&& \text{if } 2 - \frac{1-\beta}{2}\lambda^*= 1\\
		&\frac{1}{1-\frac{1-\beta}{2}\lambda^*} \frac{1}{(T_{C}^{\eta} + 1 + h)^{1-\frac{1-\beta}{2}\lambda^*}}
		&& \text{if } 2 - \frac{1-\beta}{2}\lambda^*> 1
	\end{aligned}
	\right.\notag\\
	\leq& 4\lambda^{*2}\left\{
	\begin{aligned}
		&\frac{1}{(t + h)(\frac{1-\beta}{2}\lambda^*- 1)}
		&& \text{if } \frac{1-\beta}{2}\lambda^*>1 \\
		&\frac{\log(t + h)}{t + h}
		&& \text{if } \frac{1-\beta}{2}\lambda^*= 1\\
		&\frac{1}{(t + h)^{\frac{1-\beta}{2}\lambda^*}(1-\frac{1-\beta}{2}\lambda^*)}
		&& \text{if } \frac{1-\beta}{2}\lambda^*< 1
	\end{aligned}
	\right. \label{equ:app:async_convergence_rate_case_by_case}
\end{align}
When $\frac{1-\beta}{2}\lambda^*\ge 1$, both $E_1$ and $E_2$ decay at a rate $O(1/t)$ (corresponding to the first two cases of \eqref{equ:app:async_convergence_rate_case_by_case}). We distinguish the boundary and strict cases as follows.

If $\frac{1-\beta}{2}\lambda^* = 1$, then
\begin{align*}
	\E\sqbk{\spanstar{Q_t - \overline{Q}^*}^2} \leq& 9\bracket{b_{T_{C}^{\eta}} + \spannorm{\overline{Q}^*}}^2 \bracket{\frac{T_{C}^{\eta} + h}{t + h}}\\
	& + \frac{4C^{\eta}{\lambda^{*}}^2 |\cS||\cA|\log(|\cS||\cA|)}{(1-\beta)(1-r^{\eta})^3p_\wedge^2}\frac{\tau_t^{\eta}\bracket{b_t + \spannorm{\overline{Q}^*} + 1}^2\log(t + h)}{(t+h)}\quad \text{if }\frac{1-\beta}{2}\lambda^* = 1.
\end{align*}

If $\frac{1-\beta}{2}\lambda^* > 1$, then
\begin{align*}
	\E\sqbk{\spanstar{Q_t - \overline{Q}^*}^2} \leq& 9\bracket{b_{T_{C}^{\eta}} + \spannorm{\overline{Q}^*}}^2 \bracket{\frac{T_{C}^{\eta} + h}{t + h}}^{\bracket{\frac{1-\beta}{2}\lambda^*}}\\
	& + \frac{4C^{\eta} {\lambda^*}^2|\cS||\cA|\log(|\cS||\cA|)}{(1-\beta)(1-r^{\eta})^3p_\wedge^2}\frac{\tau_t^{\eta}\bracket{b_t + \spannorm{\overline{Q}^*} + 1}^2}{(t+h)\bracket{\frac{1-\beta}{2}\lambda^* - 1}}\quad\text{if }\frac{1-\beta}{2}\lambda^*> 1.
\end{align*}
In both cases, the above bounds imply that the estimator $Q_t$ converges to $\overline Q^*$ in mean-square sense at a rate
$\widetilde{O}(1/t)$.

\subsection{Aggregate the proof of Theorem~\ref{thm:async_q_learning_optimal_rate_explicit} and~\ref{thm:async_q_learning_optimal_rate_implicit}}
\label{app:async:aggregate_the_proof}

We now aggregate the preceding results to complete the proofs of Theorems~\ref{thm:async_q_learning_optimal_rate_explicit} and~\ref{thm:async_q_learning_optimal_rate_implicit}.

From this point on, $T$ denotes the total number of iterations of the original algorithms, and the previous bound (indexed by $t$) are applied to the post-entrance segment of length $T-\tau$.

Recall that, by Lemma~\ref{lem:restart_in_class}, once the trajectory enters $\cC$ at time $\tau$, the subsequent evolution over the interval $\{ \tau,\tau+1,\ldots,T-1\}$ is distributionally equivalent to running the algorithm from an initial state in $\cC$ for $T-\tau$ steps. Therefore, the bounds derived for $\E[M(Q_t-\overline Q^*)]$ can be applied with $t = T-\tau$, and we obtain the following global error bound by conditioning on the event $\{\tau \le T/2\}$.

With 
\[
\beta = \bracket{1-\frac{1}{K2^K}}^{1/(K+1)}\quad \text{and}\quad \lambda^*= h = K(K+1)2^{K+2},
\]
we have
\begin{align*}
\frac{1-\beta}{2}\lambda^* =& \frac{1}{2}\bracket{1 - \bracket{1-\frac{1}{K2^K}}^{K+1}}K(K+1)2^{K+2}\\
\geq& 2\bracket{1-\bracket{1-\frac{1}{K2^K}}^{1/(K+1)}}K(K+1)2^{K}\\
\geq& 2 + (1-\frac{1}{K+1})\frac{1}{K2^K}\\
>& 2,
\end{align*}
which implies $\frac{1-\beta}{2}\lambda^* - 1 > 1$. 

By Lemma~\ref{lem:restart_in_class}, let $T$ denote the total number of iterations of the original algorithm. On the event $\{\tau \le T/2\}$, we have $\mathbb{I}_{\{\tau \le T/2\}}=1$ and $T-\tau \ge T/2$. Substituting the bounds derived in~\eqref{equ:mc_step_to_recurrent_class_decomposition} with $t=T-\tau \ge T/2$, and using Lemma~\ref{lem:app:q_star_span_bound} together with Lemma~\ref{lem:lazy_doubles_hitting_time}, which yields $\spannorm{\overline Q^*}\le 4K+1$, we obtain
\begin{align*}
	&\E\sqbk{\spannorm{Q_T|_{\cC} - \overline{Q}|_{\cC}^*}|_{\cC}^2}  \\
	\leq &18\bracket{b_{T_{C}^{\eta}} + 4K+1}^2 \bracket{\frac{T_{C}^{\eta} + h}{T/2 + h}}^{\frac{1-\beta}{2}\lambda^*}+ \frac{4C^{\eta} {\lambda^*}^2|\cS||\cA|\log(|\cS||\cA|)}{(1-\beta)(1-r^{\eta})^3p_\wedge^2}\frac{\tau_t^{\eta}\bracket{b_t + 4K+2}^2}{(T /2+h)(\frac{1-\beta}{2}\lambda^* - 1)} \\
	& + C (r^{\eta})^{T/2}\cdot \bracket{\lambda^* |\cS||\cA|\log\bracket{\frac{T/(|\cS||\cA|) + h}{h}} + 4K+1}^2.
\end{align*}

Since the last term decays exponentially with $T$, there exists a constant $T_C^{\eta'}$ such that for all $T\geq \max\set{2T_C^{\eta}, T_C^{\eta'}}$,
\begin{align*}
	18\bracket{b_{T_{C}^{\eta}} + 4K+1}^2 \bracket{\frac{T_{C}^{\eta} + h}{T/2 + h}}^{\frac{1-\beta}{2}\lambda^*} \geq& 2C (r^{\eta})^{T/2}\cdot \bracket{\lambda^* |\cS||\cA|\log\bracket{\frac{T/(|\cS||\cA|) + h}{h}} + 4K+1}^2,\\
	18\bracket{b_{T_{C}^{\eta}} + 4K+1}^2 \bracket{\frac{T_{C}^{\eta} + h}{T/2 + h}}^{\frac{1-\beta}{2}\lambda^*}\leq& \frac{4C^{\eta} {\lambda^*}^2|\cS||\cA|\log(|\cS||\cA|)}{(1-\beta)(1-r^{\eta})^3p_\wedge^2}\frac{\tau_t^{\eta}\bracket{b_t + 4K+2}^2}{(T /2+h)(\frac{1-\beta}{2}\lambda^* - 1)}.
\end{align*}
Similar in the synchronous case, let
\[
Q_T^{\mathrm{corr}}(s,a) = Q_{T}(s,a) - \max_{a'\in \cA} Q_T(s,a').
\]
Let
\begin{align*}
	\tau_t^{\eta} = & \frac{1}{1-r^{\eta}}\log\bracket{\frac{C(t+ h)}{\lambda^*}} \geq \frac{\log(\lambda^*/(C(t + h)))}{\log(r)},\\
	b_t =& \lambda^*|\cS||\cA|\log\bracket{\frac{t/(|\cS||\cA|) + h}{h}},\\
	\lambda^* =& K(K+1)2^{K+2}, \quad h = K(K+1)2^{K+2},
\end{align*}
Then by Lemma~\ref{lem:bound_q_error_with_lazy_q}, when $T\geq T_{C}^{\eta''} = \max\set{2T_C^{\eta}, T_{C}^{\eta'}}$, for some constnat $C^{\eta} > 0$
\[
\E\sqbk{\spannorm{Q_{T}^{\mathrm{corr}}|_{\cC} - Q^*|_{\cC}}^2} \leq \E\sqbk{2\spannorm{Q_{T}|_{\cC} - \overline{Q}^*|_{\cC}}^2} \leq \frac{C^{\eta}K^{10}2^{5K}|\cS|^3|\cA|^3\log(|\cS||\cA|)\log^2 T}{p_\wedge^2(1-r^{\eta})T}.
\]

Finally we conclude
\begin{itemize}
	\item For Theorem~\ref{thm:async_q_learning_optimal_rate_explicit} with $\eta = \mathrm{exp}$, by Lemma~\ref{lem:async:lazy_radius} with $\frac{1}{1 - r^{\mathrm{exp}}} \leq (2K+2)^22^{2K+2}$, and recall $C^{\mathrm{exp}} = \max\set{3\cdot118\cdot 600 (1-r^{\mathrm{exp}}), \frac{189C}{\lambda^*}}$, then for some universal constant $C_1 > 0$,
	\[
	\E\sqbk{\spannorm{Q_{T}^{\mathrm{corr}}|_{\cC} - Q^*|_{\cC}}^2} \leq \frac{C_12^{8K}|\cS|^3|\cA|^3\log(|\cS||\cA|)\log^2 T}{p_\wedge^2T}.
	\]
	Therefore, there exists an universal constant $C_1$, and $T_{C}^{\mathrm{exp}''}$ that are independent on $\varepsilon$ (only depends on $p_\wedge$, $K$), when
	\[
	T\geq \max\set{T_C^{\mathrm{exp}''}, \frac{C_1\,2^{8K}|\cS|^3|\cA|^3\log(|\cS||\cA|)\log^2 T}{p_\wedge^2\varepsilon^2}} \implies \E\sqbk{\spannorm{Q_{T}^{\mathrm{corr}}|_{\cC} - Q^*|_{\cC}}^2} \leq \varepsilon^2. 
	\]
	\item For Theorem~\ref{thm:async_q_learning_optimal_rate_implicit} with $\eta = \mathrm{imp}$, and $C_2 = C^{\mathrm{imp}} = \max\set{3\cdot118\cdot 600 (1-r^{\mathrm{imp}}), \frac{189C}{\lambda^*}}$
	\[
	\E\sqbk{\spannorm{Q^{\mathrm{corr}}_{T}|_{\cC} - Q^*|_{\cC}}^2} \leq \frac{C_{2}2^{6K}|\cS|^3|\cA|^3\log(|\cS||\cA|)}{p_\wedge^2(1-r^{\mathrm{imp}})T}.
	\]
	Therefore, there exists a constnat $C_2$ depends on the mixing time $\frac{1}{1-r^{\mathrm{imp}}}$, and $T_C^{\mathrm{imp}''}$ that are independent on $\varepsilon$ (only depends on $p_\wedge$, $K$, $r^{\mathrm{imp}}$), when
	\[
	T\geq \max\set{T_C^{\mathrm{imp}''}, \frac{C_2\,2^{6K}|\cS|^3|\cA|^3\log(|\cS||\cA|)\log^2 T}{p_\wedge^2\varepsilon^2}} \implies \E\sqbk{\spannorm{Q_{T}^{\mathrm{corr}}|_{\cC} - Q^*|_{\cC}}^2} \leq \varepsilon^2.
	\]
\end{itemize}

Finally, since all bounds hold on the recurrent class $\cC$, Jensen’s inequality and Lemma~\ref{lem:approximate_on_union_recurrent_class} imply
\[
\E\sqbk{\spannorm{Q_T|_{\cC}-Q^*|_{\cC}}^2}\le \varepsilon^2
\;\Longrightarrow\;
\E\sqbk{\spannorm{Q_T|_{\cC}-Q^*|_{\cC}}}\le \varepsilon
\;\Longrightarrow\;
\E\sqbk{g^*-g^{\pi_T}}\le \varepsilon.
\]
This completes the proof for Theorem~\ref{thm:async_q_learning_optimal_rate_explicit} and Theorem~\ref{thm:async_q_learning_optimal_rate_implicit} in both the explicit and implicit exploration cases.

	\section{Proof of Auxiliary Lemmas}
\label{sec:app:technical_proof}
In this section, we provide detailed proofs of the lemmas from the previous sections.

\subsection{Proof of Lemma~\ref{lem:lazy_doubles_hitting_time}}
\label{sec:app:proof_of_lem:lazy_doubles_hitting_time}
Fix $s\in\cS$. Construct a coupling as follows.
Let $(y_n)_{n\ge 0}$ be a Markov chain with kernel $\Pp_{\behpi}$ and initial state $y_0=s$.
Let $(G_n)_{n\ge 1}$ be i.i.d.\ geometric random variables with parameter $1/2$ supported on $\{1,2,\dots\}$, independent of $(y_n)$, i.e.
\[
\mathbb P\sqbk{G_n=k}=2^{-k},\qquad k=1,2,\dots,
\qquad\text{so that}\qquad
\mathbb E[G_n]=2.
\]
Define the lazy time change
\[
\Sigma_n := \sum_{i=1}^n G_i,\qquad n\ge 1,\qquad \Sigma_0:=0,
\]
and define a process $(s_t)_{t\ge 0}$ by
\[
s_t := y_n \quad\text{whenever}\quad \Sigma_n \le t < \Sigma_{n+1}.
\]
Then $(s_t)$ is a Markov chain with kernel $\overline{\Pp}_{\behpi}=\tfrac12 I+\tfrac12 \Pp_{\behpi}$.
Indeed, between successive jump times $\Sigma_n$, the chain stays put, and at each jump it evolves according to $\Pp_{\behpi}$,
with geometric$(1/2)$ waiting times corresponding exactly to the $1/2$ holding probability.

Let
\[
\tau := \inf\{n\ge 0:\, y_n=s^\dagger\},
\qquad
\overline\tau := \inf\{t\ge 0:\, s_t=s^\dagger\}.
\]
By construction, the lazy chain hits $s^\dagger$ exactly at the (random) jump time $\Sigma_\tau$, hence $\overline\tau=\Sigma_\tau$.
If $\tau=0$, then $\overline\tau=0$ and the identity is trivial.
Otherwise, conditioning on $\tau$ and using independence of $(G_n)$ from $\tau$ gives
\begin{align*}
	\overline{\mathbb E}_{\behpi}\sqbkcond{\inf\set{t>0 \mid s_t = s^\dagger}}{s_0=s}
	&= \mathbb E[\overline\tau]
	= \mathbb E\Big[\sum_{i=1}^{\tau} G_i\Big]
	= \mathbb E\Big[\sum_{i=1}^{\tau} \mathbb E[G_i]\Big]
	= 2\,\mathbb E[\tau] \\
	&= 2\,\mathbb E_{\behpi}\sqbkcond{\inf\set{t>0 \mid y_t = s^\dagger}}{y_0=s}.
\end{align*}
Taking the supremum over $s\in\cS$ yields the stated uniform bound with $\overline K=2K$.

\subsection{Proof of Lemma~\ref{lem:f_operator_properties}}
\label{sec:app:proof_of_lem:f_operator_properties}
Fix $z=(\tilde D,s,a,s')$ and write $d:=\tilde D(s,a)$.
\begin{enumerate}[leftmargin=*, label=(\arabic*), align=left]
    \item Let $Q\in \R^{|\cS||\cA|}$ defined as:
    \[
    \Delta := F^{\eta}(Q_1, z) - F^{\eta}(Q_2, z) - (Q_1 - Q_2)
    \]
    Then:
    \[
    \Delta(s,a) = \frac{1}{d}\bracket{\max_{a'\in \cA}Q_1 (s', a') - \max_{a'\in \cA}Q_2(s', a') - (Q_1 (s, a) - Q_2(s, a))}
    \]
    Since $\Delta$ has only one non-zero entry on $(s,a)$, we have:
    \[
    \spannorm{\Delta} = |\Delta(s, a)|
    \]
    And:
    \begin{align*}
        &\max_{a'\in \cA}Q_1 (s', a') - \max_{a'\in \cA}Q_2(s', a') - (Q_1 (s, a) - Q_2(s, a))\\
        \leq& \max_{s,a}\bracket{Q_1 (s,a) - Q_2(s,a)} - \min_{s,a}\bracket{Q_1 (s,a) - Q_2(s,a)}\\
        \leq& \spannorm{Q_1 - Q_2}\\
        &\max_{a'\in \cA}Q_1 (s', a') - \max_{a'\in \cA}Q_2(s', a') - (Q_1 (s, a) - Q_2(s, a))\\
        \geq& \min_{s,a}\bracket{Q_1 (s,a) - Q_2(s,a)} - \max_{s,a}\bracket{Q_1 (s,a) - Q_2(s,a)}\\
        \geq& -\spannorm{Q_1 - Q_2}
    \end{align*}
    Hence:
    \[
    \spannorm{\Delta} \leq \frac{1}{d}\spannorm{Q_1 - Q_2}
    \]
    Then we conclude:
    \begin{align*}
        \spannorm{F(Q_1, z) - F(Q_2, z)} \leq& \spannorm{Q_1 - Q_2} + \spannorm{\Delta}\\
        \leq& \bracket{1 + \frac{1}{d}}\spannorm{Q_1 - Q_2}\\
        \leq& \frac{2}{\widetilde D(s, a)}\spannorm{Q_1 - Q_2}
    \end{align*}
    \item Let $Q\in \R^{|\cS||\cA|}$, we have:
    \begin{align*}
        \spannorm{F(Q, z)}\leq& \spannorm{F(Q, z)} - \spannorm{F(\mathbf{0}, z)} + \spannorm{F(\mathbf{0}, z)}\\
        \leq& \frac{2}{d}\spannorm{Q} + \spannorm{F(\mathbf{0}, z)} \\
        \leq& \frac{2}{d}\spannorm{Q} + \frac{1}{d}\\
        \leq& \frac{2}{d}\bracket{\spannorm{Q} + 1}
    \end{align*}
    \item Under $z\sim\mu_Z$, we have $(s,a)\sim D$ and $s'\sim \overline p(\cdot| s,a)$. Thus, for any fixed $(s_0,a_0)$:
    \begin{align*}
        &\E\sqbk{F(Q, z)(s_0,a_0)}\\
        =& \E\sqbk{\frac{\mathbb I_{\set{(s, a) = (s_0,a_0)}}}{D(s_0,a_0)}\bracket{r(s_0,a_0) + \max_{a'\in \cA} Q(s', a') - Q(s_0,a_0)} + Q(s_0,a_0)}\\
        =& \sum_{s, a, s'} \frac{\mu(s)\behpi(a|s)\overline{p}(s'|s, a)\mathbb I_{\set{(s, a) = (s_0,a_0)}}}{D(s_0,a_0)}\bracket{r(s,a) + \max_{a'\in \cA} Q(s', a') - Q(s,a)}+ Q(s_0,a_0)\\
        =& \sum_{s'} \overline{p}(s'|s_0,a_0)\bracket{r(s_0,a_0) + \max_{a'\in \cA} Q(s', a')} + Q(s_0,a_0) - Q(s_0,a_0)\\
        =& \cT_{\overline{P}}(Q)(s_0,a_0)
    \end{align*}
\end{enumerate}
Proved.

\subsection{Proof of Lemma~\ref{lem:async_q_function_growth_rate}}
\label{sec:app:proof_of_lem:async_q_function_growth_rate}
When $(s,a)\neq (s_{t}, a_{t})$, then:
\[
Q_{t+1}(s,a) = \max_{s,a} Q_t(s,a) \leq \max_{s,a} Q_{t}(s,a) + \lambda_t(s_t, a_t)
\]
When $(s,a) = (s_{t}, a_{t})$, we have:
\begin{align*}
    Q_{t+1}(s,a) =& Q_{t}(s,a) + \lambda_t(s_t, a_t)\bracket{r(s_{t},a_{t}) + \max_{a'\in \cA} Q_{t}(s_{t+1}, a') - Q_{t}(s_{t}, a_{t})}\\
    \leq& \max_{s,a} Q_{t}(s,a) + \lambda_t(s_t, a_t)
\end{align*}
And it holds trivally that:
\[
\min_{s,a} Q_{t}(s,a) \leq \min_{s,a} Q_{t+1}(s,a)
\]
Then, we have:
\[
\spannorm{Q_{t+1}} \leq \max_{s,a} Q_{t}(s,a) - \min_{s,a} Q_{t}(s,a) + \lambda_t(s_t, a_t) = \spannorm{Q_{t}} + \lambda_t(s_t, a_t) 
\]
Proved.

\subsection{Proof of Lemma~\ref{lem:async:lazy_radius}}
\label{sec:app:proof_of_lem:async:lazy_radius}
Let $\overline{K}=2K$ as defied in~\ref{lem:lazy_doubles_hitting_time}. Consider $m:=\lceil \overline K\rceil+1$ and $\kappa:=2^{-m}/m$.
Fix an arbitrary initial state $s\in\cS$.
By Markov's inequality,
\[
\overline{\mathbb P}^{\behpi}\sqbkcond{T_{s}>m}{s_0=s}
\le \frac{\overline{\mathbb E}_{\behpi}\sqbkcond{T_{s}}{s_0=s}}{m}
\le \frac{\overline K}{m},
\]
so
\[
\overline{\mathbb P}^{\behpi}\sqbkcond{T_{s}\le m}{s_0=s}
\ge 1-\frac{\overline K}{m}
\ge 1-\frac{m-1}{m}
=\frac{1}{m}.
\]

Since $\overline{\Pp}_{\behpi}$ is lazy, we have $\overline{\Pp}_{\behpi}(s,s)\ge \tfrac12$.
Hence, conditionally on the event $\{T_{s}\le m\}$, once the chain hits $s^\dagger$, the probability that it stays at $s^{\dagger}$ for all the remaining $m-T_{s}$ steps is at least
\[
(\tfrac12)^{\,m-T_{s}}\ge 2^{-m}.
\]
Therefore,
\[
\overline{\Pp}_{\behpi}^{\,m}(s,s^\dagger)
=
\overline{\mathbb P}^{\behpi}\sqbkcond{s_m=s^\dagger}{s_0=s}
\ge
\overline{\mathbb P}^{\behpi}\sqbkcond{T_{s}\le m}{s_0=s}\,2^{-m}
\ge \frac{2^{-m}}{m}
= \kappa.
\]
Equivalently, for every measurable $A\subseteq\cS$,
\[
\overline{\Pp}_{\behpi}^{\,m}(s,A)\ge \kappa\,\delta_{s^\dagger}(A),
\qquad \forall s\in\cS.
\]

This minorization implies that for any $s',s''\in\cS$,
\[
\bigl\|\overline{\Pp}_{\behpi}^{\,m}(s',\cdot)-\overline{\Pp}_{\behpi}^{\,m}(s'',\cdot)\bigr\|_{\mathrm{TV}}
\le 1-\kappa.
\]
Consequently, for any probability measures $\mu,\nu$ on $\cS$ and any integer $q\ge 1$,
\[
\|\mu\, \overline{\Pp}_{\behpi}^{\,qm}-\nu\, \overline{\Pp}_{\behpi}^{\,qm}\|_{\mathrm{TV}}
\le (1-\kappa)^q \|\mu-\nu\|_{\mathrm{TV}}.
\]
Taking $\mu=\delta_{s'}$, $\nu=\rho$, and using $\rho \overline{\Pp}_{\behpi}^{\,qm}=\rho$, we obtain
\[
\bigl\|\overline{\Pp}_{\behpi}^{\,qm}(s',\cdot)-\rho\bigr\|_{\mathrm{TV}}
\le (1-\kappa)^q.
\]

For a general $t\ge 0$, write $t=qm+\ell$ with $q=\lfloor t/m\rfloor$ and $0\le \ell<m$.
Since applying a Markov kernel is a contraction in total variation,
\[
\begin{aligned}
\bigl\|\overline{\Pp}_{\behpi}^{\,t}(s',\cdot)-\rho\bigr\|_{\mathrm{TV}}
&=
\bigl\|\overline{\Pp}_{\behpi}^{\,qm}(s',\cdot)\,\overline{\Pp}_{\behpi}^{\,\ell}
-\rho\,\overline{\Pp}_{\behpi}^{\,\ell}\bigr\|_{\mathrm{TV}} \\
&\le
\bigl\|\overline{\Pp}_{\behpi}^{\,qm}(s',\cdot)-\rho\bigr\|_{\mathrm{TV}}
\le (1-\kappa)^q.
\end{aligned}
\]
Define
\[
\overline{r} :=(1-\kappa)^{1/m},
\qquad
C:=(1-\kappa)^{-1}.
\]
Then
\[
(1-\kappa)^q
=
(1-\kappa)^{t/m-\ell/m}
=
(\overline{r})^{\,t}(1-\kappa)^{-\ell/m}
\le C\,(\overline{r})^{\,t},
\]
which yields the geometric mixing bound.

Finally, since $x\mapsto x^{1/m}$ is concave on $[0,1]$, we have
\[
(1-\kappa)^{1/m}\le 1-\frac{\kappa}{m}.
\]
Therefore $1-\overline{r}\ge \kappa/m$, and hence
\[
\frac{1}{1- \overline{r}}\le \frac{m}{\kappa}=m^2 2^m.
\]
Using $r^{\mathrm{exp}} \leq \overline{r}$, $m=\lceil \overline K\rceil+1\le \overline K+2$ and $2^m\le 2^{\overline K+2}$ yields
\[
\frac{1}{1-r^{\mathrm{exp}}} \leq \frac{1}{1- \overline{r}}\le (\overline K+2)^2\,2^{\overline K+2}.
\]
Since by Lemma~\ref{lem:lazy_doubles_hitting_time}, $\overline K=2K$, the final displayed bound follows immediately as $\frac{1}{1-r^{\mathrm{exp}}}\leq (2K+2)^2 2^{2K +2}$.

\subsection{Proof of Lemma \ref{lem:sp_env_seminorm_with_min_norm_version}}
\label{sec:app:proof_of_lem:sp_env_seminorm_with_min_norm_version}
First, we show $\norm[\widetilde{sp}]{\cdot}$ defines a norm. Homogeneity and non-negativity are immediate from the definition. Since both $\spanstar{\cdot}$ and $|\cdot|$a re seminorms, the triangle inequality holds for $\|\cdot\|_{\widetilde{sp}}$: for any $Q_1, Q_2\in \R^{|\cS||\cA|}$:
	\[
	\norm[\widetilde{sp}]{Q_1 + Q_2} \leq \norm[\widetilde{sp}]{Q_1} + \norm[\widetilde{sp}]{Q_2}.
	\]

	Moreover, $\norm[\widetilde{sp}]{\mathbf 0}=0$. Conversely, if
$\norm[\widetilde{sp}]{Q}=0$, then $\spanstar{Q}=0$ and
$|\mathbf 1^\top Q|=0$. The first condition implies
$Q\in\{c\mathbf 1:c\in\R\}$, while the second implies
$\sum_{s,a}Q(s,a)=0$. Hence $Q=\mathbf{0}$. Therefore,
$\norm[\widetilde{sp}]{\cdot}$ is a norm.

Finally, for any $Q\in \R^{|\cS||\cA|}$,
	\begin{align*}
		\min_{c\in \R} \norm[\widetilde{sp}]{Q- c\mathbf 1} =& \min_{c\in \R} \spanstar{Q- c\mathbf 1} + \frac{1}{|\cS||\cA|} |\mathbf 1^\top (Q - c\mathbf 1)| \\
		=& \min_{c\in \R} \spanstar{Q} + \frac{1}{|\cS||\cA|}|\mathbf 1^\top Q - c|\cS||\cA||\\
		=& \spanstar{Q},
	\end{align*}
which completes the proof.

\subsection{Proof of Lemma \ref{lem:sp_env_equivlent_to_l_q}}
\label{sec:app:proof_of_lem:sp_env_equivlent_to_l_q}
We first establish a lowe bound. Note that
\begin{equation}
	\label{equ:app:Q_sp_env_lowerbound_max}
\norm[\widetilde{sp}]{Q} = \spanstar{Q} + \frac{1}{|\cS||\cA|}|\mathbf 1^\top Q| \geq \spannorm{Q} + \min_{(s,a)} Q(s,a) = \max_{(s,a)} Q(s,a).
\end{equation}
Similarly,
\begin{align}
	\norm[\widetilde{sp}]{Q} =& \spanstar{Q} + \frac{1}{|\cS||\cA|}|\mathbf 1^\top Q|\notag\\
	\geq& \spannorm{Q} + \frac{1}{|\cS||\cA|}|\mathbf 1^\top Q|\notag\\
	\geq& (\max_{(s,a)} Q(s,a) - \min_{(s,a)} Q(s,a)) - \max_{(s,a)} Q(s,a) \notag\\
	\geq& -\min_{(s,a)} Q(s,a). \label{equ:app:Q_sp_env_lowerbound_min}
\end{align}
Combine~\eqref{equ:app:Q_sp_env_lowerbound_max} and \eqref{equ:app:Q_sp_env_lowerbound_min}, we obtain 
\begin{equation}
	\label{equ:app:async_sp_env_large_l_infty}
\norm[\widetilde{sp}]{Q} \geq \max\set{\max_{(s,a)} Q(s,a), -\min_{(s,a)} Q(s,a)} = \linftynorm{Q}.
\end{equation}

For the upper bound, by Lemma \ref{lem:new_seminorm_equivalent_to_span}, $\spanstar{Q}\leq 2\spannorm{Q}$ and $\spannorm{Q}\leq 2 \linftynorm{Q}$ which implies
\[
\spanstar{Q} \leq 2\spannorm{Q} \leq 4\linftynorm{Q}.
\]
Together with the bound $\frac{1}{|\cS||\cA|}|\mathbf 1^\top Q| \leq \linftynorm{Q}$, we have
\[
\norm[\widetilde{sp}]{Q} \leq \spanstar{Q} + \frac{1}{|\cS||\cA|}|\mathbf{1}^\top Q|\leq  4\linftynorm{Q} +\linftynorm{Q} \leq 5\linftynorm{Q}.
\]

Finally, recalling the equivalence between $\linftynorm{\cdot}$ and $\norm[q]{\cdot}$,
\[
\bracket{|\cS||\cA|}^{-1/q}\|Q\|_q \leq \linftynorm{\cdot} \leq \norm[q]{Q},
\]
we conclude that
\[
(|\cS||\cA|)^{-1/q}\norm[q]{Q}\leq \norm[\widetilde{sp}]{Q} \leq 5 \norm[q]{Q}.
\]
This completes the proof.

\subsection{Proof of Lemma~\ref{lem:lyapunov_property}}
\label{sec:app:proof_of_lem:lyapunov_property}
Denote $E := \{c\mathbf 1 : c\in\R\}$. We verify each item in
Lemma~\ref{lem:lyapunov_property} in turn.
\begin{enumerate}[leftmargin=*, label=(\arabic*), align=left]
\item  Recall Definition~\ref{def:infimal_convolution}. By definition,
    \begin{align*}
        M(Q) =& \inf_{\mu\in \R^{|\cS||\cA|}} \set{\frac{1}{2}\spanstar{\mu}^2 + \frac{1}{2\theta}\norm[q]{Q - \mu} ^2}\\
        =& \sqbk{\bracket{\frac{1}{2}\norm[\widetilde{sp}]{\cdot}^2 \square \delta_{E}} \square \frac{1}{2\theta}\norm[q]{\cdot}^2}(Q)\\
    =& \sqbk{\frac{1}{2}\norm[\widetilde{sp}]{\cdot}^2\square\bracket{\frac{1}{2\theta}\norm[q]{\cdot}^2\square \delta_{E}}}(Q)
    \end{align*}
    Since $\tfrac{\norm[q]{\cdot}^2}{2\theta}$ is $\frac{q-1}{\theta}$-smooth with respect to $\norm[q]{\cdot}$, it follows that for any $Q_1, Q_2 \in \R^{|\cS||\cA|}$:
\begin{equation}
    \label{equ:lyapunov_smoothness}
M(Q_2) \leq M(Q_1) + \innerprod{\nabla M(Q_1)}{Q_2 - Q_1} + \frac{q-1}{2\theta }\norm[q]{Q_2 - Q_1}^2.
\end{equation}

By convexity of $M(\cdot)$, we have when $z\in E$:
\begin{align*}
    \innerprod{\nabla M(Q_1)}{z} \leq& M(Q_1 + z) - Q(Q_1) = 0\\
    \innerprod{\nabla M(Q_1)}{-z} =& \innerprod{\nabla M(Q_1 + z)}{-z} \leq M(Q_1) - M(Q_1 + z) = 0
\end{align*}
Then, let $z^* = \argmin{z\in E} \norm[\widetilde{sp}]{Q_2 - Q_1 - z}$, then:
\begin{align*}
    M(Q_2) =& M(Q_2 - z^*)\\
    \leq& M(Q_1) + \innerprod{\nabla M(Q_1)}{(Q_2 - Q_1) - z^*} + \frac{q-1}{2\theta}\|(Q_2 - Q_1) - z^*\|_q^2\\
    \overset{(i)}{\leq}& M(Q_1) + \innerprod{\nabla M(Q_1)}{(Q_2 - Q_1) - z^*} + \frac{q-1}{2\theta l_q^2}\norm[\widetilde{sp}]{Q_2 - Q_1 - z^*}^2\\
    =& M(Q_1) + \innerprod{\nabla M(Q_1)}{Q_2 - Q_1} + \frac{q-1}{2\theta l_q^2}\spanstar{Q_2 - Q_1}^2,
\end{align*}
where $(i)$ holds due to Lemma~\ref{lem:sp_env_equivlent_to_l_q}, $l_q\|\cdot\|_q\leq \norm[\widetilde{sp}]{\cdot}\leq u_q\|\cdot\|_q$.
\item Since $M(\cdot)$ is equivalent to
        \[
        M(Q) = \sqbk{\bracket{\frac{1}{2}\norm[\widetilde{sp}]{\cdot}^2\square\frac{1}{2\theta}\norm[q]{\cdot}^2}\square \delta_{E}}(Q).
        \]
        Then, define $\norm[\mathrm{M}]{\cdot}: \R^{|\cS||\cA|}\to \R$ as
        \[
        \norm[\mathrm{M}]{Q} = \bracket{\frac{1}{2}\norm[\widetilde{sp}]{\cdot}^2\square\frac{1}{2\theta}\norm[q]{\cdot}^2}(Q)
        \]
        It was shown in \citet{chen2024lyapunov} that
        \[
        l_*^2 M(Q) \leq \frac{1}{2} \spanstar{Q}^2 \leq u_*^2 M(Q), \quad \forall Q\in \R^{|\cS||\cA|}.
        \]
\item This property followed directly by the convexity of $M(\cdot)$:
\begin{align*}
    \innerprod{\nabla M(Q)}{c\mathbf 1} \leq& M(Q + c\mathbf 1) - M(Q) = 0\\
    \innerprod{\nabla M(Q)}{-c\mathbf 1} =& \innerprod{\nabla M(Q + c\mathbf 1)}{-c\mathbf 1} \leq M(Q) - M(Q + c\mathbf 1) = 0.
\end{align*}
Then we have $\innerprod{\nabla M(Q)}{z} = 0$ for all $z\in E$.
\item By Equation~\eqref{equ:lyapunov_smoothness}, consider $Q_1, Q_2\in \R^{|\cS||\cA|}$. Since $M(\cdot)$ is $\frac{q-1}{\theta}$-smooth w.r.t. $\norm[q]{\cdot}$, and $\nabla M(\cdot)$ is $\frac{q-1}{\theta}$-Lipschitz continuous w.r.t. $\norm[q]{\cdot}$. Then by cococertity, we have, for all $z\in E$
\begin{align*}
    \frac{\theta}{q-1} \norm[q, *]{\nabla M(Q_1) - \nabla M(Q_2)}^2 \leq& \innerprod{\nabla M(Q_1) - \nabla M(Q_2)}{Q_1 - Q_2}\\
    =& \innerprod{\nabla M(Q_1) - \nabla M(Q_2)}{Q_1 - Q_2 - z}\\
    \leq& \norm[q,*]{\nabla M(Q_1) - \nabla M(Q_2)} \norm[q]{Q_1 - Q_2 - z}\\
    \leq& \frac{1}{l_q}\norm[q,*]{\nabla M(Q_1) - \nabla M(Q_2)} \spanstar{Q_1 - Q_2}.
\end{align*}
Rearranging yields
\[
\norm[q,*]{\nabla M(Q_1)-\nabla M(Q_2)}
\le \frac{q-1}{\theta l_q}\spanstar{Q_1-Q_2}.
\]
Now, for any $Q_3\in \R^{|\cS||\cA|}$, let $z^* = \argmin{z \in E} \norm[\widetilde{sp}]{Q_3 - z}$
\begin{align*}
    \innerprod{\nabla M(Q_1) - \nabla M(Q_2)}{Q_3} =& \innerprod{\nabla M(Q_1) - \nabla M(Q_2)}{Q_3 - z^*}\\
    \leq& \norm[q,*]{\nabla M(Q_1) - \nabla M(Q_2)} \norm[q]{Q_3 - z}\\
    \leq& \frac{q-1}{\theta l_q^2}\norm[q,*]{\nabla M(Q_1) - \nabla M(Q_2)} \norm[\widetilde{sp}]{Q_3 - z^*}\\
    \leq& \frac{q-1}{\theta l_q^2} \spanstar{Q_1 - Q_2} \spanstar{Q_3},
\end{align*}
which proves the claim.
\end{enumerate}

\subsection{Proof of Lemma~\ref{lem:async_T_1_bound}}
\label{sec:app:proof_of_lem:async_T_1_bound}
We first prove the following auxiliary lemma.
\begin{lemma}
	\label{lem:seminorm_gradient_bound_in_dual}
	For any $Q\in \R^{|\cS||\cA|}$, the gradient of $\seminorm[M]{Q}$ satisfies
	\[
	\norm[M,*]{\nabla \seminorm[M]{Q}} \leq 1
	\]
	almost everywhere, where $\norm[M]{\cdot}$ is the norm defined in item~(2) of Lemma~\ref{lem:lyapunov_property}, and $\norm[M,*]{\cdot}$ is its
	dual norm.
\end{lemma}
\begin{proof}
	Since $\seminorm[M]{\cdot}$ is a seminorm, for any
	$Q_1,Q_2\in \R^{|\cS||\cA|}$ we have
	\[
	|\seminorm[M]{Q_1}-\seminorm[M]{Q_2}|
	\le \seminorm[M]{Q_1-Q_2}
	= \min_{c\in\R}\norm[M]{Q_1-Q_2-c\mathbf 1}
	\le \norm[M]{Q_1-Q_2}.
	\]
	Thus, $\seminorm[M]{\cdot}$ is $1$-Lipschitz with respect to $\norm[M]{\cdot}$.
	By Rademacher’s theorem, it is differentiable almost everywhere and its
	gradient satisfies
	\[
	\norm[M,*]{\nabla \seminorm[M]{Q}} \le 1.
	\]
\end{proof}

Recall that
\begin{align*}
T_1 =& \E \sqbk{\innerprod{\nabla M(Q_{t-1} - \overline {Q}^*)}{\cT_{\overline{P}}(Q_{t-1}) - Q_{t-1}}}\\
=& \E\big[\underbrace{\innerprod{\nabla M(Q_{t-1} - \overline {Q}^*)}{\cT_{\overline{P}}(Q_{t-1}) - \cT(\overline {Q}^*)}}_{T_{1,1}}\big] + \E\big[\underbrace{\innerprod{\nabla M(Q_{t-1} - \overline {Q}^*)}{\overline {Q}^* + g^*\mathbf 1 - Q_{t-1}}}_{T_{1,2}}\big].
\end{align*}

We first bound $T_{1,1}$. Since $M(Q)=\frac12\seminorm[M]{Q}^2$, its gradient is
\[
\nabla M(Q)=\seminorm[M]{Q}\,\nabla\seminorm[M]{Q}.
\]
For any constant $c\in\R$, inserting and subtracting $c\mathbf 1$ yields
\begin{align*}
	T_{1,1} =& \innerprod{\nabla M(Q_{t-1} - \overline{Q}^*)}{\cT_{\overline{P}}(Q_{t-1}) - \cT_{\overline{P}}(\overline{Q}^*) + c\mathbf 1}\\
	\leq& \seminorm[M]{Q_{t-1} - \overline{Q}^*} \innerprod{\nabla \seminorm[M]{Q_{t-1} - \overline{Q}^*}}{\cT_{\overline{P}}(Q_{t-1}) - \cT_{\overline{P}}(\overline{Q}^*) + c\mathbf 1}\\
	\leq& \seminorm[M]{Q_{t-1} - \overline{Q}^*} \norm[M,*]{\nabla \seminorm[M]{Q_{t-1} - \overline{Q}^*}} \norm[M]{\cT_{\overline{P}}(Q_{t-1}) - \cT_{\overline{P}}(\overline{Q}^*) -c\mathbf 1}.
\end{align*}
By Lemma~\ref{lem:seminorm_gradient_bound_in_dual},
$\norm[M,*]{\nabla\seminorm[M]{Q_{t-1}-\overline Q^*}}\le 1$.
Choosing
\[
c^* \in \arg\min_{c\in\R}
\norm[M]{\cT_{\overline P}(Q_{t-1})-\cT_{\overline P}(\overline Q^*)-c\mathbf 1},
\]
and using item~(2) of Lemma~\ref{lem:lyapunov_property}, we obtain
\begin{align*}
	T_{1,1} \leq& \seminorm[M]{Q_{t-1} - \overline{Q}^*} \norm[M, *]{\nabla \seminorm[M]{Q_{t-1} - \overline{Q}^*}}\norm[M]{\cT_{\overline{P}}(Q_{t-1}) - \cT_{\overline{P}}(\overline{Q}^*) -c^*\mathbf 1} \\
	\overset{(i)}{\leq}& \seminorm[M]{Q_{t-1} - \overline{Q}^*} \seminorm[M]{\cT_{\overline{P}}(Q_{t-1}) - \cT_{\overline{P}}(\overline{Q}^*)}\\
	\overset{(ii)}{\leq}& \frac{1}{l_*} \seminorm[M]{Q_{t-1} - \overline{Q}^*} \spanstar{\cT_{\overline{P}}(Q_{t-1}) - \cT_{\overline{P}}(\overline{Q}^*)}\\
	\overset{(iii)}{\leq}& \frac{\beta}{l_*}\seminorm[M]{Q_{t-1} - \overline{Q}^*} \spanstar{Q_{t-1} - \overline{Q}^*}\\
	\leq& \frac{\beta u_*}{l_*} \seminorm[M]{Q_{t-1} - \overline{Q}^*}^ 2\\
	=& \frac{2\beta u_*}{l_*} M(Q_{t-1} - \overline{Q}^*)
\end{align*}
Here $(i), (iii)$ are derived by norm equivalence property (2) in Lemma~\ref{lem:lyapunov_property}
\[
l_* \seminorm[M]{Q} \leq \spanstar{Q} \leq u_*\seminorm[M]{Q},
\]
and $(ii)$ is derived by the contraction property of the Bellman operator $\cT_{\overline{P}}(\cdot)$ under the seminorm $\spanstar{\cdot}$.

Next, we bound $T_{1,2}$. By convexity of $\seminorm[M]{\cdot}$, for all $\mathbf x, \mathbf y \in \R^{|\cS||\cA|}$, and $\alpha\in [0,1]$:
\[
\seminorm[M]{\alpha \mathbf x + (1-\alpha) \mathbf y} \leq \alpha \seminorm[M]{\mathbf x}+ (1-\alpha) \seminorm[M]{\mathbf y},
\]
let $\mathbf x \leftarrow \mathbf 0$ and $\mathbf y \leftarrow Q_{t-1} - \overline {Q}^*$
\[
\innerprod{\nabla\seminorm[M]{Q_{t-1} - \overline {Q}^*}}{\overline {Q}^* - Q_{t-1}}\leq \seminorm[M]{\mathbf 0} - \seminorm[M]{Q_{t-1} - \overline {Q}^*} = - \seminorm[M]{Q_{t-1} - \overline {Q}^*}.
\]
We obtain
\begin{align*}
	T_{1,2} =& \innerprod{\nabla M(Q_{t-1} - \overline {Q}^*)}{\overline {Q}^* + g^*\mathbf 1 - Q_{t-1}}\\
	\overset{(i)}{=}& \innerprod{\nabla M(Q_{t-1} - \overline {Q}^*)}{\overline {Q}^* - Q_{t-1}}\\
	=& \seminorm[M]{Q_{t-1} - \overline {Q}^*}\innerprod{\nabla \seminorm[M]{Q_{t-1} - \overline {Q}^*}}{\overline {Q}^* - Q_{t-1}}\\
	\leq& -\seminorm[M]{Q_{t-1} - \overline {Q}^*}^2\\
	=& -2 M(Q_{t-1} - \overline {Q}^*).
\end{align*}
where $(i)$ is by $(3)$ in Lemma \ref{lem:lyapunov_property}. 

Combining the bounds on $T_{1,1}$ and $T_{1,2}$ yields
\begin{align*}
	T_1 = \E\sqbk{\innerprod{\nabla M(Q_{t-1} - \overline {Q}^*)}{\cT_{\overline{P}}(Q_{t-1}) - Q_{t-1}}}\leq& \E\sqbk{T_{1,1}} + \E\sqbk{T_{1,2}}\\
	\leq& \frac{2\beta u_*}{l_*} \E\sqbk{M(Q_{t-1} - \overline {Q}^*)} - 2\E\sqbk{M(Q_{t-1} - \overline {Q}^*)} \\
	\leq& -2\bracket{1- \frac{\beta u_*}{l_*}} \E\sqbk{M (Q_{t-1} - \overline {Q}^*)}.
\end{align*}
This completes the proof.

\subsection{Proof of Lemma~\ref{lem:async_T_2_bound}}
\label{sec:app:proof_of_lem:async_T_2_bound}
$T_2^{\eta}$ is can be decomposed as:
\begin{align*}
	T_2^\eta
	=&\;
	\E\sqbk{\innerprod{\nabla M(Q_{t-1}-\overline Q^*)}
	{F^\eta(Q_{t-1},D,y_{t-1})-\cT_{\overline P}(Q_{t-1})}}\\
	=&\;
	\underbrace{
	\E\sqbk{\innerprod{\nabla M(Q_{t-1-\tau_{t-1}^\eta}-\overline Q^*)}
	{F^\eta(Q_{t-1-\tau_{t-1}^\eta},D,y_{t-1})
	-\cT_{\overline P}(Q_{t-1-\tau_{t-1}^\eta})}}
	}_{T_{2,1}^{\eta}}\\
	&+
	\underbrace{
	\E\sqbk{\innerprod{\nabla M(Q_{t-1-\tau_{t-1}^\eta}-\overline Q^*)}
	{F^\eta(Q_{t-1},D,y_{t-1})
	- F^\eta(Q_{t-1-\tau_{t-1}^\eta},D,y_{t-1})}}
	}_{T_{2,2}^{\eta}}\\
	&+
	\underbrace{
	\E\sqbk{\innerprod{\nabla M(Q_{t-1-\tau_{t-1}^\eta}-\overline Q^*)}
	{\cT_{\overline P}(Q_{t-1-\tau_{t-1}^\eta})
	-\cT_{\overline P}(Q_{t-1})}}
	}_{\text{continuation of }T_{2,2}^{\eta}}\\
	&+
	\underbrace{
	\E\sqbk{\innerprod{
	\nabla M(Q_{t-1}-\overline Q^*)
	-\nabla M(Q_{t-1-\tau_{t-1}^\eta}-\overline Q^*)}
	{F^\eta(Q_{t-1},D,y_{t-1})
	-\cT_{\overline P}(Q_{t-1})}}
	}_{T_{2,3}^{\eta}}.
\end{align*}
Let $\cF_{t}$ be the $\sigma$-algebra generated by the set of random variables of the trajectory up to time $t$, then $Q_{t-1}$ is measurable w.r.t. $\cF_{t-1}$.

We first introduce a useful lemma in bounding $T_2$.
\begin{lemma}[Lemma~B.2 in \citet{chen2025nonasymptotic}]
	\label{lem:app:lemma_b_2_in_chen_2025}
	Let $t_1$ be a positive integer. Then for all $t\ge t_1$,
	\[
	\spannorm{Q_t-Q_{t_1}}
	\le
	f\!\bracket{\frac{\lambda^*(t-t_1)}
	{N_{t_1-1,\min}+1+h}}
	\bigl(\spannorm{Q_{t_1}}+1\bigr),
	\]
	where $f(x):=xe^x$ for all $x>0$ and
	$N_{t,\min}:=\min_{s,a}N_t(s,a)$. Moreover, there exists a constant $T_{C_2}^{\eta}$,
	depending only on $p_\wedge$, $r^{\eta}$ $\lambda^*$, such that for all $t\ge T_{C_2}^{\eta}$,
	\begin{equation}
		\label{equ:app:special_function_expectation_bound}
		\E\sqbk{
		f\!\bracket{\frac{\lambda^*\tau_t^\eta}
		{N_{t-\tau_t^\eta-1,\min}+1+h}}}
		\le
		\frac{4\tau_t^\eta|\cS||\cA|}{p_\wedge}
		\frac{\lambda^*}{t+h}.
	\end{equation}
\end{lemma}

\begin{lemma}
	\label{lem:async_T_2_1}
	The subterm $T_{2,1}$ can be bounded as:
	\[
	T_{2,1}^{\eta} \leq \frac{16 L \bracket{b_{t-1} + \spannorm{\overline {Q}^*} + 1}^2}{p_\wedge}\lambda_{t-1}.
	\]
\end{lemma}
\begin{proof}
	Let $\cF_t$ denote the $\sigma$-algebra generated by the trajectory up to time $t$. Then $Q_t$ is $\cF_t$-measurable. Recall
	\[
	y_t := (s_{t-1},a_{t-1},s_t)\in\cY:=\cS\times\cA\times\cS,
	\]
	which admits the invariant distribution
	\[
	\nu(s_0,a_0,s_1)=\rho(s_0)\behpi(a_0| s_0)\overline p(s_1\mid s_0,a_0).
	\]
	
	Recall that
	\begin{align*}
	T_{2,1}^{\eta}
	=&\;
	\E\sqbk{\innerprod{\nabla M(Q_{t-1-\tau_{t-1}^\eta}-\overline Q^*)}
	{F^\eta(Q_{t-1-\tau_{t-1}^\eta},D,y_{t-1})
	-\cT_{\overline P}(Q_{t-1-\tau_{t-1}^\eta})}}\\
	=&\;
	\E\sqbk{\innerprod{\nabla M(Q_{t-1-\tau_{t-1}^\eta}-\overline Q^*)}
	{\E\sqbkcond{F^\eta(Q_{t-1-\tau_{t-1}^\eta},D,y_{t-1})}
	{\cF_{t-1-\tau_{t-1}^\eta-1}}
	-\cT_{\overline P}(Q_{t-1-\tau_{t-1}^\eta})}}.
	\end{align*}

	By item~(4) of Lemma~\ref{lem:lyapunov_property},
	\[
	\innerprod{\nabla M(Q_1-Q_2)}{Q_3}
	\le L\,\spanstar{Q_1-Q_2}\spanstar{Q_3},
	\]
	we obtain
	\begin{align*}
	T_{2,1}^{\eta}
	\le
	L\,\E\sqbk{
	\spanstar{Q_{t-1-\tau_{t-1}^\eta}-\overline Q^*}
	\spanstar{
	\E\sqbkcond{F^\eta(Q_{t-1-\tau_{t-1}^\eta},D,y_{t-1})}
	{\cF_{t-1-\tau_{t-1}^\eta-1}}
	-\cT_{\overline P}(Q_{t-1-\tau_{t-1}^\eta})
	}}.
	\end{align*}

	Using the seminorm equivalence $\spanstar{\cdot}\leq 2\spannorm{\cdot}$ and Lemma~\ref{lem:async_q_span_bound}, we have
	\begin{align*}
		\spanstar{Q_{t-1- \tau^{\eta}_{t-1}}  - \overline {Q}^*} \leq& 2\spannorm{Q_{t-1 - \tau^{\eta}_{t-1}} - \overline {Q}^*}\\
		\leq& 2 b_{t-1 - \tau^{\eta}_{t-1}} + 2 \spannorm{\overline {Q}^*}\\ \leq& 2b_{t-1} + 2\spannorm{\overline {Q}^*}.
	\end{align*}
	Therefore,
	\begin{align}
	T_{2,1}^{\eta}
	\le
	4L\bracket{b_{t-1}+\spannorm{\overline Q^*}}
	\E\sqbk{
	\spannorm{
	\E\sqbkcond{F^\eta(Q_{t-1-\tau_{t-1}^\eta},D,y_{t-1})}
	{\cF_{t-1-\tau_{t-1}^\eta-1}}
	-\cT_{\overline P}(Q_{t-1-\tau_{t-1}^\eta})
	}}.
	\label{equ:T_2_1_ergodic}
	\end{align}
	
	When $\eta = \mathrm{exp}$, the dynamics are governed by $\overline P_{\behpi}$. Using the unbiasedness property of $F^{\mathrm{exp}}$,
	\[
	\E_{y\sim\nu}\bigl[F^{\mathrm{exp}}(Q,D,y)\bigr]
	=\cT_{\overline P}(Q),
	\]

	We next bound the span seminorm of the conditional bias term. Using the unbiasedness of $F^{\mathrm{exp}}$, we have
	\begin{align*}
	&\E\sqbkcond{F^{\mathrm{exp}}(Q,D,y_{t-1})}{\cF_{t-1-\tau_{t-1}^{\mathrm{exp}}-1}}
	-\cT_{\overline P}(Q) \notag\\
	=&
	\sum_{y\in\cY}
	\Bigl(
	\overline{\mathbb P}^{\behpi}\bigl(y_{t-1}=y \mid \cF_{t-1-\tau_{t-1}^{\mathrm{exp}}-1}\bigr)
	-\nu(y)
	\Bigr)\,F^{\mathrm{exp}}(Q,D,y).
	\end{align*}
	
	Taking the span seminorm on both sides and using the triangle inequality, we obtain
\begin{align*}
&\spannorm{
\E\sqbkcond{F^{\mathrm{exp}}(Q,D,y_{t-1})}{\cF_{t-1-\tau_{t-1}^{\mathrm{exp}}-1}}
-\cT_{\overline P}(Q)
} \notag\\
\le&\;
\sum_{y\in\cY}
\bigl|
\overline{\mathbb P}^{\behpi}(y_{t-1}=y | \cF_{t-1-\tau_{t-1}^{\mathrm{exp}}-1})
-\nu(y)
\bigr|
\;\spannorm{F^{\mathrm{exp}}(Q,D,y)}.
\end{align*}

We now bound the total variation term. Writing $y=(s_0,a_0,s_1)$ and recalling $\nu(s_0,a_0,s_1)=\rho(s_0)\behpi(a_0| s_0)\overline p(s_1\mid s_0,a_0)$, we have
\begin{align*}
&\sum_{y\in\cY}
\bigl|
\overline{\mathbb P}^{\behpi}(y_{t-1}=y | \cF_{t-1-\tau_{t-1}^{\mathrm{exp}}-1})
-\nu(y)
\bigr| \notag\\
=&
\sum_{s_0,a_0,s_1}
\Bigl|
\overline{\mathbb P}^{\behpi}(s_{t-1}=s_0 \mid s_{t-1-\tau_{t-1}^{\mathrm{exp}}})
-\rho(s_0)
\Bigr|
\,\behpi(a_0\mid s_0)\,\overline p(s_1\mid s_0,a_0) \notag\\
=&
\sum_{s_0}
\Bigl|
\overline{\mathbb P}^{\behpi}(s_{t-1}=s_0 \mid s_{t-1-\tau_{t-1}^{\mathrm{exp}}})
-\rho(s_0)
\Bigr| \notag\\
=&
2\,\bigl\|
\overline{\mathbb P}^{\behpi}(s_{t-1}=\cdot \mid s_{t-1-\tau_{t-1}^{\mathrm{exp}}})
-\rho(\cdot)
\bigr\|_{\mathrm{TV}} \notag\\
\le&
2C\,(r^{\mathrm{exp}})^{\tau_{t-1}^{\mathrm{exp}}}.
\end{align*}
	An analogous decomposition holds for $\eta=\mathrm{imp}$.
	\[
	\sum_{y\in\cY}
\bigl|
\overline{\mathbb P}^{\behpi}(y_{t-1}=y | \cF_{t-1-\tau_{t-1}^{\mathrm{exp}}-1})
-\nu(y)
\bigr| \leq 2 C(r^{\mathrm{imp}})^{\tau_{t-1}^{\mathrm{imp}}}
	\]

	By Lemma~\ref{lem:f_operator_properties} (2) and the fact that
$D(s,a)=\rho(s)\behpi(a|s)\ge p_\wedge$, we have
\begin{equation}
	\label{equ:app:async_using_lemma_in_zaiwei}
	\spannorm{F^\eta(Q_{t-1-\tau_{t-1}^\eta},D,y)}
	\le \frac{2(\spannorm{Q_{t-1-\tau_{t-1}^\eta}}+1)}{p_\wedge}
	\le \frac{2(b_{t-1}+1)}{p_\wedge}.
\end{equation}
Combining the above bounds yields
\begin{align*}
T_{2,1}^{\eta}
\le&
\frac{16L\bracket{b_{t-1}+\spannorm{\overline Q^*}}(b_{t-1}+1)}{p_\wedge}
\,C(r^\eta)^{\tau_{t-1}^\eta}\\
\le&
\frac{16L\bracket{b_{t-1}+\spannorm{\overline Q^*}+1}^2}{p_\wedge}
\,\lambda_{t-1},
\end{align*}
where the last inequality follows from the definition of $\tau_{t-1}^\eta$. This completes the proof.
\end{proof}

\begin{lemma}
	\label{lem:async_T_2_2}
	The subterm $T^{\eta}_{2, 2}$ can be bounded as:
	\[
	T^{\eta}_{2, 2} \leq \frac{48L \tau^{\eta}_{t-1} |\cS||\cA|(b_{t-1} + 1)^2}{p_\wedge^2}\lambda_{t-1}.
	\]
\end{lemma}
\begin{proof}
	Recall the definition of $T^{\eta}_{2,2}$:
	\begin{align*}
T^{\eta}_{2,2}
=&\;
\E\sqbk{\innerprod{\nabla M(Q_{t-1-\tau^{\eta}_{t-1}}-\overline Q^*)}
{F^{\eta}(Q_{t-1},D,y_{t-1})
- F^{\eta}(Q_{t-1-\tau^{\eta}_{t-1}},D,y_{t-1})}}\\
&+
\E\sqbk{\innerprod{\nabla M(Q_{t-1-\tau^{\eta}_{t-1}}-\overline Q^*)}
{\cT_{\overline P}(Q_{t-1-\tau^{\eta}_{t-1}})
- \cT_{\overline P}(Q_{t-1})}}\\
\overset{(i)}{\le}&\;
4L\,\E\sqbk{
\spannorm{Q_{t-1-\tau^{\eta}_{t-1}}-\overline Q^*}\,
\spannorm{F^{\eta}(Q_{t-1},D,y_{t-1})
- F^{\eta}(Q_{t-1-\tau^{\eta}_{t-1}},D,y_{t-1})}}\\
&+
4L\,\E\sqbk{
\spannorm{Q_{t-1-\tau^{\eta}_{t-1}}-\overline Q^*}\,
\spannorm{\cT_{\overline P}(Q_{t-1})
- \cT_{\overline P}(Q_{t-1-\tau^{\eta}_{t-1}})}}\\
\overset{(ii)}{\le}&\;
\Bigl(\frac{8L}{p_\wedge}+4L\Bigr)
\E\sqbk{
\spannorm{Q_{t-1-\tau^{\eta}_{t-1}}-\overline Q^*}\,
\spannorm{Q_{t-1}-Q_{t-1-\tau^{\eta}_{t-1}}}}\\
\le&\;
\frac{12L}{p_\wedge}\,
\E\sqbk{
\spannorm{Q_{t-1-\tau^{\eta}_{t-1}}-\overline Q^*}\,
\spannorm{Q_{t-1}-Q_{t-1-\tau^{\eta}_{t-1}}}},
\end{align*}
where (i) follows from item~(4) of Lemma~\ref{lem:lyapunov_property} and the
seminorm equivalence between $\spanstar{\cdot}$ and $\spannorm{\cdot}$, and
(ii) follows from Lemma~\ref{lem:f_operator_properties}. In the last inequality
we use the fact that $p_\wedge\le 1$.

By Lemma~\ref{lem:app:lemma_b_2_in_chen_2025}, we have
\begin{align*}
\spannorm{Q_{t-1}-Q_{t-1-\tau^{\eta}_{t-1}}}
&\le
f\!\bracket{\frac{\lambda\tau^{\eta}_{t-1}}
{N_{t-1-\tau^{\eta}_{t-1}-1,\min}+1+h}}
\,(b_{t-1}+1),\\
\spannorm{Q_{t-1-\tau^{\eta}_{t-1}}-\overline Q^*}
&\le
b_{t-1}+\spannorm{\overline Q^*}.
\end{align*}
Therefore,
\begin{align}
	T^{\eta}_{2,2} \leq& \frac{12 L\bracket{b_{t-1} + \spannorm{\overline {Q}^*}}(b_{t-1} + 1)}{p_\wedge}\E\sqbk{f\bracket{\frac{\lambda \tau^{\eta}_{t-1}}{N_{t-1 - \tau^{\eta}_{t-1} - 1, \min} + 1 + h}}}\\
	\leq& \frac{48\tau^{\eta}_{t-1} L |\cS||\cA|(b_{t-1} + \spannorm{\overline {Q}^*} + 1)^2}{p_\wedge^2}\lambda_{t-1}
\end{align}
where the last inequality uses the following Equation~\eqref{equ:app:special_function_expectation_bound} in Lemma \ref{lem:app:lemma_b_2_in_chen_2025}.
\end{proof}

\begin{lemma}
	\label{lem:async_T_2_3}
	The subterm $T_{2, 3}^{\eta}$ can be bounded as:
	\[
	T_{2,3}^{\eta} \leq \frac{48 L\tau^{\eta}_{t-1}  |\cS||\cA|\bracket{b_{t-1} + 1}^2}{p_\wedge^2} \lambda_{t-1}.
	\]
\end{lemma}
\begin{proof}
Recall the definition of $T_{2,3}^{\eta}$:
\begin{align*}
	T_{2,3}^{\eta}
	=&\;
	\E\sqbk{\innerprod{\nabla M(Q_{t-1}-\overline Q^*)
	-\nabla M(Q_{t-1-\tau^{\eta}_{t-1}}-\overline Q^*)}
	{F^{\eta}(Q_{t-1},D,y_{t-1})-\cT_{\overline P}(Q_{t-1})}}\\
	\leq&\;
	4L\,\E\sqbk{\spannorm{Q_{t-1}-Q_{t-1-\tau^{\eta}_{t-1}}}\,
	\spannorm{F^{\eta}(Q_{t-1},D,y_{t-1})-\cT_{\overline P}(Q_{t-1})}},
\end{align*}
where we used item~(4) of Lemma~\ref{lem:lyapunov_property} together with the seminorm equivalence between $\spanstar{\cdot}$ and $\spannorm{\cdot}$.

By Lemma~\ref{lem:app:lemma_b_2_in_chen_2025},
\[
\spannorm{Q_{t-1}-Q_{t-1-\tau^{\eta}_{t-1}}}
\le
f\!\bracket{\frac{\lambda\tau^{\eta}_{t-1}}
{N_{t-1-\tau^{\eta}_{t-1}-1,\min}+1+h}}\,(b_{t-1}+1).
\]
Moreover, by Lemma~\ref{lem:f_operator_properties}(2) and the fact that
$D(s,a)\ge p_\wedge$,
\[
\spannorm{F^{\eta}(Q_{t-1},D,y_{t-1})}
\le \frac{2}{p_\wedge}\,(\spannorm{Q_{t-1}}+1).
\]
Since $r\in[0,1]$, we also have
$\spannorm{\cT_{\overline P}(Q_{t-1})}\le \spannorm{Q_{t-1}}+1$.
Combining these bounds and using $\spannorm{Q_{t-1}}\le b_{t-1}$ yields
\[
\spannorm{F^{\eta}(Q_{t-1},D,y_{t-1})-\cT_{\overline P}(Q_{t-1})}
\le \frac{3(b_{t-1}+1)}{p_\wedge}.
\]
Therefore,
\begin{align*}
T_{2,3}^{\eta}
\le&\;
4L\cdot \frac{3(b_{t-1}+1)}{p_\wedge}\,
\E\sqbk{
f\!\bracket{\frac{\lambda\tau^{\eta}_{t-1}}
{N_{t-1-\tau^{\eta}_{t-1}-1,\min}+1+h}}\,(b_{t-1}+1)
}\\
\le&\;
\frac{12L(b_{t-1}+1)^2}{p_\wedge}\,
\E\sqbk{f\!\bracket{\frac{\lambda\tau^{\eta}_{t-1}}
{N_{t-1-\tau^{\eta}_{t-1}-1,\min}+1+h}}}\\
\le&\;
\frac{48L\,\tau^{\eta}_{t-1}\,|\cS||\cA|\,(b_{t-1}+1)^2}{p_\wedge^2}\,
\lambda_{t-1},
\end{align*}
where the last inequality follows from
\eqref{equ:app:special_function_expectation_bound} in
Lemma~\ref{lem:app:lemma_b_2_in_chen_2025}.
\end{proof}
Combining Lemma~\ref{lem:async_T_2_1}, Lemma~\ref{lem:async_T_2_2}, Lemma~\ref{lem:async_T_2_3}, we have when $t\geq T_{C_2}^{\eta}$:
	\begin{align*}
	T_2^{\eta} \leq& T_{2,1}^{\eta} + T_{2,2}^{\eta} + T_{2,3}^{\eta}\\
	&\frac{16L\bracket{b_{t-1} + \spannorm{\overline {Q}^*} + 1}^2}{p_\wedge}\lambda_{t-1} + \frac{48 \tau^{\eta}_{t-1} L |\cS||\cA|\bracket{b_{t-1} + \spannorm{\overline {Q}^*} + 1}^2}{p_\wedge^2} \lambda_{t-1}\\
	& + \frac{48 \tau^{\eta}_{t-1} L |\cS||\cA|(b_{t-1} + 1)^2}{p_\wedge^2}\lambda_{t-1}\\
	\leq& \frac{112 L \tau^{\eta}_{t-1}|\cS||\cA|\bracket{b_{t-1} +\spannorm{\overline {Q}^*} + 1}^2}{p_\wedge^2}\lambda_{t-1}.
\end{align*}
Then we completed the proof of Lemma~\ref{lem:async_T_2_bound}.

\subsection{Proof of Lemma~\ref{lem:async_T_3_bound}}
\label{sec:app:proof_of_lem:async_T_3_bound}
The term $T_3^{\eta}$ captures the difference between $F^{\eta}(Q_{t-1}, D_{t-1}, y_{t-1})$ and $F^{\eta}(Q_{t-1}, D, y_{t-1})$ due to the difference between $D_{t-1}$ and $D$.
\begin{lemma}[\citet{chen2025nonasymptotic}]
	\label{lem:empirical_stationary_distribution_bound}
	There exists a constant $T_{C_3}^{\eta}$ depends on $\frac{1}{1-r^{\eta}}$, $p_\wedge$, such that when $t\geq T_{C_3}^{\eta}$
	\begin{align}
		\E\sqbk{\frac{1}{D_t(s_t, a_t)^2}} \leq& \frac{6|\cS||\cA|}{p_\wedge^2}\label{equ:async:frequency_square_bound}\\
		\E\sqbk{\bracket{\frac{1}{D_{t}(s_{t}, a_{t})} - \frac{1}{D(s_{t}, a_{t})}}^2} \leq& \frac{63 C |\cS||\cA|}{(1-r_{\eta})(t + h)p_\wedge^2}.\label{equ:async:frequency_square_difference_bound}
	\end{align}
\end{lemma}
Recall that
\[
T_3^{\eta}
=
\E\sqbk{\innerprod{\nabla M(Q_{t-1}-\overline Q^*)}
{F^{\eta}(Q_{t-1},D_{t-1},y_{t-1})
- F^{\eta}(Q_{t-1},D,y_{t-1})}}.
\]

Let
\begin{equation}
    \label{equ:c_star}
c^* = \argmin{c\in \R} \norm[M]{F^{\eta}(Q_{t-1}, D_{t-1}, y_{t-1}) - F^{\eta}(Q_{t-1}, D, y_{t-1}) - c\mathbf{1}}.
\end{equation}
Then
\begin{align*}
    &\innerprod{\nabla M (Q_{t-1} - \overline{Q}^*)}{F^{\eta}(Q_{t-1}, D_{t-1}, y_{t-1}) - F^{\eta}(Q_{t-1}, D, y_{t-1})}\\
    =& \innerprod{\nabla M (Q_{t-1} - \overline{Q}^*)}{F^{\eta}(Q_{t-1}, D_{t-1}, y_{t-1}) - F^{\eta}(Q_{t-1}, D, y_{t-1}) - c^*\mathbf{1}}\\
    =& \seminorm[M]{Q_{t-1} - \overline{Q}^*} \innerprod{\nabla \seminorm[M]{Q_{t-1} - \overline{Q}^*}}{F^{\eta}(Q_{t-1}, D_{t-1}, y_{t-1}) - F^{\eta}(Q_{t-1}, D, y_{t-1}) - c^*\mathbf{1}}\\
    \leq& \seminorm[M]{Q_{t-1} - \overline{Q}^*}\norm[M, *]{\nabla \seminorm[M]{Q_{t-1} - \overline{Q}^*}} \norm[M]{F^{\eta}(Q_{t-1}, D_{t-1}, y_{t-1}) - F^{\eta}(Q_{t-1}, D, y_{t-1}) - c^*\mathbf 1}\\
    \overset{(i)}{=}& \seminorm[M]{Q_{t-1} - \overline{Q}^*}\norm[M,*]{\nabla \seminorm[M]{Q_{t-1} - \overline{Q}^*}} \seminorm[M]{F^{\eta}(Q_{t-1}, D_{t-1}, y_{t-1}) - F^{\eta}(Q_{t-1}, D, y_{t-1})}
\end{align*}
where $(i)$ is derived by \eqref{equ:c_star}. 

By Lemma~\ref{lem:seminorm_gradient_bound_in_dual}, $\norm[M, *]{\nabla \seminorm[M]{\cdot }}\leq 1$,
\begin{align*}
&\innerprod{\nabla M (Q_{t-1} - \overline{Q}^*)}{F^{\eta}(Q_{t-1}, D_{t-1}, y_{t-1}) - F^{\eta}(Q_{t-1}, D, y_{t-1})}\\
\leq& \seminorm[M]{Q_{t-1} - \overline{Q}^*}\seminorm[M]{F^{\eta}(Q_{t-1}, D_{t-1}, y_{t-1}) - F^{\eta}(Q_{t-1}, D, y_{t-1})}\\
\leq& \frac{1}{l_*}\seminorm[M]{Q_{t-1} - \overline{Q}^*} \spanstar{F^{\eta}(Q_{t-1}, D_{t-1}, y_{t-1}) - F^{\eta}(Q_{t-1}, D, y_{t-1})}\\
\overset{(i)}{\leq}& \frac{1}{2\kappa l_*}\seminorm[M]{Q_{t-1} - \overline{Q}^*}^2 + \frac{\kappa}{2 l_*} \spanstar{F^{\eta}(Q_{t-1}, D_{t-1}, y_{t-1}) - F^{\eta}(Q_{t-1}, D, y_{t-1})}^2\\
=& \frac{1}{\kappa l_*} M (Q_{t-1} - \overline{Q}^*) + \frac{\kappa}{2l_*}\spanstar{F^{\eta}(Q_{t-1}, D_{t-1}, y_{t-1}) - F^{\eta}(Q_{t-1}, D, y_{t-1})}^2\\
\overset{(ii)}{\leq}& \frac{1}{\kappa l_*} M(Q_{t-1} - \overline{Q}^*) + \frac{\kappa}{l_*}\spannorm{F^{\eta}(Q_{t-1}, D_{t-1}, y_{t-1}) - F^{\eta}(Q_{t-1}, D, y_{t-1})}^2,
\end{align*}
where (i) is by the Young's inequality, and (ii) is by the seminorm equivalence for $\spannorm{\cdot}$ and $\spanstar{\cdot}$. Choosing 
\[
    \kappa = \frac{1}{l_*(1-\beta u_*/l_*)},
\]
we obtain
\begin{align*}
&\innerprod{\nabla M (Q_{t-1} - \overline{Q}^*)}{F^{\eta}(Q_{t-1}, D_{t-1}, y_{t-1}) - F^{\eta}(Q_{t-1}, D, y_{t-1})}\\
=& \bracket{1 - \frac{\beta u_*}{l_*}} M(Q_{t-1} - \overline{Q}^*) + \frac{1}{l_*^2(1-\beta u_*/l_*)} \spannorm{F^{\eta}(Q_{t-1}, D_{t-1}, y_{t-1}) - F^{\eta}(Q_{t-1}, D, y_{t-1})}^2.
\end{align*}
Since $F^{\eta}(Q, D, y_{t-1}) = F^{\eta}(Q, D, s_{t-1}, a_{t-1}, s_{t})$ updates only the visited coordinate $(s_{t-1},a_{t-1})$, then the vector $F^{\eta}(Q_{t-1}, D_{t-1}, y_{t-1}) - F(Q_{t-1}, D, y_{t-1})$ only has non-zero entry at the coordinate corresponding to $(s_{t-1}, a_{t-1})$, we have:
\begin{align*}
&\spannorm{F^{\eta}(Q_{t-1}, D_{t-1}, y_{t-1}) - F^{\eta}(Q_{t-1}, D, y_{t-1})}^2\\
=& \bracket{\frac{1}{D_{t-1}(s_{t-1}, a_{t-1})} - \frac{1}{D(s_{t-1}, a_{t-1})}}^2 \bracket{r(s_{t-1}, a_{t-1}) + \max_{a'\in\cA} Q_{t-1}(s_t, a') - Q_{t-1}(s_{t-1}, a_{t-1})}^2\\
\leq& \bracket{\frac{1}{D_{t-1}(s_{t-1}, a_{t-1})} - \frac{1}{D(s_{t-1}, a_{t-1})}}^2(1 + \spannorm{Q_{t-1}})^2\\
\leq&(b_{t-1} +1)^2 \bracket{\frac{1}{D_{t-1}(s_{t-1}, a_{t-1})} - \frac{1}{D(s_{t-1}, a_{t-1})}}^2
\end{align*}
Taking expectation and apply Lemma \ref{lem:empirical_stationary_distribution_bound} to bound the expectation of the above term, we have, when $t\geq T_{C_3}^{\eta}$
\begin{align*}
T_3^{\eta} =& \E\sqbk{\innerprod{\nabla M(Q_{t-1} - \overline{Q}^*)}{F^{\eta}(Q_{t-1}, D_{t-1}, y_{t-1}) - F^{\eta}(Q_{t-1}, D, y_{t-1})}}\\
\leq& \bracket{1-\frac{\beta u_*}{l_*}}\E\sqbk{M(Q_{t-1} - \overline{Q}^*)} + \frac{(b_{t-1} + 1)^2}{l_*^2(1-\beta u_*/l_*)} \frac{63C|\cS||\cA|}{(1-r^{\eta})(t-1+h)p^2_\wedge}\\
\leq& \bracket{1-\frac{\beta u_*}{l_*}} \E\sqbk{M(Q_{t-1} - \overline{Q}^*)} + \frac{63C|\cS||\cA|(b_{t-1} + 1)^2}{l_*^2 (1-\beta u_*/l_*) (1-r^{\eta})\lambda^* p_\wedge^2}\lambda_{t-1}.
\end{align*}
This compeletes the proof.

\subsection{Proof of Lemma~\ref{lem:async_T_4_bound}}
\label{sec:app:proof_of_lem:async_T_4_bound}
We bound the second-order term $T_4^{\eta}$ by controlling the magnitude of a single asynchronous update.

When $\eta = \mathrm{exp}$, we have:
\begin{align*}
T_4^{\mathrm{exp}} =& \E\sqbk{\spanstar{F^{\mathrm{exp}}(Q_{t-1}, z_{t-1}) - Q_{t-1}}^2}\\
=& \E\sqbk{\spanstar{\frac{1}{D_{t-1}(s_{t-1}, a_{t-1})} \bracket{r(s_{t-1}, a_{t-1}) + \max_{a'\in \cA} Q_{t-1}(s_t, a') - Q_{t-1}(s_{t-1}, a_{t-1})} }^2}\\
\leq& 2\E\sqbk{\spannorm{\frac{1}{D_{t-1}(s_{t-1}, a_{t-1})} \bracket{r(s_{t-1}, a_{t-1}) + \max_{a'\in \cA} Q_{t-1}(s_t, a') - Q_{t-1}(s_{t-1}, a_{t-1})} }^2}\\
\leq& 2\E\sqbk{\bracket{\frac{1}{D_{t-1}(s_{t-1}, a_{t-1})}}^2\bracket{|r(s_{t-1}, a_{t-1})| + |\max_{a' \in \cA} Q_{t-1}(s_t, a') - Q_{t-1}(s_{t-1}, a_{t-1})|}^2}\\
\leq& 2\E\sqbk{\bracket{\frac{1}{D_{t-1}(s_{t-1}, a_{t-1})}}^2 (\spannorm{Q_{t-1}} + 1)^2}.
\end{align*}
When $\eta = \mathrm{imp}$, similarly,
\[
T_{4}^{\mathrm{imp}} \leq 2\E\sqbk{\bracket{\frac{1}{D_{t-1}(s_{t-1}, a_{t-1})}}^2 (\spannorm{Q_{t-1}} + 1)^2}.
\]
Since we have:
\[
\spannorm{Q_t} \leq \lambda^* |\cS||\cA|\log\bracket{\frac{t/(|\cS||\cA|) + h}{h}} = b_{t}
\]
Then:
\begin{align*}
T_4 \leq& 2 \E\sqbk{\bracket{\frac{1}{D_{t-1}(s_{t-1}, a_{t-1})}}^2 (b_{t-1} + 1)^2}\\
\leq& 2(b_{t-1} + 1) ^2 \E\sqbk{\bracket{\frac{1}{D_{t-1}(s_{t-1}, a_{t-1})}}^2}\\
=& 2(b_{t-1} + 1)^2 \sum_{s,a}\E\sqbk{\frac{\mathbb{I}_{\set{(s_{t-1}, a_{t-1}) = (s,a)}}}{D_{t-1}(s,a)^2}}.
\end{align*}
Apply the inequality~\eqref{equ:async:frequency_square_bound} in Lemma~\ref{lem:empirical_stationary_distribution_bound}, we have, there exists a constant $T_{C_4} = T_{C_3}$, when $t\geq T_{C_4}$
\[
T_4^{\eta} \leq \frac{12 |\cS||\cA|(b_{t-1}^2 + 1)}{p_\wedge^2}.
\]
This completes the proof.

    \section{Experimental MDP Construction}
\label{sec:app:numerical_experiments}

This section presents our numerical simulations for verifying Theorems~\ref{thm:sync_q_learning_optimal_rate}, ~\ref{thm:async_q_learning_optimal_rate_explicit} and~\ref{thm:async_q_learning_optimal_rate_implicit}. To demonstrate the convergence rate, we construct an AMDP that has four states $s_1, s_2, s_3, s_4$ and two actions $a_1, a_2$. Specifically, the transition probabilities and the reward function are defined in the following:
\[
\begin{aligned}
&p(s_3 \mid s_1,a_1)=p, 
&&p(s_4 \mid s_1,a_1)=1-p,
&&p(s_4 \mid s_2,a_1)=p,
&&p(s_3 \mid s_2,a_1)=1-p,\\
&p(s_1 \mid s_3,a_1)=p,
&&p(s_2 \mid s_3,a_1)=1-p,
&&p(s_1 \mid s_4,a_1)=1-p,
&&p(s_2 \mid s_4,a_1)=p,\\[2mm]
&p(s_3 \mid s_1,a_2)=q, 
&&p(s_4 \mid s_1,a_2)=1-q,
&&p(s_4 \mid s_2,a_2)=q,
&&p(s_3 \mid s_2,a_2)=1-q,\\
&p(s_1 \mid s_3,a_2)=q,
&&p(s_2 \mid s_3,a_2)=1-q,
&&p(s_1 \mid s_4,a_2)=1-q,
&&p(s_2 \mid s_4,a_2)=q.\\[2mm]
& r(s_1,a_1)=1,
&& r(s_2,a_1)=0,
&& r(s_3,a_1)=0,
&& r(s_4,a_1)=0,\\
& r(s_1,a_2)=1,
&& r(s_2,a_2)=0,
&& r(s_3,a_2)=0,
&& r(s_4,a_2)=0.
\end{aligned}
\]
The transition diagram of this kernel is provided in~\ref{fig:four_state_mdp_diamond}. One can verify that, in this example, we have $\max_{(s,a), (s', a')}\|p(\cdot|, s, a) - p(\cdot | s', a')\|_{\mathrm{TV}}=1$. Therefore, the Bellman operator $\cT_P$ is not a contraction under the span seminorm.

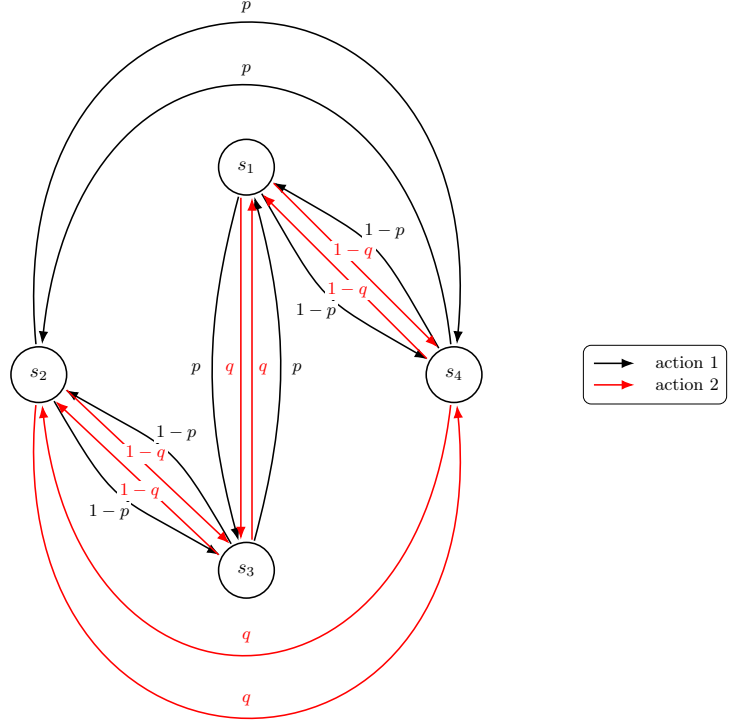
\begin{figure}[t]
\centering
\resizebox{0.6\linewidth}{!}{
\begin{tikzpicture}[
  >=Latex,
  node distance=30mm and 30mm,
  state/.style={circle, draw, thick, minimum size=10mm, inner sep=1pt},
  lab/.style={fill=white, inner sep=1.2pt, font=\small},
  actone/.style={->, thick, black},   
  acttwo/.style={->, thick, red},     
  shorten >=1pt, shorten <=1pt
]

\node[state] (s1) {$s_1$};
\node[state, below left=of s1] (s2) {$s_2$};
\node[state, below right=of s1] (s4) {$s_4$};
\node[state, below=of $(s2)!0.5!(s4)$] (s3) {$s_3$};

\draw[actone] (s1) to[out=255, in=105, looseness=1]
  node[lab, pos=0.50, xshift=-3mm] {$p$} (s3);
\draw[actone] (s3) to[out=75,  in=285, looseness=1]
  node[lab, pos=0.50, xshift= 3mm] {$p$} (s1);

\draw[acttwo, transform canvas={xshift=-1.0mm}] (s1) --
  node[lab, pos=0.50, xshift=-2mm] {$q$} (s3);
\draw[acttwo, transform canvas={xshift= 1.0mm}] (s3) --
  node[lab, pos=0.50, xshift= 2mm] {$q$} (s1);

\draw[actone] (s1) to[out=300, in=150, looseness=1.4]
  node[lab, pos=0.50, xshift=-2mm, yshift=-3mm] {$1-p$} (s4);
\draw[actone] (s4) to[out=120, in=330, looseness=1.4]
  node[lab, pos=0.50, xshift= 2mm, yshift= 3mm] {$1-p$} (s1);

\draw[acttwo, transform canvas={xshift= 1.0mm, yshift= 1.0mm}] (s1) --
  node[lab, pos=0.45, xshift=1mm, yshift=1mm] {$1-q$} (s4);
\draw[acttwo, transform canvas={xshift= -1.0mm, yshift= -1.0mm}] (s4) --
  node[lab, pos=0.45, xshift=-1mm, yshift=-1mm] {$1-q$} (s1);

\draw[actone]
  (s2) .. controls ($(s1)+(-45mm,45mm)$) and ($(s1)+(45mm,45mm)$) ..
  node[lab, pos=0.50, yshift=3mm] {$p$}
  (s4);
\draw[actone]
  (s4) .. controls ($(s1)+(30mm,30mm)$) and ($(s1)+(-30mm,30mm)$) ..
  node[lab, pos=0.50, yshift=3mm] {$p$}
  (s2);

\draw[acttwo]
  (s2) .. controls ($(s3)+(-45mm,-45mm)$) and ($(s3)+(45mm,-45mm)$) ..
  node[lab, pos=0.50, yshift=3mm] {$q$}
  (s4);
\draw[acttwo]
  (s4) .. controls ($(s3)+(30mm,-30mm)$) and ($(s3)+(-30mm,-30mm)$) ..
  node[lab, pos=0.50, yshift=3mm] {$q$}
  (s2);

\draw[actone] (s2) to[out=300, in=150, looseness=1.4]
  node[lab, pos=0.50, xshift=-2mm, yshift=-3mm] {$1-p$} (s3);
\draw[actone] (s3) to[out=120, in=330, looseness=1.4]
  node[lab, pos=0.50, xshift= 2mm, yshift= 3mm] {$1-p$} (s2);

\draw[acttwo, transform canvas={xshift= 1.0mm, yshift= 1.0mm}] (s2) --
  node[lab, pos=0.45, xshift=1mm, yshift=1mm] {$1-q$} (s3);
\draw[acttwo, transform canvas={xshift=-1.0mm, yshift=-1.0mm}] (s3) --
  node[lab, pos=0.45, xshift=-1mm, yshift=-1mm] {$1-q$} (s2);

\node[draw, rounded corners, inner sep=4pt, font=\small, anchor=west]
  at ($(s4.east)+(18mm,0)$) {%
    \begin{tabular}{@{}l@{}}
      \tikz{\draw[actone] (0,0) -- (8mm,0);} \ \ action 1\\
      \tikz{\draw[acttwo] (0,0) -- (8mm,0);} \ \ action 2
    \end{tabular}
  };

\end{tikzpicture}
}
\caption{Four-state MDP in diamond layout with two actions.}
\label{fig:four_state_mdp_diamond}
\end{figure}

\end{document}